\documentclass{article}

\usepackage{microtype}
\usepackage{graphicx}
\usepackage{booktabs} 
\usepackage{tabularx}
\usepackage{longtable}
\usepackage{booktabs}
\usepackage{natbib}
\usepackage{bbm}
\usepackage{enumitem}
\usepackage{graphicx}
\usepackage{comment}
\usepackage{caption}
\usepackage{subcaption}
\usepackage{tikz}
\usepackage{dsfont}
\usetikzlibrary{shapes.geometric, arrows.meta, positioning, shadows}

\usepackage{colortbl} 
\usepackage[table,xcdraw]{xcolor} 

\usepackage{algpseudocode}
\usepackage{algorithm}
\usepackage{archive}

 \algnewcommand\Parameters{\item[\textbf{Parameters:}]}

\usepackage[mathscr]{euscript}

\usepackage{multirow}
\usepackage{adjustbox}
\usepackage{graphicx}
\usepackage{caption}
\usepackage{subcaption}

\usepackage{amsmath}
\usepackage{amssymb}
\usepackage{mathtools}
\usepackage{amsthm}
\usepackage{mathrsfs}
\usepackage{nameref}
\usepackage{float}
\usepackage{soul}
\usepackage{stmaryrd}
\usepackage{tikz}
\usetikzlibrary{calc,positioning,intersections,quotes,decorations.markings,arrows.meta, bending}
\usepackage[doc]{optional}
\usepackage{tkz-euclide}

\definecolor{lightgray}{gray}{0.9}
\definecolor{darkblue}{rgb}{0.0,0.0,0.65}

\usepackage{xcolor}
\definecolor{refkey}{rgb}{0,0.6,0.0}
\definecolor{Brown}{rgb}{0.45,0.0,0.05}
\definecolor{lime}{rgb}{0.00,0.8,0.0}
\definecolor{lblue}{rgb}{0.5,0.5,0.99}
\usepackage[colorlinks=true,
            linkcolor= Brown,
            urlcolor=lblue,
            citecolor=blue]{hyperref}

\usepackage{graphicx}
\definecolor{labelkey}{rgb}{0,0.08,0.45}

\usepackage[capitalize,nameinlink]{cleveref}
\crefname{equation}{}{equations}
\crefname{figure}{Figure}{Figures}
\crefname{chapter}{Appendix}{chapters}
\crefname{item}{}{items}
\crefname{enumi}{}{}

\theoremstyle{plain}
\newtheorem{theorem}{Theorem}[section]
\newtheorem{proposition}[theorem]{Proposition}
\newtheorem{lemma}[theorem]{Lemma}
\newtheorem{corollary}[theorem]{Corollary}
\theoremstyle{definition}
\newtheorem{definition}[theorem]{Definition}
\newtheorem{assumption}[theorem]{Assumption}

\theoremstyle{remark}
\newtheorem{remark}[theorem]{Remark}

\crefname{assumption}{Assumption}{Assumptions}
\crefname{proposition}{Proposition}{Propositions}
\crefname{definition}{Definition}{Definitions}
\crefname{corollary}{Corollary}{Corollaries}

\usepackage{xspace}

\usepackage[textsize=tiny]{todonotes}

\title{Percolation Dynamics in Optimization: Variance Cascades and Nested Symmetry}

\author{{\bfseries Sai Niranjan Ramachandran}\thanks{School of Computation, Information and Technology, Technical University of Munich, Germany}\, \thanks{Munich Center for Machine Learning (MCML) }%
  \and
  {\bfseries Suvrit Sra}\footnotemark[1]\, \footnotemark[2] }

\begin{document}

\maketitle
\begin{abstract}
We study the dynamics of Stochastic Gradient Descent (SGD), which is known to
steer deep neural networks toward invariant sets that correspond to simpler
subnetworks. How this steering unfolds over time remains poorly understood.
We answer this by modeling the stochastic gradient flow (SGF) as a
percolation process, in which nested architectural symmetries force subnetworks to
merge in discrete blocks rather than by single-edge attachment. These
structural transitions register as variance spikes in a macroscopic order
parameter echoing physical phase transitions. We further state sufficient conditions under which the
trapping argument carries over to Adam and AdamW under heavy-tailed gradient
noise and measure them on a trained Transformer.
\end{abstract}

\section{Introduction}

Unlike classical statistical learning, deep learning systems rely on a complex interaction between the dataset, architecture, and choice of optimizer. Furthermore, they are known to achieve strong generalization without explicit regularization, a phenomenon attributed to implicit biases introduced during training. Recent work has established that Stochastic Gradient Descent (SGD) is a central source of this implicit bias collapsing networks onto invariant sets that
behave like much simpler subnetworks \citep{wei2008dynamics, chen2023stochastic}. However, the dynamics by which these invariant sets are reached are poorly understood.

This gap limits our ability to mechanistically explain anomalous training behaviors. Delayed generalization in transformer networks \citep{power2022grokking, liu2022towards} is one example. A network may memorize the training data perfectly, but it only discovers a generalizing solution after thousands of epochs, challenging the classical view of smooth optimization. Similar patterns can also be seen in exact deep linear networks \citep{saxe2013exact} and task shifts \citep{goodfellow2013empirical,
kirkpatrick2017overcoming}. This suggests that the topological structure of these dynamics contains crucial information about how optimization progresses.

To chart this evolution, we describe the optimization continuously with a stochastic differential equation (SDE) that records how parameters drift and diffuse over time \citep{li2017stochastic}. Architectural symmetries produce invariant sets which draw parameter trajectories \citep{chen2023stochastic}. Independent parameters entering these sets come together and fuse into equivalence classes, and the converging weights collide and merge, folding the parameter space into a simpler topology in the settings we examine \citep{naitzat2020topology}. We represent this dynamics as a percolation process where isolated components link and fuse, potentially causing sudden macroscopic transitions \citep{achlioptas2009explosive, chen2014microtransition}.  We develop this framework for SGD, where the mechanism is cleanest to state and prove, but most large-scale training uses adaptive optimizers, so we state sufficient conditions under which the trapping and cascade arguments carry over to Adam and AdamW with heavy-tailed gradient noise.

Building on this percolation framework, our contributions are as follows:

\begin{enumerate}[label=\textbf{\arabic*.}]
\item \textbf{Percolation Model of Topological Condensation:} We construct a
framework that maps stochastic gradient flow near invariant sets onto a graph
that tracks merges and splits (Reeb graph), which is then renormalized into a
percolation process. Nested architectural symmetries force merges across subtrees to occur as block merges, alongside
single-edge attachment within them, and we give the condition under which this
discontinuity survives as the network grows large.

\item \textbf{Variance Peaks and Topological Cascades:} We use relative
variance across training trajectories to isolate discrete microtransitions.
Each block merge of multiplicity $n$ produces a peak of height $(n-1)^2/(4n)$
at a fixed quantile of its hitting law. Modeling the block merges as a coalescence process, we show that the peaks
form a geometric cascade exhibiting the Discrete Scale Invariance (DSI) of
physical critical phenomena \citep{sornette1998discrete}. The ratio between
successive levels is $n$ in the merge fraction and $n^{1-\gamma_K}$ in
training time.

\item \textbf{Empirical Observations:} We test this framework across
environments and observe the predicted cascade structure. In toy models we
track parameters clustering under shifting data distributions, and a
comparison of a flat layer with a binary tree built from the same units shows
that nested symmetry produces the macroscopic block merges and that a
weight-decay-paced clock spaces the levels uniformly in time. Cascades with graded statistical support appear in MLPs on UCI tabular classification and vision
benchmarks and in a Transformer on modular arithmetic, on which we also
measure the Adam trapping condition.
\end{enumerate}

\textbf{Key takeaways.} Implicit bias is frequently characterized as a continuous shift toward simpler networks. We show that it instead happens via discrete merges of subnetworks, where each merge is marked by a peak in relative variance. Deep networks inherit nested symmetry from composition, leading to discontinuous joint merges of groups of subnetworks that are absent in a single flat layer. Such merges follow a geometric cascade and are evenly spaced under an implicit clock. We observe that in AdamW the weight decay additionally functions as a driver of subnetwork merges.


 \section{Stochastic Collapse in SGD}
\label{sec:stochastic_collapse}
As our principal aim is to study the transient dynamics of a network, we characterize the invariant sets of the training dynamics and the conditions under which they are attractive. We rely on permutation symmetries for this purpose.

Let $\boldsymbol{\theta} \in \mathbb{R}^d$ denote a model's parameter vector and $\mathcal{L}(\boldsymbol{\theta})$ a nonconvex loss, for example a neural network's empirical risk. At iteration $t$, optimizing $\mathcal{L}$ with learning rate $\eta$ over a mini-batch $\mathcal{B}_t$ yields the discrete update:
\begin{equation*}
    \boldsymbol{\theta}_{t+1} = \boldsymbol{\theta}_t - \eta \nabla \mathcal{L}_{\mathcal{B}_t}(\boldsymbol{\theta}_t).
\end{equation*}
By separating the exact full-batch gradient $\nabla \mathcal{L}(\boldsymbol{\theta}_t)$ from the zero-mean batch noise, we approximate this discrete sequence continuously using a Stochastic Gradient Flow (SGF) modeled by an It\^o SDE \citep{li2017stochastic}:
\begin{equation}
    d\boldsymbol{\theta}_t = -\nabla \mathcal{L}(\boldsymbol{\theta}_t)dt + \sqrt{\eta \Sigma(\boldsymbol{\theta}_t)} dW_t,
\end{equation}
where $W_t$ is a standard $d$-dimensional Wiener process and $\Sigma(\boldsymbol{\theta})$ is the position-dependent noise covariance matrix of the mini-batch gradients.

Architectural symmetries, such as permutation invariances among neurons, generically create flat, degenerate regions in the parameter space. We formally define these regions as invariant sets, following \citep{chen2023stochastic}.

\begin{definition}[Invariant Set]
A Borel set $A \subset \mathbb{R}^d$ is invariant for a stochastic process $\{\boldsymbol{\theta}_t\}$ if every trajectory initialized in $A$ stays in $A$ almost surely. For discrete SGD, $t$ ranges over the iterations.
\end{definition}

We note that for affine invariant sets, invariance under SGD carries over exactly to SGF.

\begin{proposition}[Affine Invariance in SGF]\label{prop:affine_invariance_main}
Consider an SGD process where the individual sample gradients $\nabla \ell(\cdot; x_i, y_i)$ are Lipschitz continuous and bounded. If an affine subset $A \subseteq \mathbb{R}^d$ is invariant under every sample gradient step, so that each $\nabla\ell(\cdot; x_i, y_i)$ is tangent to $A$ on $A$, then $A$ is an invariant set of the continuous SGF process.
\end{proposition}
\begin{proof}[Proof Sketch \citep{chen2023stochastic}]
As the symmetric root of $\eta\Sigma$ is only H\"older-$\tfrac12$ on $A$, we drive the SDE by the Lipschitz factor $G(\boldsymbol{\theta})$ of the centered sample gradients, $G G^\top = \eta\Sigma$, which admits a unique strong solution. Since the drift and $G$ are tangent to $A$ on $A$, the projected SDE matches the original one and trajectories stay in $A$ (Appendix \ref{app:sde_and_invariance}).
\end{proof}

We note that our analysis uses SGF while our experiments run on discrete SGD, which shares the same invariant sets (Proposition~\ref{prop:affine_invariance_main}). The discretization error vanishes on the invariant set \citep{li2017stochastic, chen2023stochastic}. Further, since the cascade predictions are ratios of times, they are unaffected by the conversion from steps to continuous time.

Once parameters drift near these invariant sets, the model predicts a qualitative change in the optimization dynamics. If the mini-batch variance $\Sigma(\boldsymbol{\theta})$ decays as $\boldsymbol{\theta}$ approaches $A$, an inward pull emerges that can trap trajectories near $A$. We formalize this as stochastic attractivity:

\begin{definition}[Stochastic Attractivity via Drift Dominance]
\label{def:stoch-attac}
Write $\boldsymbol{x} = \boldsymbol{\theta}^{(i)} - \pi_A(\boldsymbol{\theta}^{(i)})$ for the transverse vector, $Y_t^{(i)} = \|\boldsymbol{x}_t\|_2^2$ for the transverse distance process, and $\Sigma_\perp = P^\perp \Sigma_{ii} P^\perp$ for the transverse noise covariance. An invariant set $A \subset \mathbb{R}^{d_i}$ of the block $\boldsymbol{\theta}^{(i)}$ (Definition~\ref{def:subnet-part}) of the stochastic process $\{\boldsymbol{\theta}_t \in \mathbb{R}^d : t \geq 0\}$ is locally stochastically attractive of order $\alpha \in (0,1]$ if there exists a basin $\mathcal{N}(A, \epsilon)$ such that for all $\boldsymbol{\theta} \in \mathcal{N}(A, \epsilon) \setminus A$,
\begin{equation}
\label{eq:attractivity}
    -2\, \boldsymbol{x}^\top \nabla_{\boldsymbol{\theta}^{(i)}} \mathcal{L}(\boldsymbol{\theta}) + \eta \operatorname{tr}\Sigma_\perp(\boldsymbol{\theta}) \;\le\; 2(1-\alpha)\, \eta\, \frac{\boldsymbol{x}^\top \Sigma_\perp(\boldsymbol{\theta}) \boldsymbol{x}}{\|\boldsymbol{x}\|_2^2}.
\end{equation}
$A$ is strictly attractive with rate $\kappa > 0$ if the gap between the two sides is at least $(\kappa/\alpha)\|\boldsymbol{x}\|_2^2$, so that $\mathscr{G}(Y^{(i)})^\alpha \le -\kappa\,(Y^{(i)})^\alpha$. The infinitesimal generator applied to $(Y_t^{(i)})^\alpha$ equals $\alpha (Y_t^{(i)})^{\alpha-1}$ times the difference of the two sides of \eqref{eq:attractivity} (Appendix \ref{sec:appendix_topological_mechanics}, Equation \eqref{eq:generator}), so \eqref{eq:attractivity} states that the inward drift, helped by the It\^o correction of the concave map $Y \mapsto Y^\alpha$, overpowers the outward transverse diffusion.
\end{definition}

Sending $\alpha \downarrow 0$ recovers the collapse threshold $\mu \le \zeta^2/2$ of Appendix \ref{sec:appendix_toy_experiments} \citep{ziyin2022sgd, chen2023stochastic}. With isotropic noise in $k$ transverse directions, noise alone traps only for $k < 2(1-\alpha)$, and beyond that the inward drift or weight decay must carry the condition. SGF behaves as a non-equilibrium system driven by anisotropic, position-dependent noise \citep{chaudhari2018stochastic, mandt2017stochastic, mori2022logarithmic}, which motivates the analysis below.

\section{Topological Dynamics of SGF via Attractor Equivalence}
\label{sec:topological_dynamics}

To characterize how the network approaches these sets, we analyze the convergence of subnetworks under symmetries, e.g.\ the neurons of a layer, which are interchangeable under permutation and whose redundancy is exploited in pruning and dropout \citep{simsek2021geometry, frankle2019lottery, srivastava2014dropout}.

\begin{definition}[Subnetwork Partition]
\label{def:subnet-part}
Partition the global parameter vector into $N$ distinct subnetworks, $\boldsymbol{\theta} = [\boldsymbol{\theta}^{(1)}, \dots, \boldsymbol{\theta}^{(N)}]$, such that each $\boldsymbol{\theta}^{(i)} \in \mathbb{R}^{d_i}$ evolves over a marginal function space $C([0,\infty), \mathbb{R}^{d_i})$.
\end{definition}

We model the collapse of a network as a sequence of topological phase transitions between its subnetworks. To this end, we track the marginal path measure of each subnetwork over a filtered probability space $(\Omega, \mathcal{F}, \{\mathcal{F}_t\}_{t \ge 0}, \mathbb{P})$. We treat subnetworks separately, as symmetry gives each pair of identical subnetworks its own invariant set, while neural network training is documented to possess a simplicity bias that decouples parameter blocks \citep{simsek2021geometry, huh2021lowrank, shah2020pitfalls}.

The constructions below use only the transverse distance of each subnetwork to its invariant set, so no separate bound on the cross-covariance between blocks is needed: per-sample symmetry makes the pairwise transverse noise vanish on the coincidence set. To formalize structural binding, we project these trajectories into a shared quotient space by tracking their squared transverse distance $Y_t^{(i)} = \|\boldsymbol{\theta}_t^{(i)} - \pi_A(\boldsymbol{\theta}_t^{(i)})\|_2^2$ to a target invariant set $A$. Under stochastic attractivity, the inward deterministic gradient dominates the transverse diffusion, opposing outward drift. We show this is enough to trap the subnetwork within the local basin.

\begin{theorem}[Local Supermartingale Trapping]
\label{thm:supermartingale}
Given a stochastically attractive invariant set $A$ of order $\alpha$, there exists a local basin $\mathcal{N}(A, \epsilon)$ such that the stopped transverse process $(Y_{t \wedge \tau_\epsilon}^{(i)})^{\alpha}$ operates as a non-negative supermartingale, where $\tau_\epsilon = \inf\{t \ge t_0 : Y_t^{(i)} \ge \epsilon\}$ marks the boundary of escape. By Doob's maximal inequality, the probability of escaping the $\epsilon$-neighborhood is then at most $\epsilon^{-\alpha}\, \mathbb{E}[(Y_0^{(i)})^\alpha]$.
\end{theorem}
\begin{proof}[Proof Sketch]
By It\^o's lemma, the drift of $(Y_t^{(i)})^\alpha$ equals $\alpha (Y_t^{(i)})^{\alpha-1}$ times the difference of the two sides of \eqref{eq:attractivity}, so stochastic attractivity makes it non-positive before escape. The remaining stochastic integral has integrand of order $(Y_t^{(i)})^\alpha$, since the factor $(Y_t^{(i)})^{\alpha-1}$ is offset by transverse noise vanishing on $A$ at Lipschitz rate. As the stopped process is bounded by $\epsilon$ (Lemma \ref{lemma:stopped_bounded}), the integral is a genuine martingale \citep{karatzas1991brownian} and vanishes under conditional expectation.
\end{proof}

Theorem \ref{thm:supermartingale} traps decoupled subnetworks within shared invariant subspaces and drives their path measures toward synchronization. We formalize this binding as \emph{Attractor Linkage}.

\begin{definition}[Attractor Linkage, $\sim_A$]
\label{def:attractor_equiv}
For two subnetworks of identical architecture let $A_{ij} = \{\boldsymbol{\theta} : \boldsymbol{\theta}^{(i)} = \boldsymbol{\theta}^{(j)}\}$ and $Y_t^{(ij)} = \tfrac12 \|\boldsymbol{\theta}_t^{(i)} - \boldsymbol{\theta}_t^{(j)}\|_2^2$ its squared transverse distance, with escape time $\tau_\epsilon^{(ij)} = \inf\{s \ge t : Y_s^{(ij)} \ge \epsilon\}$. Fix an inner radius $\epsilon_{\mathrm{in}} < \epsilon$. Then $\boldsymbol{\theta}^{(i)} \sim_A \boldsymbol{\theta}^{(j)}$ at time $t$ if and only if some $s \le t$ has $Y_s^{(ij)} < \epsilon_{\mathrm{in}}$ and $Y_u^{(ij)} < \epsilon$ for all $u \in [s,t]$, so a link forms at the inner radius and breaks at the next escape from the tube. The link probability
\begin{equation}
    \mathcal{P}_{\mathrm{link}}^{(i, j)}(t) := \mathbb{P}\left( \tau_\epsilon^{(ij)} = \infty \mid \mathcal{F}_t \right)
\end{equation}
measures whether the link survives. When $A_{ij}$ is stochastically attractive of order $\alpha$, Lemma \ref{lemma:doob-app} gives $\mathcal{P}_{\mathrm{link}}^{(i,j)}(t) \ge 1 - \epsilon^{-\alpha} (Y_t^{(ij)})^{\alpha}$. The class $[i]_t$ is the connected component of $i$ in the graph $G_t^\epsilon$ whose edges are the links. Components partition $\mathcal{S}$, so the classes are transitive by construction.
\end{definition}

Links are reversible, and we project them onto a continuous Reeb graph \citep{edelsbrunner2008topological, carlsson2009topology}. Let $\mathcal{S} = \{1, \dots, N\}$ index the subnetworks. A \emph{topological transition} is any discrete change in the membership of $[i]_t$. We project these classes onto the quotient space $\mathcal{R} := (\mathcal{S} \times [0, \infty)) / \sim_{\mathcal{R}}$, where $(i,t) \sim_{\mathcal{R}} (j,t)$ identifies two subnetworks whenever $j \in [i]_t$.

\begin{theorem}[Non-Equilibrium Condensation and Fragmentation]
\label{thm:reeb_transitions-main}
The Reeb graph $\mathcal{R}$ realizes these topological transitions via two mechanisms:
\begin{itemize}
    \item \textbf{Condensation ($\cup$):} When $Y_t^{(ij)}$ falls below $\epsilon_{\mathrm{in}}$, the edge forms and Lemma \ref{lemma:doob-app} bounds the chance that it ever breaks by $(\epsilon_{\mathrm{in}}/\epsilon)^\alpha$. If the classes were disjoint, the quotient map binds them ($[i]_t = [i]_{t-\delta} \cup [j]_{t-\delta}$) into a single merge node.
    \item \textbf{Fragmentation ($\emptyset$):} If a variance deviation breaches the basin, $\tau_\epsilon^{(ij)} \le t$ triggers, the edge is severed, and the class splits ($[i]_t \cap [j]_t = \emptyset$) into decoupled branches.
\end{itemize}
\end{theorem}
\begin{proof}[Proof Sketch]
Doob's bound at the inner radius gives the link probability, the quotient projection $\pi_{\mathcal{R}}$ maps merging branches into a vertex of in-degree $>1$, and a large deviation \citep{freidlin2012random} that breaches $\mathcal{N}(A_{ij}, \epsilon)$ deletes the edge and produces a vertex of out-degree $>1$ \citep{edelsbrunner2008topological}.
\end{proof}

\section{Percolative Collapse and Discrete Scale Invariance}
\label{sec:main_dsi_percolation}

As merges accumulate during the transient, we analyze the structure of this process as training progresses. We show that it admits a regular structure, namely a percolation process, in which merges can be tracked by peaks in relative variance.

\subsection{Architectural Symmetry and Graph Renormalization}

The equivalence classes defining the Reeb graph's vertices correspond to topological attractors within the loss landscape, since permutation invariances among functionally equivalent neurons partition the parameter space into lower-dimensional geometric subspaces \citep{chen2023stochastic}.

\begin{definition}[Approximate $Q$-Symmetry]
\label{def:approximate-q-symmetry-main}
Let $Q \in \mathbb{R}^{d \times d}$ be orthogonal, $Q Q^\top = I_d$. A loss functional $\mathcal{L}(\boldsymbol{\theta})$ exhibits approximate $Q$-symmetry around $A = \{\boldsymbol{\theta} \in \mathbb{R}^d \mid Q\boldsymbol{\theta} = \boldsymbol{\theta}\}$ if, for any $\epsilon > 0$, there exists $\delta > 0$ such that
\begin{equation}
    |\mathcal{L}(Q\boldsymbol{\theta}) - \mathcal{L}(\boldsymbol{\theta})| < \epsilon \cdot d(\boldsymbol{\theta}, A) \quad \forall \boldsymbol{\theta} \text{ s.t. } d(\boldsymbol{\theta}, A) < \delta,
\end{equation}
where $d(\boldsymbol{\theta}, A) = \inf_{\boldsymbol{a} \in A} \|\boldsymbol{\theta} - \boldsymbol{a}\|_2$.
\end{definition}

Proposition~\ref{prop:affine_invariance_main} needs the per-sample version, which permutation symmetries satisfy exactly.

\begin{theorem}[Affine Trapping and Transverse Diffusion]
\label{thm:main_affine_trap}
If $\mathcal{L}$ satisfies approximate $Q$-symmetry around $A$, deterministic gradient drift is tangent to $A$ on $A$. The affine geometry of $A$ decouples the transverse stochastic distance $Y_t = d(\boldsymbol{\theta}_t, A)^2$ from the tangential motion without any curvature drift.
\end{theorem}
\begin{proof}[Proof Sketch]
On $A$, along a normal vector $\boldsymbol{n} \perp A$, approximate symmetry forces $\nabla_{\boldsymbol{n}} \mathcal{L}(\boldsymbol{\theta}) = 0$, so the drift keeps $A$ invariant. Because $A$ is affine, all principal curvatures vanish ($H_A = 0$), so the It\^o correction of $Y_t$ carries no geometric term and the transverse generator takes the flat form of Definition \ref{def:stoch-attac}. Details are in Appendix \ref{sec:appendix_dsi_percolation}.
\end{proof}

Stochastic gradient noise breaks and reforms edges on the transient Reeb graph, and the ensemble tube occupancy supplies the monotone control parameter.

\begin{assumption}[Condensation Phase]
\label{assum:condensation_phase}
On the training interval under study, the expected number of connected components $\mathbb{E}[M(\tau)]$ of $G_\tau^\epsilon$ is nonincreasing in $\tau$.
\end{assumption}

\begin{proposition}[Reeb Graph Renormalization to Percolation Graph]
\label{thm:main_renormalization}
Let $M(\tau)$ be the number of connected components of $G_\tau^\epsilon$ and define the merge fraction $p(\tau) := 1 - \mathbb{E}[M(\tau)]/N \in [0, 1 - 1/N]$. Under Assumption \ref{assum:condensation_phase}, $p$ is nondecreasing. Indexing the realization graphs $G_\tau := G_\tau^\epsilon$ by $p$ through $\tau_p := \inf\{\tau : p(\tau) \ge p\}$ turns the transient Reeb graph $\mathcal{R}_t$ into a percolation process $G_p := G_{\tau_p}$ with control parameter $p \in [0, 1 - 1/N]$.
\end{proposition}
\begin{proof}[Proof Sketch]
$\mathbb{E}[M]$ is nonincreasing by assumption, while individual realizations still fragment, and those events feed the relative variance below. Links form under strict attractivity (Proposition \ref{prop:strict_decay}) and persist with probability at least $1 - (\epsilon_{\mathrm{in}}/\epsilon)^\alpha$ (Lemma \ref{lemma:doob-app}), and the generalized inverse $\tau_p$ handles plateaus of $p$ (Appendix \ref{sec:appendix_dsi_percolation}).
\end{proof}

We track the connectivity of $G_p$ via the structural order parameter, the fractional size of the largest connected component:
\begin{equation}
    \mathcal{O}(p) = \frac{|C_{\max}(p)|}{N}.
\end{equation}

Unlike Erd\H{o}s-R\'enyi percolation, where the largest component grows by absorbing components of vanishing relative size \citep{stauffer1994introduction}, symmetry-induced percolation can merge components in blocks. A condensation event has \emph{multiplicity} $n$ if it joins $n$ components of equal size $i$ into one of size $ni$. Pairwise coincidence sets $A_{ij}$ have the smallest codimension among the $S_n$-fixed subspaces, so generic condensation is pairwise and $n = 2$ (Proposition \ref{thm:hypercondensation}). We keep $n$ as a parameter.

\begin{proposition}[Non-Erd\H{o}s-R\'enyi Discontinuity]
\label{thm:main_er_discontinuity}
A condensation event of multiplicity $n$ on components of size $i$ produces a jump $\Delta \mathcal{O}(p) = (n-1)i/N$, a genuine discontinuity as $N \to \infty$ precisely when the merging components are already macroscopic ($i = \Theta(N)$), and a vanishing one when $i = O(1)$ (Corollary \ref{cor:microscopic_washout-app}). In $G_{ER}(N,p)$ every jump of $\mathcal{O}_{ER}$ is $o(1)$ with high probability, so the symmetry mechanism departs from Erd\H{o}s-R\'enyi percolation only in the macroscopic regime.
\end{proposition}
\begin{proof}[Proof Sketch]
Joining $n$ components of size $i$ raises $|C_{\max}|$ to $ni$. In Erd\H{o}s-R\'enyi graphs the largest component only ever absorbs components of size $O(N^{2/3})$ inside the critical window and $O(\log N)$ outside it (Appendix \ref{sec:appendix_dsi_percolation}). Bounded-choice Achlioptas processes are continuous in this limit too \citep{riordan2011explosive}.
\end{proof}

Because stochastic noise shifts the density at which these discrete merges occur, single-trajectory observations obscure the critical thresholds. We isolate structural shocks via the relative variance of the order parameter across an ensemble of training trajectories:
\begin{equation}
    R_v(p) = \frac{\mathbb{E}\left[(\mathcal{O}(p) - \mathbb{E}[\mathcal{O}(p)])^2\right]}{\mathbb{E}[\mathcal{O}(p)]^2}.
\end{equation}

\begin{theorem}[Bounded Peak of Relative Variance]
\label{thm:main_rv_divergence}
Assume that near a microtransition the ensemble law of $\mathcal{O}$ is a two-point mixture of $\mathcal{O}_1$ and $\mathcal{O}_1 + \Delta\mathcal{O}$ with weight $w$ on the merged state, as a single block merge produces. Then
\begin{equation}
    R_v(w) = \frac{w(1-w)\,\Delta\mathcal{O}^2}{(\mathcal{O}_1 + w\,\Delta\mathcal{O})^2}
\end{equation}
vanishes at $w = 0$ and at $w = 1$ and peaks at $w^* = \mathcal{O}_1/(2\mathcal{O}_1 + \Delta\mathcal{O})$. For a block merge $\Delta\mathcal{O} = (n-1)\mathcal{O}_1$, so the peak occurs at $w^* = 1/(n+1)$ with height $(n-1)^2/(4n)$, independent of $i$ and $N$. Each microtransition therefore leaves a sharp, bounded peak in $R_v(p)$.
\end{theorem}
\begin{proof}[Proof Sketch]
The mixture has mean $\mathcal{O}_1 + w\,\Delta\mathcal{O}$ and variance $w(1-w)\,\Delta\mathcal{O}^2$, which gives the ratio above. Setting the derivative of its logarithm in $w$ to zero yields $w^*$ \citep{chen2014microtransition}. For a block merge, $\Delta\mathcal{O} = (n-1)\mathcal{O}_1$, so $\mathcal{O}_1$ cancels and the peak height $(n-1)^2/(4n)$ depends on neither $i$ nor $N$. The first merge out of isolated vertices has $\mathcal{O}_1 = 1/N$ and gives the same height (Appendix \ref{sec:appendix_dsi_percolation}).
\end{proof}
The spectral observable of Section~\ref{sec:experiments} records the peak times, not this height (Proposition~\ref{thm:spectral_relaxation}).

Block merges break the continuous group of dilations into a discrete subgroup. An observable exhibits Discrete Scale Invariance (DSI) if it is self-similar under a discrete set of magnification factors $\lambda^{k}$ rather than under every dilation \citep{sornette1998discrete}. Let $T_i$ denote the time at which the mean component size $N/M(t)$ of a realization first reaches $i$, and call its law the hitting law of size $i$. In the balanced cascade below every component at a level has the same size, so $T_i$ is also the time at which the largest component reaches $i$. Merged subnetworks coincide on $A_{ij}$, so a component has a single representative and its coincidence set with another component is again an affine subspace of the form of Theorem \ref{thm:main_affine_trap}. We assume that the transverse process between two components depends on them only through their sizes, so that the merges form a Marcus--Lushnikov coalescence process \citep{marcus1968, lushnikov1978, aldous1999} with a kernel.

\begin{assumption}[Homogeneous Coalescence Kernel]
\label{assum:kernel}
During the condensation phase, two components of sizes $a$ and $b$ link at rate $K(a,b)/N$, where $K$ is symmetric, depends on the components only through their sizes, is homogeneous of degree $\gamma_K < 1$, that is, $K(ca, cb) = c^{\gamma_K} K(a,b)$, and admits a dynamical scaling solution of the Smoluchowski equation. In the coalescence idealization, links formed during the phase are not broken.
\end{assumption}

Dynamical scaling is proved for constant and additive kernels and conjectured in general \citep{aldous1999}.

\begin{lemma}[Self-Similar Hitting Laws]
\label{lem:dsi_selfsim}
Under Assumption \ref{assum:kernel}, in the mean-field limit the typical component size grows as $s(t) \propto t^{1/(1-\gamma_K)}$, so $T_{ni}/T_i \to \lambda_t := n^{1-\gamma_K}$. If in addition the merges are balanced, so that components of equal size merge in groups of $n$ (Definition \ref{def:merge_multiplicity}), then $T_{n^{k+1}} \overset{d}{=} \lambda_t T_{n^k}$ up to a relative correction of order $\lambda_t^{-k} + (n^k/N)^{1/2}$, and the merge fraction satisfies $1 - p_{n^k} = n^{-k}$. For the constant kernel $T_{n^k} = 2(n^k - 1)$, so consecutive ratios are $(n^{k+1}-1)/(n^k-1)$ and tend to $n$ (Appendix \ref{sec:appendix_dsi_percolation}).
\end{lemma}

The measured ratio gives $\gamma_K = 1 - \log_n\lambda_t$, into which a time-varying link rate also enters (Appendix \ref{sec:appendix_block_merges}).

\begin{theorem}[Coalescence with Symmetry Yields DSI]
\label{thm:main_dsi}
Under Assumption \ref{assum:kernel} with a time-homogeneous kernel and balanced merges, the mean transition time, every quantile, and the location $q_i$ of the $R_v$ peak, the $w^*$-quantile of $T_i$ by Theorem \ref{thm:main_rv_divergence}, all satisfy $q_{ni}/q_i = \lambda_t = n^{1-\gamma_K}$ up to the corrections of Lemma \ref{lem:dsi_selfsim}, asymptotically in the level $k$. The peak times are invariant in law under $t \mapsto \lambda_t t$, which is DSI in training time \citep{sornette1998discrete}, and the peak merge fractions satisfy $1 - q^{(p)}_{ni} = (1 - q^{(p)}_i)/n$.
\end{theorem}
\begin{proof}[Proof Sketch]
Under dynamical scaling, the typical component size grows as $t^{1/(1-\gamma_K)}$, so the time to reach size $i$ scales as $i^{1-\gamma_K}$ and $T_{ni}/T_i \to n^{1-\gamma_K}$ (Lemma \ref{lem:dsi_selfsim}). Rescaling a law by $\lambda_t$ rescales its mean and every quantile by $\lambda_t$. The $R_v$ peak is the $w^*$-quantile of $T_i$, so its location inherits the same factor. Balanced merges divide the number of components by $n$ at each level, which gives the relation for the merge fraction (Appendix \ref{sec:appendix_dsi_percolation}).
\end{proof}
The theorem describes the balanced, time-homogeneous coalescence idealization, and Section \ref{sec:experiments} tests trained networks for the geometric spacing and peak structure it predicts. A weight-decay-paced clock keeps this level structure but spaces the levels uniformly in $t$ (Appendix~\ref{sec:appendix_block_merges}).

\section{Conditional Extension to Adam and AdamW}
\label{sec:main_adam}

As most large-scale training uses adaptive optimizers, we extend the trapping and cascade arguments of Sections \ref{sec:topological_dynamics} and \ref{sec:main_dsi_percolation} to Adam \citep{kingma2015adam} and AdamW. We show that they carry over under a trapping condition, which we measure on a trained Transformer in Section \ref{sec:experiments}, with full derivations in Appendix \ref{sec:appendix_adam}. Throughout, $\epsilon_{\mathrm{opt}}$ denotes Adam's numerical constant and $\epsilon$ the tube radius.

Adam and AdamW maintain a joint state $(\boldsymbol\theta_t,m_t,v_t)$, and the elementwise squaring in $v_t$ narrows the admissible symmetry group to coordinate permutations $P_\pi$, a class that already contains the $S_n$ neuron-permutation symmetry, so the joint invariant set $\tilde A=\{(\boldsymbol\theta,m,v):P_\pi\boldsymbol\theta=\boldsymbol\theta,P_\pi m=m,P_\pi v=v\}$ extends Theorem \ref{thm:affine_generation} without modifying the geometry of $A$.

\begin{assumption}[Heavy-Tailed Gradient Noise]
\label{assum:main_heavy_tail}
There exist $p\in(1,2]$ and $\sigma>0$ such that $\mathbb E[|g_{t,i}|^p \mid \mathcal{F}_{t-1}]\le\sigma^p$ almost surely for every coordinate $i$ and every $t$, matching the heavy-tailed regime documented for attention-based architectures \citep{zhang2020adaptive}.
\end{assumption}

As heavy tails permit infinite variance for $p<2$, we analyze Adam with gradients clipped at a threshold $\tau$, $\hat g_t := \mathrm{clip}_\tau(g_t)$, as attention-based models are trained in practice. Clipping bounds every moment at a bias of at most $\sigma^p\tau^{1-p}$. We show that $\hat v_t$ and $\hat m_t$ then track the second moment and the gradient with high probability, where $\tau_2 = 1/(1-\beta_2)$ is the memory of $\hat v_t$ (Appendix \ref{sec:appendix_adam}). Across a symmetric pair of units, the preconditioner splits as $c_t + E_t$, with $c_t$ diagonal, role-wise and shared by both units. Theorem \ref{thm:main_adam_trapping} then states the condition that fixes the sign of the transverse drift.

\begin{assumption}[Local Quasi-Static Regime and Macroscopic Merges]
\label{assum:main_nondeg_blocksc}
We assume that $\hat v_i(\boldsymbol\theta_t) \ge v_{\min} > 0$ for every coordinate under analysis and $t \le \tau_\epsilon^{\mathrm{blk}}$, that $1-\beta_2$ is small enough for $\hat v$ on the block to stay frozen over the trapping horizon, so that $c_t = c$ and $E_t = E$, and that the merging components satisfy $i(N) = \Theta(N)$.
\end{assumption}

Assumption \ref{assum:main_nondeg_blocksc} describes a local window of a run, and Remark \ref{remark:adam_diagnostic} measures the variation of $\hat v$ inside such windows. Under the frozen preconditioner, the transverse coordinate $\boldsymbol{x}_t$ of the block follows the Euler--Maruyama step of the reduced SDE \eqref{eq:adam_reduced}, with drift $-\eta c\nabla_\perp\mathcal L - \eta\lambda\boldsymbol{x}_t\mathds{1}_{\mathrm{AdamW}}$, diffusion $\eta c(\Sigma^\tau_\perp)^{1/2}$ for the transverse covariance $\Sigma^\tau_\perp$ of the clipped gradient, a frozen forcing $\boldsymbol{f} = -\eta P^\perp E\,\nabla\mathcal L$, and a residual $\|\boldsymbol\varrho_t\| \le \rho_{\max}$ collecting the momentum lag, clipping bias and tracking errors (Proposition \ref{thm:adam_sgf_reduction}). The white-noise reading of the filtered momentum noise in that proposition is a modeling step of the same kind as the one that turns SGD into the SGF of Section \ref{sec:stochastic_collapse}. The forcing is $O(\sqrt{\epsilon})$ and moves the trap off $A$ by $O(\|\boldsymbol{f}\|)$.

\begin{theorem}[Conditional Trapping Criterion for Adam and AdamW]
\label{thm:main_adam_trapping}
Fix $\alpha \in (0,1]$ and $0 < r_0 < \sqrt\epsilon$, let $\rho_\parallel \le \rho_{\max}$ be a constant with $|\boldsymbol{x}_t^\top\boldsymbol{\varrho}_t| \le \rho_\parallel\sqrt{Y_t}$ on the tracking event of Lemma \ref{lemma:uniform_tracking}, and suppose that on $\{r_0^2 \le Y < \epsilon\}$
\begin{equation}
\label{eq:adam_condition_main}
    -2\eta\, \boldsymbol{x}^\top c\,\nabla_\perp\mathcal L - 2\eta\lambda Y \mathds{1}_{\mathrm{AdamW}} + 2\,\boldsymbol{x}^\top\boldsymbol{f} + 2\sqrt{Y}\,\rho_\parallel + \eta^2 \operatorname{tr}(c\,\Sigma^\tau_\perp c) \;\le\; 2(1-\alpha)\,\eta^2\, \frac{\boldsymbol{x}^\top c\,\Sigma^\tau_\perp c\,\boldsymbol{x}}{Y}.
\end{equation}
Then, under \eqref{eq:adam_reduced}, $Y^{\alpha}$, stopped when $Y$ leaves the annulus $[r_0^2, \epsilon)$, is a non-negative supermartingale, and over $T$ steps the trajectory escapes the tube before reaching the core $\{Y < r_0^2\}$ with probability at most $\epsilon^{-\alpha}\mathbb E[Y_{t_0}^\alpha \mid \mathcal F_{t_0}] + 2\delta$, the $2\delta$ from Lemma \ref{lemma:uniform_tracking}.
\end{theorem}
\begin{proof}[Proof Sketch]
It\^o's lemma on $Y^\alpha$ under \eqref{eq:adam_reduced} produces the left side of \eqref{eq:adam_condition_main} and the concavity correction on the right. The diffusion enters at second order in $c$ on both sides, so positivity of $c$ alone does not determine the sign, and \eqref{eq:adam_condition_main} is the hypothesis that determines it. The forcing and the residual enter through their projections on $\boldsymbol{x}$, and Cauchy--Schwarz gives the cruder sufficient condition with $2\sqrt Y(\|\boldsymbol{f}\| + \rho_{\max})$ in their place. Doob's inequality gives the escape bound (Appendix \ref{sec:appendix_adam}).
\end{proof}

\begin{theorem}[Conditional Cascade Transfer to Adam and AdamW]
\label{thm:main_adam_dsi}
Under Assumptions \ref{assum:main_heavy_tail}, \ref{assum:main_nondeg_blocksc}, \ref{assum:condensation_phase} and \ref{assum:kernel}, and the confinement condition of Corollary \ref{cor:adam_nonescape}, on $t \le \tau_\epsilon^{\mathrm{blk}}$ the Adam- or AdamW-trained trajectory admits a percolation process whose $R_v$ peaks form the cascade of Theorem \ref{thm:main_dsi} with the same factor $\lambda_t$.
\end{theorem}
\begin{proof}[Proof Sketch]
The confinement bound of Corollary \ref{cor:adam_nonescape} supplies the link survival that Lemma \ref{lemma:doob-app} provides for SGD, and the frozen preconditioner $c$ never enters the ratio (Appendix \ref{sec:appendix_adam}, measured in Remark \ref{remark:adam_diagnostic}).
\end{proof}

Corollary \ref{cor:adam_nonescape} certifies less than one inter-slingshot window, so the guarantee applies window by window.


\section{Empirical Validation}
\label{sec:experiments}
 
To test whether this percolation holds beyond our assumptions, we track merges in trained networks through the order parameter $\mathcal{O}$ and its relative variance $\mathcal{R}_v$ across training seeds, which realizes the two-point mixture of Theorem~\ref{thm:main_rv_divergence}. We start from toy models, where merges can be observed directly, and move to MLPs and a Transformer, where we rely on the spectral effective rank (Appendices~\ref{sec:appendix_toy_experiments} and~\ref{sec:appendix_empirical_suite}).

\subsection{Stochastic Condensation and Reactive Fragmentation}
 
In a controlled toy setting adapting \citet{chen2023stochastic}, pairwise
merges ($n = 2$) under a constant kernel ($\gamma_K = 0$) give $\lambda_t = 2$
asymptotically (Corollary~\ref{cor:peak_times}). A constrained $N=3$ kinematic
model checks that the detection pipeline recovers imposed merge times, and
unconstrained $N=6$ label-noise SGD gives a two-peak ratio $\lambda_t \approx 2.00$ from
localized $\mathcal{R}_v$ peaks, with the finite-$N$ prediction for six units at $2.5$
(Appendix~\ref{sec:appendix_toy_experiments}), and a task shift at $t=3000$ confirms the predicted
reversibility, producing fragmentation at the rate $\eta(\mu - \zeta^2/2)$ sets
(Appendix~\ref{sec:appendix_toy_experiments}, Figure~\ref{fig:toy_experiments}).
We also compare a nested binary-tree network with a flat layer built from the
same units and trained by SGD with weight decay, whose exponential clock lies outside Lemma~\ref{lem:dsi_selfsim} but keeps its level structure. The tree also produces the macroscopic block
merges of Proposition~\ref{thm:main_er_discontinuity} that make the $R_v$
peaks sharp, with the first-level peak at the height $1/8$ and weight $w^* = 1/3$ of Theorem~\ref{thm:main_rv_divergence}, reaching one cluster in $47\%$ of seeds at $N=32$, while the flat
layer condenses to two clusters by single attachments
(Figure~\ref{fig:block_merges}, Appendix~\ref{sec:appendix_block_merges}).
 
\subsection{Universality, Grokking, and Benchmark Generalization}
 
Deeper networks carry the nested structure of the tree through the
composition of their layers; however, moving from the toy models to general,
unconstrained networks requires replacing distance-based clustering with
Spectral Effective Rank as the order parameter.\footnote{Spectral Effective Rank measures soft dimensionality collapse via the exponentiated Shannon entropy of a weight matrix's singular value spectrum, generalizing hard, exact-zero collapse to the case where weights rarely vanish exactly.}
 
In a Transformer
trained on modular arithmetic, delayed generalization (grokking;
\citealp{power2022grokking}) coincides with a 3-peak cascade (epochs $4500$, $9500$, $20000$; Figure~\ref{fig:unified_generalization}) immediately
preceding the performance spike ($\lambda_t = 2.11$, phase-randomized spectral
null false positive rate (FPR) $\le 0.1\%$), though we do not claim this
cascade is the sole driver of the transition. On the same AdamW run we measured the quantities in condition \eqref{eq:adam_condition_main} on permutation pairs of feed-forward units (Remark \ref{remark:adam_diagnostic}). With dropout active, the trapping condition holds at $50$ to $62\%$ of checkpoints ($62$ to $66\%$ inside the tube), the transverse drift is inward at $63$ to $74\%$, and contraction phases reach $3.6$ decades. Within the horizon that Corollary \ref{cor:adam_nonescape} certifies, these measurements agree with the local assumptions of Section \ref{sec:main_adam}. Slingshot instabilities \citep{thilak2022slingshot} punctuate the run about once per $1.4\,\tau_2$ and restart the condensation phase (Remark~\ref{remark:adam_diagnostic}). In contrast, a cascade structure
appears with cleaner transient dynamics and fractional $\lambda_t$ across UCI tabular classification and vision benchmarks trained with SGD, with graded surrogate support
(Appendix~\ref{sec:appendix_empirical_suite}, Figures~\ref{fig:uci_heart} and~\ref{fig:suite_fashionmnist}).

\begin{figure}[t]
    \centering
    \includegraphics[width=\textwidth]{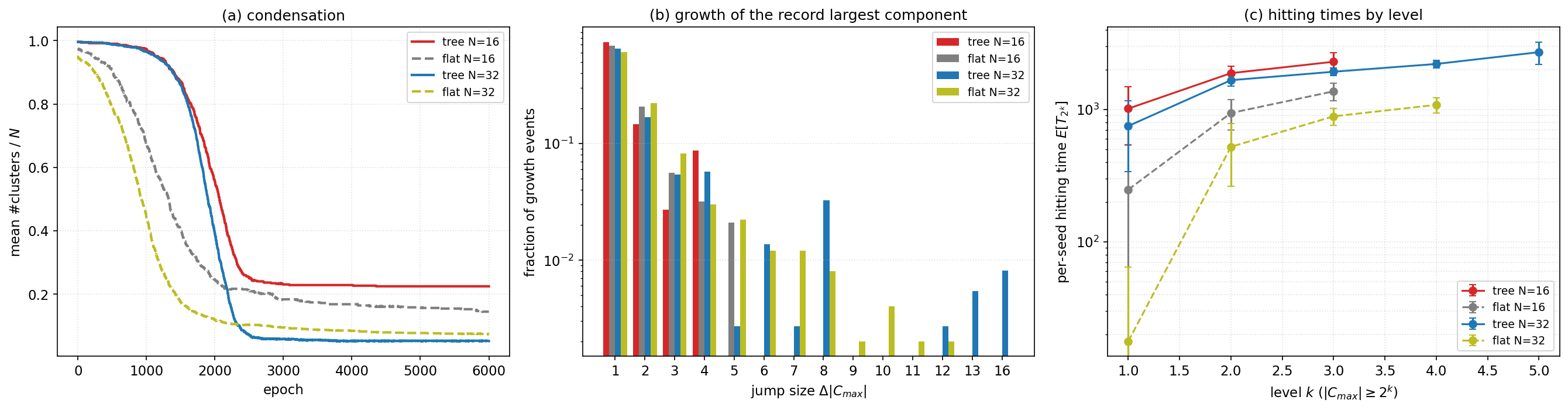}
    \caption{\textbf{Block merges under nested symmetry.} Binary tree (solid) and flat layer (dashed) at $N=16$ and $32$ over $32$ seeds. \textbf{(a)} Only the tree reaches a single cluster. \textbf{(b)} Jumps of $+16$, merging two half-trees, occur only in the tree. \textbf{(c)} Hitting times per level are evenly spaced under weight decay.}
    \label{fig:block_merges}
    \vspace{-0.5em}
\end{figure}
\begin{figure}[t]
    \centering
    \begin{minipage}[c]{0.64\textwidth}
        \centering
        \includegraphics[width=\linewidth]{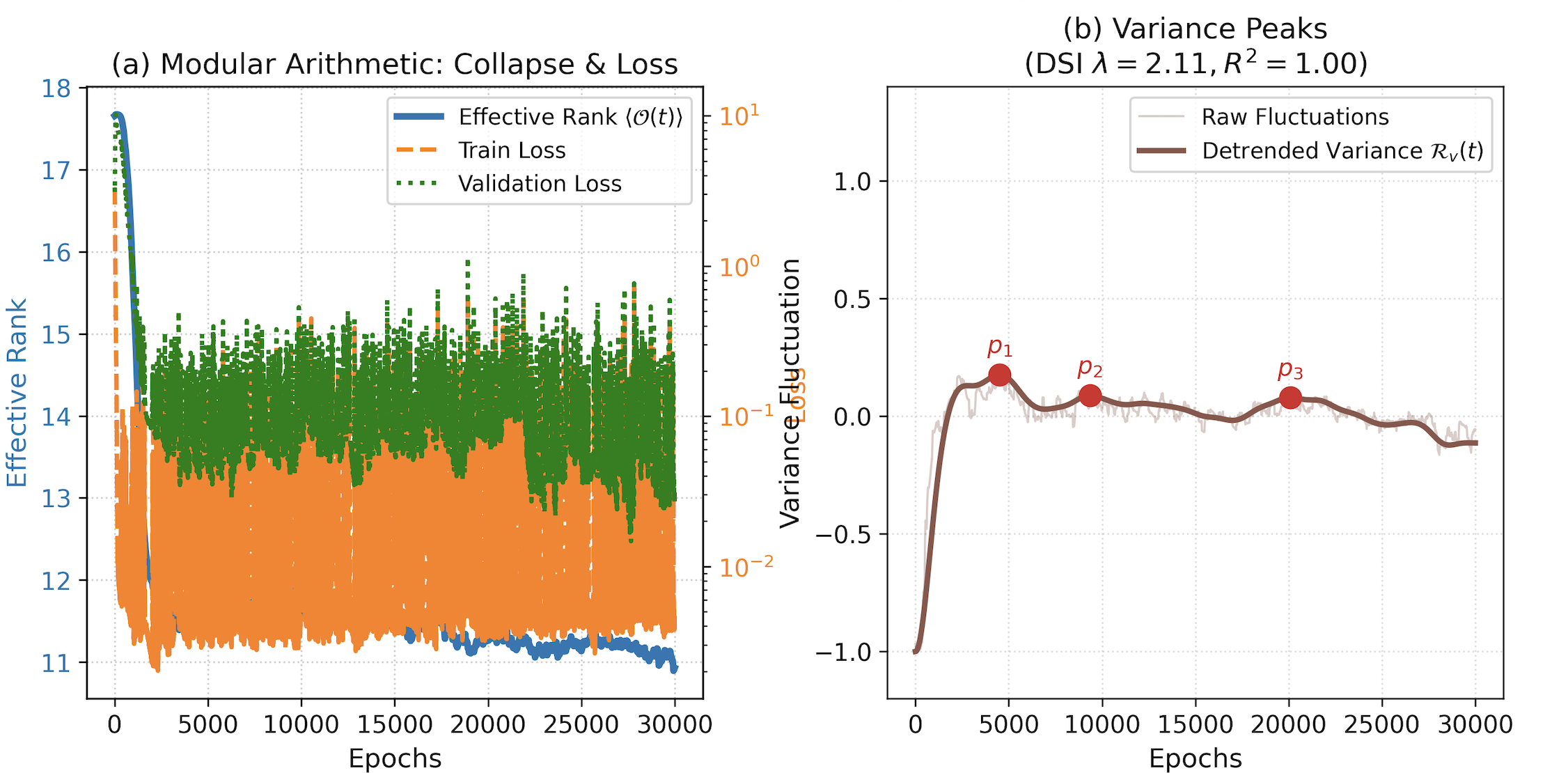}
    \end{minipage}\hfill
    \begin{minipage}[c]{0.33\textwidth}
        \captionsetup{justification=raggedright,singlelinecheck=false}
        \caption{\textbf{Grokking (modular arithmetic).} Delayed generalization co-occurs with dimensionality collapse and a log-linear cascade of variance peaks (epochs $4500$, $9500$, $20000$; $\lambda_t = 2.11$). Other datasets show similar cascades (Appendix~\ref{sec:appendix_empirical_suite}).}
        \label{fig:unified_generalization}
    \end{minipage}
    \vspace{-1em}
\end{figure}

\section{Related Work}
\label{sec:related_work}

\textbf{Stochastic Dynamics and Implicit Bias:} Traditionally SGD is modeled
as driving networks toward low-rank manifolds \citep{blanc2020implicit,
nacson2022implicit, chen2023stochastic} through anomalous diffusion and glassy
dynamics \citep{baityjesi2018comparing, geiger2019jamming, kunin2024limiting,
chen2020anomalous}. We instead track the transient path to these manifolds
through network percolation on a continuous Reeb graph.

\textbf{Graph and Percolation Perspectives on Training:} Recent work applies
percolation tools to training, studying connectivity under dropout
\citep{devlin2025dropout} or synaptic invariants that forecast performance
\citep{li2022capacitance}. Our percolation graph is instead constructed from
a renormalized Reeb graph (Theorem~\ref{thm:main_renormalization}) that
tracks collapse into shared invariant sets (Definition~\ref{def:attractor_equiv}),
with edges driven by architectural symmetry rather than random deletion or
correlation. 

\textbf{Neuron and Weight Condensation:} Condensation serves as a complementary lens on stochastic collapse focusing on initialization scale rather than symmetry
\citep{luo2021phase, zhou2022condensation, zhou2023phase}. Our framework gives an explicit mechanism for the transient dynamics. 

\section{Limitations, Conclusion and Future Work}
\label{sec:conclusion}

We formalize SGD's transient dynamics as a symmetry-induced percolation
process and show under assumptions that it exhibits Discrete Scale Invariance most sharply when the parameters carry the
nested symmetry that deep networks inherit from composition. Future work could test whether the discontinuities persist at practical widths
(Corollary~\ref{cor:microscopic_washout-app}) and larger moduli, and relate the cascades to representation-level measures such as neural collapse. Interesting directions include an extension towards 
curved invariant manifolds, the usage of $\mathcal{R}_v$ peaks to schedule
learning rates and uncovering deeper connections to empirical scaling laws
(Remark~\ref{remark:scaling_laws_stagewise}). Extending the Adam and AdamW measurements to large-scale training is left to future work.

\clearpage
\newpage

\bibliographystyle{plainnat}
\bibliography{references}

\begin{thebibliography}{54}
\providecommand{\natexlab}[1]{#1}
\providecommand{\url}[1]{\texttt{#1}}
\expandafter\ifx\csname urlstyle\endcsname\relax
  \providecommand{\doi}[1]{doi: #1}\else
  \providecommand{\doi}{doi: \begingroup \urlstyle{rm}\Url}\fi

\bibitem[Achlioptas et~al.(2009)Achlioptas, D'Souza, and Spencer]{achlioptas2009explosive}
Dimitris Achlioptas, Raissa~M D'Souza, and Joel Spencer.
\newblock Explosive percolation in random networks.
\newblock \emph{Science}, 323\penalty0 (5920):\penalty0 1453--1455, 2009.

\bibitem[Baity-Jesi et~al.(2018)Baity-Jesi, Sagun, Geiger, Spigler, and Ben~Arous]{baityjesi2018comparing}
Marco Baity-Jesi, Levent Sagun, Mario Geiger, Stefano Spigler, and G{\'e}rard Ben~Arous.
\newblock Comparing dynamics: Deep neural networks versus glassy systems.
\newblock \emph{Journal of Statistical Mechanics}, 2018:\penalty0 033301, 2018.

\bibitem[Blanc et~al.(2020)Blanc, Gupta, Valiant, and Valiant]{blanc2020implicit}
Guy Blanc, Neha Gupta, Gregory Valiant, and Paul Valiant.
\newblock Implicit regularization for deep neural networks driven by an ornstein-uhlenbeck like process.
\newblock In \emph{Conference on Learning Theory}, pages 483--513. PMLR, 2020.

\bibitem[Carlsson(2009)]{carlsson2009topology}
Gunnar Carlsson.
\newblock Topology and data.
\newblock \emph{Bulletin of the American Mathematical Society}, 46\penalty0 (2):\penalty0 255--308, 2009.

\bibitem[Chaudhari and Soatto(2018)]{chaudhari2018stochastic}
Pratik Chaudhari and Stefano Soatto.
\newblock Stochastic gradient descent performs variational inference, converges to limit cycles for deep networks.
\newblock In \emph{International Conference on Learning Representations}, 2018.

\bibitem[Chen et~al.(2023)Chen, Kunin, Yamamura, and Ganguli]{chen2023stochastic}
Feng Chen, Daniel Kunin, Atsushi Yamamura, and Surya Ganguli.
\newblock Stochastic collapse: How gradient noise attracts {SGD} dynamics towards simpler subnetworks.
\newblock In \emph{Thirty-seventh Conference on Neural Information Processing Systems}, 2023.

\bibitem[Chen et~al.(2020)Chen, Qu, and Gong]{chen2020anomalous}
Guozhang Chen, Cheng Qu, and Pulin Gong.
\newblock Anomalous diffusion dynamics of learning in deep neural networks.
\newblock \emph{arXiv preprint arXiv:2009.10588}, 2020.

\bibitem[Chen et~al.(2014)Chen, Schr{\"o}der, D’Souza, Sornette, and Nagler]{chen2014microtransition}
Wei Chen, Malte Schr{\"o}der, Raissa~M D’Souza, Didier Sornette, and Jan Nagler.
\newblock Microtransition cascades to percolation.
\newblock \emph{Physical review letters}, 112\penalty0 (15):\penalty0 155701, 2014.

\bibitem[\c{S}im\c{s}ek et~al.(2021)\c{S}im\c{s}ek, Ged, Jacot, Spadaro, Hongler, Gerstner, and Brea]{simsek2021geometry}
Berfin \c{S}im\c{s}ek, Fran\c{c}ois Ged, Arthur Jacot, Francesco Spadaro, Cl{\'e}ment Hongler, Wulfram Gerstner, and Johanni Brea.
\newblock Geometry of the loss landscape in overparameterized neural networks: Symmetries and invariances.
\newblock In \emph{Proceedings of the 38th International Conference on Machine Learning}, volume 139, pages 9722--9732. PMLR, 2021.

\bibitem[Devlin and Sanders(2025)]{devlin2025dropout}
Finley Devlin and Jaron Sanders.
\newblock Dropout neural network training viewed from a percolation perspective.
\newblock \emph{arXiv preprint arXiv:2512.13853}, 2025.

\bibitem[Edelsbrunner and Harer(2008)]{edelsbrunner2008topological}
Herbert Edelsbrunner and John Harer.
\newblock Topological persistence and simplification.
\newblock \emph{Discrete \& Computational Geometry}, 39\penalty0 (1--3):\penalty0 427--463, 2008.

\bibitem[Freidlin and Wentzell(2012)]{freidlin2012random}
Mark~I Freidlin and Alexander~D Wentzell.
\newblock \emph{Random perturbations of dynamical systems}.
\newblock Springer Science \& Business Media, 2012.

\bibitem[Gardiner(2009)]{gardiner2009stochastic}
Crispin Gardiner.
\newblock \emph{Stochastic Methods: A Handbook for the Natural and Social Sciences}.
\newblock Springer Berlin Heidelberg, 2009.

\bibitem[Geiger et~al.(2019)Geiger, Spigler, d'Ascoli, Sagun, and Baity-Jesi]{geiger2019jamming}
Mario Geiger, Stefano Spigler, St{\'e}phane d'Ascoli, Levent Sagun, and Marco Baity-Jesi.
\newblock The jamming transition as a paradigm to understand the loss landscape of deep neural networks.
\newblock \emph{Physical Review E}, 100:\penalty0 012115, 2019.

\bibitem[Goodfellow et~al.(2013)Goodfellow, Mirza, Xiao, Courville, and Bengio]{goodfellow2013empirical}
Ian~J Goodfellow, Mehdi Mirza, Da~Xiao, Aaron Courville, and Yoshua Bengio.
\newblock An empirical investigation of catastrophic forgetting in gradient-based neural networks.
\newblock \emph{arXiv preprint arXiv:1312.6211}, 2013.

\bibitem[HaoChen et~al.(2021)HaoChen, Wei, Lee, and Ma]{haochen2021shape}
Jeff~Z HaoChen, Colin Wei, Jason Lee, and Tengyu Ma.
\newblock Shape matters: Understanding the implicit bias of the noise covariance.
\newblock In \emph{Conference on Learning Theory}, pages 2315--2357. PMLR, 2021.

\bibitem[Hoffmann et~al.(2022)Hoffmann, Borgeaud, Mensch, Buchatskaya, Cai, Rutherford, Casas, Hendricks, Welbl, Clark, et~al.]{hoffmann2022training}
Jordan Hoffmann, Sebastian Borgeaud, Arthur Mensch, Elena Buchatskaya, Trevor Cai, Eliza Rutherford, Diego de~Las Casas, Lisa~Anne Hendricks, Johannes Welbl, Aidan Clark, et~al.
\newblock Training compute-optimal large language models.
\newblock In \emph{Advances in Neural Information Processing Systems}, volume~35, pages 30016--30030, 2022.

\bibitem[Hoogland et~al.(2024)Hoogland, Wang, Farrugia-Roberts, Carroll, Wei, and Murfet]{hoogland2024loss}
Jesse Hoogland, George Wang, Matthew Farrugia-Roberts, Liam Carroll, Susan Wei, and Daniel Murfet.
\newblock Loss landscape degeneracy and stagewise development in transformers.
\newblock \emph{Transactions on Machine Learning Research}, 2024.

\bibitem[Huh et~al.(2021)Huh, Mobahi, Zhang, Cheung, Agrawal, and Isola]{huh2021lowrank}
Minyoung Huh, Hossein Mobahi, Richard Zhang, Brian Cheung, Pulkit Agrawal, and Phillip Isola.
\newblock The low-rank simplicity bias in deep networks.
\newblock \emph{Transactions on Machine Learning Research}, 2021.
\newblock URL \url{https://openreview.net/forum?id=3a0Cbtb1Hs}.

\bibitem[Kaplan et~al.(2020)Kaplan, McCandlish, Henighan, Brown, Chess, Child, Gray, Radford, Wu, and Amodei]{kaplan2020scaling}
Jared Kaplan, Sam McCandlish, Tom Henighan, Tom~B Brown, Benjamin Chess, Rewon Child, Scott Gray, Alec Radford, Jeffrey Wu, and Dario Amodei.
\newblock Scaling laws for neural language models.
\newblock \emph{arXiv preprint arXiv:2001.08361}, 2020.

\bibitem[Karatzas and Shreve(1991)]{karatzas1991brownian}
Ioannis Karatzas and Steven~E. Shreve.
\newblock \emph{Brownian Motion and Stochastic Calculus}, volume 113 of \emph{Graduate Texts in Mathematics}.
\newblock Springer-Verlag, 2nd edition, 1991.

\bibitem[Kingma and Ba(2015)]{kingma2015adam}
Diederik~P. Kingma and Jimmy Ba.
\newblock Adam: A method for stochastic optimization.
\newblock In \emph{International Conference on Learning Representations}, 2015.

\bibitem[Kirkpatrick et~al.(2017)Kirkpatrick, Pascanu, Rabinowitz, Veness, Desjardins, Rusu, Milan, Quan, Ramalho, Grabska-Barwinska, et~al.]{kirkpatrick2017overcoming}
James Kirkpatrick, Razvan Pascanu, Neil Rabinowitz, Joel Veness, Guillaume Desjardins, Andrei~A Rusu, Kieran Milan, John Quan, Tiago Ramalho, Agnieszka Grabska-Barwinska, et~al.
\newblock Overcoming catastrophic forgetting in neural networks.
\newblock \emph{Proceedings of the national academy of sciences}, 114\penalty0 (13):\penalty0 3521--3526, 2017.

\bibitem[Kunin et~al.(2024)Kunin, Sagastuy-Bre{\~n}a, Gillespie, Margalit, and Tanaka]{kunin2024limiting}
Daniel Kunin, Javier Sagastuy-Bre{\~n}a, Lauren Gillespie, Eshed Margalit, and Hidenori Tanaka.
\newblock The limiting dynamics of {SGD}: Modified loss, phase-space oscillations, and anomalous diffusion.
\newblock \emph{Neural Computation}, 36:\penalty0 151--197, 2024.

\bibitem[Li et~al.(2017)Li, Tai, and E]{li2017stochastic}
Qianxiao Li, Cheng Tai, and Weinan E.
\newblock Stochastic modified equations and adaptive stochastic gradient algorithms.
\newblock In \emph{International Conference on Machine Learning}, pages 2101--2110. PMLR, 2017.

\bibitem[Li et~al.(2022{\natexlab{a}})Li, Ma, Lu, Chen, and Song]{li2022capacitance}
Yang Li, Kai Ma, Han Lu, Xin Chen, and Renjie Song.
\newblock Neural capacitance: A new perspective of neural network selection via edge dynamics.
\newblock \emph{arXiv preprint arXiv:2201.04194}, 2022{\natexlab{a}}.

\bibitem[Li et~al.(2022{\natexlab{b}})Li, Wang, and Arora]{li2022what}
Zhiyuan Li, Tianhao Wang, and Sanjeev Arora.
\newblock What happens after sgd reaches zero loss? -a mathematical framework.
\newblock In \emph{International Conference on Learning Representations}, 2022{\natexlab{b}}.

\bibitem[Liu et~al.(2022)Liu, Kitouni, Nolte, Michaud, Tegmark, and Williams]{liu2022towards}
Ziming Liu, Ouail Kitouni, Niklas Nolte, Eric Michaud, Max Tegmark, and Mike Williams.
\newblock Towards understanding grokking: An effective theory of representation learning.
\newblock In \emph{Advances in Neural Information Processing Systems}, volume~35, pages 34651--34663, 2022.

\bibitem[Loshchilov and Hutter(2019)]{loshchilov2019decoupled}
Ilya Loshchilov and Frank Hutter.
\newblock Decoupled weight decay regularization.
\newblock In \emph{International Conference on Learning Representations}, 2019.

\bibitem[Luo et~al.(2021)Luo, Xu, Ma, and Zhang]{luo2021phase}
Tao Luo, Zhi-Qin~John Xu, Zheng Ma, and Yaoyu Zhang.
\newblock Phase diagram for two-layer {ReLU} neural networks at infinite-width limit.
\newblock \emph{Journal of Machine Learning Research}, 22\penalty0 (71):\penalty0 1--47, 2021.

\bibitem[Mandt et~al.(2017)Mandt, Hoffman, and Blei]{mandt2017stochastic}
Stephan Mandt, Matthew~D Hoffman, and David~M Blei.
\newblock Stochastic gradient descent as approximate bayesian inference.
\newblock \emph{The Journal of Machine Learning Research}, 18\penalty0 (1):\penalty0 4873--4907, 2017.

\bibitem[Mori et~al.(2022)Mori, Ziyin, Liu, and Ueda]{mori2022logarithmic}
Takashi Mori, Liu Ziyin, Kangqiao Liu, and Masahito Ueda.
\newblock Logarithmic landscape and power-law escape rate of {SGD}.
\newblock In \emph{International Conference on Machine Learning}, pages 15959--15978. PMLR, 2022.

\bibitem[Nacson et~al.(2022)Nacson, Ravichandran, Srebro, and Soudry]{nacson2022implicit}
Mor~Shpigel Nacson, Kavya Ravichandran, Nathan Srebro, and Daniel Soudry.
\newblock Implicit bias of the step size in linear diagonal neural networks.
\newblock In \emph{International Conference on Machine Learning}, pages 16270--16295. PMLR, 2022.

\bibitem[Naitzat et~al.(2020)Naitzat, Zhitnikov, and Lim]{naitzat2020topology}
Gregory Naitzat, Andrey Zhitnikov, and Lek-Heng Lim.
\newblock Topology of data in deep learning.
\newblock \emph{Frontiers in Applied Mathematics and Statistics}, 6:\penalty0 554449, 2020.

\bibitem[Nanda et~al.(2023)Nanda, Chan, Lieberum, Smith, and Steinhardt]{nanda2022progress}
Neel Nanda, Lawrence Chan, Tom Lieberum, Jess Smith, and Jacob Steinhardt.
\newblock Progress measures for grokking via mechanistic interpretability.
\newblock In \emph{The Eleventh International Conference on Learning Representations}, 2023.

\bibitem[Oppenheim and Schafer(2010)]{oppenheim2010discrete}
Alan~V. Oppenheim and Ronald~W. Schafer.
\newblock \emph{Discrete-Time Signal Processing}.
\newblock Prentice Hall, Upper Saddle River, NJ, 3rd edition, 2010.

\bibitem[Papoulis and Pillai(2002)]{papoulis2002probability}
Athanasios Papoulis and S.~Unnikrishna Pillai.
\newblock \emph{Probability, Random Variables, and Stochastic Processes}.
\newblock McGraw-Hill, New York, 4th edition, 2002.

\bibitem[Power et~al.(2022)Power, Burda, Edwards, Babuschkin, and Misra]{power2022grokking}
Alethea Power, Yuri Burda, Harri Edwards, Igor Babuschkin, and Vedant Misra.
\newblock Grokking: Generalization beyond overfitting on small algorithmic datasets.
\newblock In \emph{International Conference on Learning Representations}, 2022.
\newblock URL \url{https://openreview.net/forum?id=9Vrb9D0WI4}.

\bibitem[Ramachandran and Sra(2026)]{ramachandrantrees}
Sai~Niranjan Ramachandran and Suvrit Sra.
\newblock Trees to flows and back: Unifying decision trees and diffusion models.
\newblock In \emph{Forty-third International Conference on Machine Learning}, 2026.

\bibitem[Revuz and Yor(1999)]{revuz1999continuous}
Daniel Revuz and Marc Yor.
\newblock \emph{Continuous martingales and Brownian motion}, volume 293.
\newblock Springer Science \& Business Media, 1999.

\bibitem[Riordan and Warnke(2011)]{riordan2011explosive}
Oliver Riordan and Lutz Warnke.
\newblock Explosive percolation is continuous.
\newblock \emph{Science}, 333:\penalty0 322--324, 2011.

\bibitem[Risken(1996)]{risken1996fokker}
Hannes Risken.
\newblock \emph{The Fokker-Planck Equation: Methods of Solution and Applications}.
\newblock Springer-Verlag, 1996.

\bibitem[Samorodnitsky and Taqqu(1994)]{samorodnitsky1994stable}
Gennady Samorodnitsky and Murad~S. Taqqu.
\newblock \emph{Stable Non-{G}aussian Random Processes: Stochastic Models with Infinite Variance}.
\newblock Chapman \& Hall, New York, 1994.

\bibitem[Saxe et~al.(2013)Saxe, McClelland, and Ganguli]{saxe2013exact}
Andrew~M. Saxe, James~L. McClelland, and Surya Ganguli.
\newblock Exact solutions to the nonlinear dynamics of learning in deep linear neural networks.
\newblock \emph{arXiv preprint arXiv:1312.6120}, 2013.

\bibitem[Shah et~al.(2020)Shah, Tamuly, Raghunathan, Jain, and Netrapalli]{shah2020pitfalls}
Harshay Shah, Kaustubh Tamuly, Aditi Raghunathan, Prateek Jain, and Praneeth Netrapalli.
\newblock The pitfalls of simplicity bias in neural networks.
\newblock In \emph{Advances in Neural Information Processing Systems}, volume~33, pages 9573--9585, 2020.

\bibitem[Sornette(1998)]{sornette1998discrete}
Didier Sornette.
\newblock Discrete scale invariance and complex dimensions.
\newblock \emph{Physics Reports}, 297\penalty0 (5):\penalty0 239--270, 1998.

\bibitem[Stauffer and Aharony(1994)]{stauffer1994introduction}
Dietrich Stauffer and Ammon Aharony.
\newblock \emph{Introduction to Percolation Theory}.
\newblock Taylor \& Francis, 1994.

\bibitem[Thilak et~al.(2022)Thilak, Littwin, Zhai, Saremi, Joshua, and Susskind]{thilak2022slingshot}
Vimal Thilak, Etai Littwin, Shuanghai Zhai, Omid Saremi, Roni Joshua, and Leonard Susskind.
\newblock The slingshot mechanism: An empirical study of adaptive optimizers and the grokking phenomenon.
\newblock \emph{arXiv preprint arXiv:2206.04817}, 2022.

\bibitem[Wei et~al.(2008)Wei, Zhang, Cousseau, Ozeki, and Amari]{wei2008dynamics}
Haikun Wei, Jun Zhang, Florent Cousseau, Tomoko Ozeki, and Shun-ichi Amari.
\newblock Dynamics of learning near singularities in layered networks.
\newblock \emph{Neural Computation}, 20\penalty0 (3):\penalty0 813--843, 2008.

\bibitem[Wei et~al.(2022)Wei, Murfet, Gong, Li, Gell-Redman, and Quella]{wei2022deep}
Susan Wei, Daniel Murfet, Ming Gong, Huan Li, Jesse Gell-Redman, and Thomas Quella.
\newblock Deep learning is singular, and that's good.
\newblock \emph{IEEE Transactions on Neural Networks and Learning Systems}, 34\penalty0 (12):\penalty0 10473--10486, 2022.

\bibitem[Zhang et~al.(2020)Zhang, Karimireddy, Veit, Kim, Reddi, Kumar, and Sra]{zhang2020adaptive}
Jingzhao Zhang, Sai~Praneeth Karimireddy, Andreas Veit, Seungyeon Kim, Sashank~J. Reddi, Sanjiv Kumar, and Suvrit Sra.
\newblock Why are adaptive methods good for attention models?
\newblock In \emph{Advances in Neural Information Processing Systems}, volume~33, 2020.

\bibitem[Zhou et~al.(2022)Zhou, Zhou, Luo, Zhang, and Xu]{zhou2022condensation}
Hanxu Zhou, Qixuan Zhou, Tao Luo, Yaoyu Zhang, and Zhi-Qin~John Xu.
\newblock Towards understanding the condensation of neural networks at initial training.
\newblock In \emph{Advances in Neural Information Processing Systems}, volume~35, 2022.

\bibitem[Zhou et~al.(2023)Zhou, Zhou, Jin, Luo, Zhang, and Xu]{zhou2023phase}
Hanxu Zhou, Qixuan Zhou, Zhenyuan Jin, Tao Luo, Yaoyu Zhang, and Zhi-Qin~John Xu.
\newblock Phase diagram of initial condensation for two-layer neural networks.
\newblock \emph{arXiv preprint arXiv:2303.06561}, 2023.

\bibitem[Ziyin et~al.(2022)Ziyin, Li, Simon, and Ueda]{ziyin2022sgd}
Liu Ziyin, Botao Li, James~B Simon, and Masahito Ueda.
\newblock {SGD} with a constant large learning rate can converge to local maxima.
\newblock In \emph{International Conference on Learning Representations}, 2022.

\end{thebibliography}

\newpage
\tableofcontents
\newpage
\appendix

\section{SDE Formulations and Invariant Sets of Learning Dynamics}
\label{app:sde_and_invariance}

To analyze the transient dynamics of stochastic collapse, we formalize the relationship between discrete Stochastic Gradient Descent (SGD) and continuous Stochastic Gradient Flow (SGF). We then formally define invariant sets for these continuous processes and prove that affine invariant sets in SGD strictly translate to SGF. We adapt this section from \cite{chen2023stochastic}.

\subsection{Deriving Stochastic Gradient Flow from SGD}
\label{app:sgd_sgf_relation}

Recall the standard SGD update rule. At iteration $t$, optimizing a network parameterized by $\boldsymbol{\theta} \in \mathbb{R}^d$ with learning rate $\eta$ over a mini-batch $\mathcal{B}_t$ of size $B$ drawn from a dataset of size $N_{\mathcal{D}}$ yields:
\begin{equation}
    \boldsymbol{\theta}_{t+1} = \boldsymbol{\theta}_t - \frac{\eta}{B} \sum_{i \in \mathcal{B}_t} \nabla \ell(\boldsymbol{\theta}_t; x_i, y_i).
    \label{eq:sgd_update}
\end{equation}

We factorize the right-hand side to explicitly isolate the exact full-batch population gradient $\nabla \mathcal{L}(\boldsymbol{\theta}_t)$ from the zero-mean mini-batch deviation:
\begin{equation*}
    \boldsymbol{\theta}_{t+1} = \boldsymbol{\theta}_t - \eta \nabla \mathcal{L}(\boldsymbol{\theta}_t) - \sqrt{\eta} \left( \sqrt{\eta} \frac{1}{B} \sum_{i \in \mathcal{B}_t} \Big( \nabla \ell(\boldsymbol{\theta}_t; x_i, y_i) - \nabla \mathcal{L}(\boldsymbol{\theta}_t) \Big) \right).
\end{equation*}

Let $\xi_{\mathcal{B}}(\boldsymbol{\theta})$ represent the gradient noise at position $\boldsymbol{\theta}$. Assuming we sample mini-batches with replacement, we disentangle this batch noise into a sum of independent, identically distributed per-sample gradient deviations: $\xi_{\mathcal{B}}(\boldsymbol{\theta}) = \frac{\sqrt{\eta}}{B} \sum_{i \in \mathcal{B}_t} \xi_i$, where $\xi_i = \nabla \ell(\boldsymbol{\theta}; x_i, y_i) - \nabla \mathcal{L}(\boldsymbol{\theta})$. By definition, the expectation $\mathbb{E}[\xi_i(\boldsymbol{\theta})] = 0$ for all $\boldsymbol{\theta}$.

Assuming these gradient noises accumulate as Gaussian random variables, we approximate the discrete step as:
\begin{equation*}
    \Delta \boldsymbol{\theta}_t = -\eta \nabla \mathcal{L}(\boldsymbol{\theta}_t) + \sqrt{2 D(\boldsymbol{\theta}_t)} \Delta W,
\end{equation*}
where $\Delta W \sim \mathcal{N}(0, \eta I)$ and the diffusion matrix is defined as $D(\boldsymbol{\theta}) = \frac{1}{2} \text{Var}(\xi_{\mathcal{B}}) = \frac{\eta}{2B} \text{Var}(\xi_i)$. This discrete step functions as the Euler-Maruyama discretization of the following continuous-time It\^o Stochastic Differential Equation (SDE), which we formally denote as Stochastic Gradient Flow (SGF):
\begin{equation}
    d\boldsymbol{\theta}_t = -\nabla \mathcal{L}(\boldsymbol{\theta}_t) dt + \sqrt{2 D(\boldsymbol{\theta}_t)} dW_t, \quad \boldsymbol{\theta}_0 = \boldsymbol{\theta}(0).
    \label{eq:sgf_sde}
\end{equation}

We express the position-dependent diffusion matrix exactly as:
\begin{equation*}
    D(\boldsymbol{\theta}_t) = \frac{\eta}{2B} \left( \frac{1}{N_{\mathcal{D}}} \sum_{i=1}^{N_{\mathcal{D}}} \Big( \nabla \ell(\boldsymbol{\theta}_t; x_i, y_i) - \nabla \mathcal{L}(\boldsymbol{\theta}_t) \Big) \Big( \nabla \ell(\boldsymbol{\theta}_t; x_i, y_i) - \nabla \mathcal{L}(\boldsymbol{\theta}_t) \Big)^\top \right).
\end{equation*}

This formulation explicitly decomposes the diffusion matrix into a product of a parameter-independent magnitude scalar $D_m$ and a parameter-dependent shape matrix $D_s(\boldsymbol{\theta})$ such that $D(\boldsymbol{\theta}_t) = D_m D_s(\boldsymbol{\theta}_t)$, where:
\begin{align*}
    D_m &= \frac{\eta}{2B}, \\
    D_s(\boldsymbol{\theta}) &= \frac{1}{N_{\mathcal{D}}} \sum_{i=1}^{N_{\mathcal{D}}} \Big( \nabla \ell(\boldsymbol{\theta}_t; x_i, y_i) - \nabla \mathcal{L}(\boldsymbol{\theta}_t) \Big) \Big( \nabla \ell(\boldsymbol{\theta}_t; x_i, y_i) - \nabla \mathcal{L}(\boldsymbol{\theta}_t) \Big)^\top.
\end{align*}
Optimization hyperparameters, specifically learning rate and batch size, strictly control the scalar magnitude $D_m$. Conversely, the network architecture, training dataset, and loss topology completely dictate the geometric shape $D_s(\boldsymbol{\theta})$.

The matrix $2D$ factors as $2D(\boldsymbol{\theta}) = G(\boldsymbol{\theta}) G(\boldsymbol{\theta})^\top$ with the $d \times N_{\mathcal{D}}$ noise factor
\begin{equation*}
    G(\boldsymbol{\theta}) := \sqrt{\frac{\eta}{B N_{\mathcal{D}}}}\,\Big[\nabla \bar{\ell}(\boldsymbol{\theta}; x_1, y_1),\ \dots,\ \nabla \bar{\ell}(\boldsymbol{\theta}; x_{N_{\mathcal{D}}}, y_{N_{\mathcal{D}}})\Big], \qquad \bar{\ell}(\boldsymbol{\theta}; x_i, y_i) := \ell(\boldsymbol{\theta}; x_i, y_i) - \mathcal{L}(\boldsymbol{\theta}),
\end{equation*}
and we drive the SGF by an $N_{\mathcal{D}}$-dimensional Brownian motion through $G$, $d\boldsymbol{\theta}_t = -\nabla \mathcal{L}(\boldsymbol{\theta}_t) dt + G(\boldsymbol{\theta}_t) dW_t$. The law of the solution depends only on $GG^\top = 2D$, so the choice of factor does not affect any statement that follows. The factor itself is the reason the coefficients are Lipschitz.

\subsection{Invariant Sets of Continuous Processes}
\label{app:invariant_sets}

To mathematically track how trajectories collapse, we formalize the topological regions that permanently trap optimization dynamics.

\begin{definition}[Invariant Sets of Continuous Processes]
For a given stochastic process $\{\boldsymbol{\theta}_t \in \mathbb{R}^d : t \geq 0\}$, a Borel-measurable set $A \subset \mathbb{R}^d$ acts as an invariant set if, for every initial point $\boldsymbol{\theta}_0$ strictly inside $A$, the probability that the process remains in $A$ for all subsequent time $t \geq 0$ equals 1. Formally:
\begin{equation}
    P\big[\boldsymbol{\theta}_t \in A \text{ for all } t \ge 0 \mid \boldsymbol{\theta}_0 \in A\big] = 1.
    \label{eq:invariant_prob}
\end{equation}
\end{definition}

\begin{proposition}[Affine Invariance Transfer]
\label{prop:affine_invariance}
Consider a discrete SGD process given by Equation \eqref{eq:sgd_update}, where the individual sample gradients $\nabla \ell(\cdot; x_i, y_i) : \mathbb{R}^d \to \mathbb{R}^d$ remain Lipschitz continuous and bounded for any $i \in [N_{\mathcal{D}}]$. Let an affine subset $A \subseteq \mathbb{R}^d$ be invariant for every sample gradient step, that is, $P\nabla\ell(\boldsymbol{\theta}; x_i, y_i) = \nabla\ell(\boldsymbol{\theta}; x_i, y_i)$ for all $i \in [N_{\mathcal{D}}]$ and all $\boldsymbol{\theta} \in A$, where $P$ is the orthogonal projection onto the linear subspace parallel to $A$. Then $A$ is a strict invariant set for the continuous SGF process given by Equation \eqref{eq:sgf_sde}.
\end{proposition}

\begin{proof}
By our foundational assumption, the individual sample gradients act as $L$-Lipschitz continuous in $\boldsymbol{\theta}$ and remain bounded such that $\|\nabla \ell(\boldsymbol{\theta}; x_i, y_i)\| \le L$ for some constant $L > 0$. We first demonstrate that the population gradient $\nabla \mathcal{L}(\boldsymbol{\theta})$ inherits these exact structural properties. Bounding the difference yields:
\begin{equation*}
    \|\nabla \mathcal{L}(\boldsymbol{\theta}_1) - \nabla \mathcal{L}(\boldsymbol{\theta}_2)\| = \left\| \frac{1}{N_{\mathcal{D}}} \sum_{i=1}^{N_{\mathcal{D}}} \Big(\nabla \ell(\boldsymbol{\theta}_1; x_i, y_i) - \nabla \ell(\boldsymbol{\theta}_2; x_i, y_i)\Big) \right\| \le L \|\boldsymbol{\theta}_1 - \boldsymbol{\theta}_2\|.
\end{equation*}
Similarly, the global norm remains bounded by $\|\nabla \mathcal{L}(\boldsymbol{\theta})\| \le \frac{1}{N_{\mathcal{D}}} \sum_{i=1}^{N_{\mathcal{D}}} \|\nabla \ell(\boldsymbol{\theta}; x_i, y_i)\| \le L$.

The centered sample gradient $\nabla \bar{\ell}(\cdot; x_i, y_i)$ is $2L$-Lipschitz continuous and bounded by $2L$ for any index $i \in [N_{\mathcal{D}}]$, by the triangle inequality. Hence every column of the noise factor $G(\boldsymbol{\theta})$ of Section \ref{app:sgd_sgf_relation} is $2L\sqrt{\eta/(B N_{\mathcal{D}})}$-Lipschitz, and
\begin{equation*}
    \|G(\boldsymbol{\theta}_1) - G(\boldsymbol{\theta}_2)\|_F \le 2L\sqrt{\frac{\eta}{B}}\, \|\boldsymbol{\theta}_1 - \boldsymbol{\theta}_2\|.
\end{equation*}
Drift and diffusion coefficient are therefore globally Lipschitz, which guarantees the existence of a unique strong solution to the SDE in Equation \eqref{eq:sgf_sde} driven by $G$. We work with $G$ rather than with the symmetric root $\sqrt{2D}$ because the Powers--St{\o}rmer inequality, $\|\sqrt{2D_1} - \sqrt{2D_2}\|_F^2 \le \|2D_1 - 2D_2\|_1$, makes the symmetric root only H\"older-$\tfrac12$ in $\boldsymbol{\theta}$ wherever $D$ degenerates, and $D$ degenerates in the transverse directions on $A$ itself, where every sample gradient is tangent to $A$. H\"older-$\tfrac12$ diffusion coefficients do not guarantee pathwise uniqueness in dimension $d > 1$, and the next step requires pathwise uniqueness.

It now suffices to prove that this continuous SDE possesses a solution $\{\boldsymbol{\theta}_t : t \ge 0\}$ that satisfies the invariant probability condition from Equation \eqref{eq:invariant_prob}. We define $P : \mathbb{R}^d \to \mathbb{R}^d$ as the orthogonal projection onto the linear subspace $V$ parallel to $A$ and construct an auxiliary projected SDE started in $A$:
\begin{equation}
    d\tilde{\boldsymbol{\theta}}_t = -P\nabla \mathcal{L}(\tilde{\boldsymbol{\theta}}_t)dt + P G(\tilde{\boldsymbol{\theta}}_t) dW_t.
\end{equation}
Its increments lie in $V$, so this projected system confines its solution $\tilde{\boldsymbol{\theta}}_t$ within $A$, satisfying the invariance condition, and its coefficients are Lipschitz, so its solution is unique. By hypothesis every sample gradient is tangent to $A$ on $A$, hence $P\nabla\mathcal{L} = \nabla\mathcal{L}$ and $PG = G$ on $A$, column by column. The projected SDE and the original SDE therefore have identical coefficients along any path in $A$, and pathwise uniqueness identifies their solutions from any $\boldsymbol{\theta}_0 \in A$ almost surely. Hence the affine subset $A$ acts as a strict invariant set for the continuous SGF process.
\end{proof}

\section{Mechanics of Topological Transitions}
\label{sec:appendix_topological_mechanics}

Following the subnetwork partition (Definition \ref{def:subnet-part}) established in the main text, we isolate the local dynamics. The constructions below use the transverse distance $Y^{(i)}$ of each subnetwork to a target set $A$ and the transverse distance $Y^{(ij)}$ of each pair to its coincidence set $A_{ij}$. Both are well defined for any coupling between blocks; the vanishing of the pairwise transverse noise on $A_{ij}$ comes from per-sample symmetry, which forces $\Sigma_{ij} \to \Sigma_{ii}$ there.

To formalize how these path measures structurally bind, we project the local SDEs into a shared quotient space by rigorously defining the transverse geometry.

\subsection{Formalizing the Transverse Process}

\begin{definition}[Transverse Distance Process]
Let $A \subset \mathbb{R}^{d_i}$ define an invariant manifold of the $i$-th block. Define the orthogonal projection operator $\pi_A: \mathbb{R}^{d_i} \to A$. We establish a local tubular neighborhood $\mathcal{N}(A, \epsilon) = \{\boldsymbol{\theta} \in \mathbb{R}^d : \|\boldsymbol{\theta}^{(i)} - \pi_A(\boldsymbol{\theta}^{(i)})\|_2^2 < \epsilon\}$ wherein this projection remains uniquely defined. Within this boundary, we map the local dynamics into the quotient space $\mathbb{R}^{d_i} / A$ by continuously tracking the squared transverse distance functional for the $i$-th subnetwork:
\begin{equation}
    Y_t^{(i)} = \|\boldsymbol{\theta}_t^{(i)} - \pi_A(\boldsymbol{\theta}_t^{(i)})\|_2^2.
\end{equation}
\end{definition}

Mechanically, stochastic attractivity pulls the inward deterministic gradient to overcome the transverse diffusion. We map this directly to the infinitesimal generator $\mathscr{G}$ of the SDE. Write $\boldsymbol{x}_t = \boldsymbol{\theta}_t^{(i)} - \pi_A(\boldsymbol{\theta}_t^{(i)})$ and $\Sigma_\perp = P^\perp \Sigma_{ii} P^\perp$. Applying It\^o's Lemma to $Y_t^{(i)}$ and to $\phi_\alpha(Y) = Y^\alpha$, $\alpha \in (0,1]$, expands the generators as
\begin{align}
    \mathscr{G} Y_t^{(i)} &= -2\, \boldsymbol{x}_t^\top \nabla_{\boldsymbol{\theta}^{(i)}} \mathcal{L}(\boldsymbol{\theta}_t) + \eta \operatorname{tr}\Sigma_\perp(\boldsymbol{\theta}_t), \label{eq:generator_Y}\\
    \mathscr{G} (Y_t^{(i)})^\alpha &= \alpha (Y_t^{(i)})^{\alpha-1}\Big[\mathscr{G} Y_t^{(i)} - 2(1-\alpha)\,\eta\, \frac{\boldsymbol{x}_t^\top \Sigma_\perp(\boldsymbol{\theta}_t) \boldsymbol{x}_t}{Y_t^{(i)}}\Big], \label{eq:generator}
\end{align}
the second term in the bracket being $\tfrac12 \phi_\alpha''(Y)\, \nabla Y^\top (\eta\Sigma_{ii}) \nabla Y$ with $\nabla Y = 2\boldsymbol{x}$. Only the transverse block of $\Sigma_{ii}$ enters, since $\nabla Y$ and the Hessian $2P^\perp$ of $Y$ both live in $A^\perp$.

\begin{lemma}[Boundedness of the Stopped Transverse Process]
\label{lemma:stopped_bounded}
Suppose $\boldsymbol{\theta}_t$ solves the SDE in Equation \eqref{eq:sgf_sde} with (locally) Lipschitz
coefficients, so that it admits a strong solution with almost surely continuous sample paths.
Then $Y_t^{(i)} = \|\boldsymbol{\theta}_t^{(i)} - \pi_A(\boldsymbol{\theta}_t^{(i)})\|_2^2$ is a continuous
functional of $\boldsymbol{\theta}_t^{(i)}$, and consequently the stopped process satisfies
\begin{equation}
    Y_{s \wedge \tau_\epsilon}^{(i)} \in [0, \epsilon], \qquad (Y_{s \wedge \tau_\epsilon}^{(i)})^\alpha \in [0, \epsilon^\alpha] \quad \text{a.s. for all } s \ge t_0.
\end{equation}
\end{lemma}
\begin{proof}
Continuity of $\boldsymbol{\theta}_t^{(i)}$ and of the projection $\pi_A$ (an affine
map with constant linear part $P$, by Theorem \ref{thm:affine_generation}) imply $Y_t^{(i)}$ is a.s.
continuous in $t$. By definition $\tau_\epsilon = \inf\{t \ge t_0: Y_t^{(i)} \ge \epsilon\}$;
continuity of the path forces $Y_t^{(i)} < \epsilon$ for $t < \tau_\epsilon$ and
$Y_{\tau_\epsilon}^{(i)} = \epsilon$ whenever $\tau_\epsilon < \infty$ — a continuous path cannot
jump over the boundary. Hence $Y_{s\wedge\tau_\epsilon}^{(i)} \le \epsilon$ for every $s$, and
non-negativity of $Y^{(i)}$ gives the lower bound. Monotonicity of $y \mapsto y^\alpha$ transfers both bounds to the power.
\end{proof}

\begin{theorem}[Local Supermartingale Trapping]
\label{thm:supermartingale-app}
Given a stochastically attractive invariant set $A$ of order $\alpha$, there exists an $\epsilon > 0$ defining the basin $\mathcal{N}(A, \epsilon)$ such that the stopped transverse process $(Y_{t \wedge \tau_\epsilon}^{(i)})^\alpha$ operates as a non-negative supermartingale, where $\tau_\epsilon = \inf\{t \ge t_0 : Y_t^{(i)} \ge \epsilon\}$ marks the exact boundary of escape.
\end{theorem}

\begin{proof}
Let $A$ be stochastically attractive of order $\alpha$. By Definition \ref{def:stoch-attac}, the bracket in Equation \eqref{eq:generator} is non-positive on $\mathcal{N}(A,\epsilon) \setminus A$, so $\mathscr{G}(Y_t^{(i)})^\alpha \le 0$ prior to the stopping time $\tau_\epsilon$ wherever $Y_t^{(i)} > 0$. Since $A$ is invariant (Proposition \ref{prop:affine_invariance}), a path that reaches $Y = 0$ stays there. For $\delta > 0$ let $\sigma_\delta := \inf\{t \ge t_0 : Y_t^{(i)} \le \delta\}$. On $[t_0, \tau_\epsilon \wedge \sigma_\delta]$ the map $\phi_\alpha$ is $C^2$, so we argue for the process stopped at $\tau_\epsilon \wedge \sigma_\delta$ and let $\delta \downarrow 0$ at the end.

Applying It\^o's formula to the process stopped at $\tau_\epsilon \wedge \sigma_\delta$ decomposes the trajectory into a Lebesgue drift integral and a stochastic integral:
\begin{equation}
    (Y_{t \wedge \tau_\epsilon}^{(i)})^\alpha = (Y_{t_0}^{(i)})^\alpha + \int_{t_0}^{t \wedge \tau_\epsilon} \mathscr{G} (Y_s^{(i)})^\alpha\, ds + \int_{t_0}^{t \wedge \tau_\epsilon} 2\alpha (Y_s^{(i)})^{\alpha-1}\, \boldsymbol{x}_s^\top G_\perp(\boldsymbol{\theta}_s)\, dW_s,
\end{equation}
with $G_\perp = P^\perp G_i$ the transverse part of the rows $G_i$ of the noise factor of Appendix \ref{app:sde_and_invariance} that belong to $\boldsymbol{\theta}^{(i)}$. Because $\mathscr{G} (Y_s^{(i)})^\alpha \le 0$, the Lebesgue integral accumulates non-positive drift. On $A$ every sample gradient is tangent to $A$, so $P^\perp\nabla\bar\ell(\boldsymbol{\theta}) = P^\perp\big(\nabla\bar\ell(\boldsymbol{\theta}) - \nabla\bar\ell(\pi_A\boldsymbol{\theta})\big)$ and $\|G_\perp(\boldsymbol{\theta})\| \le 2L\sqrt{\eta/B}\,\sqrt{Y}$. The transverse noise amplitude therefore vanishes on $A$ at Lipschitz rate. The integrand of the stochastic integral is therefore bounded by $2\alpha Y^{\alpha-1}\sqrt{Y}\cdot 2L\sqrt{\eta/B}\sqrt{Y} = 4\alpha L\sqrt{\eta/B}\, Y^\alpha$, which is at most $4\alpha L\sqrt{\eta/B}\,\epsilon^\alpha$ by Lemma \ref{lemma:stopped_bounded}. A stochastic integral with a bounded, progressively measurable integrand is a genuine $L^2$-martingale rather than only a local one \citep{karatzas1991brownian}; taking the conditional expectation $\mathbb{E}[\cdot \mid \mathcal{F}_t]$ annihilates it without any further localization argument. This leaves only the non-positive accumulated drift, enforcing
\begin{equation}
    \mathbb{E}\left[(Y_{(t+s) \wedge \tau_\epsilon}^{(i)})^\alpha \,\Big|\, \mathcal{F}_t\right] \le (Y_{t \wedge \tau_\epsilon}^{(i)})^\alpha.
\end{equation}
Since $Y_t^{(i)}$ measures a squared distance, it is bounded below by $0$. Letting $\delta \downarrow 0$, bounded convergence ($0 \le Y^\alpha \le \epsilon^\alpha$) carries the inequality over to the process stopped at $\tau_\epsilon$ alone. Thus $(Y_{t \wedge \tau_\epsilon}^{(i)})^\alpha$ constitutes a non-negative supermartingale.
\end{proof}

\begin{corollary}[Probabilistic Non-Escape]
\label{cor:prob_non_escape-app}
For any stochastically attractive invariant set $A$ of order $\alpha$, the probability of a trajectory initialized at $\boldsymbol{\theta}_0$ escaping the $\epsilon$-neighborhood is bounded by $\epsilon^{-\alpha}\,\mathbb{E}[(Y_{t_0}^{(i)})^\alpha \mid \mathcal{F}_{t_0}]$.
\end{corollary}

\begin{proof}
By Theorem \ref{thm:supermartingale-app}, $(Y_{s \wedge \tau_\epsilon}^{(i)})^\alpha$ is a continuous, non-negative supermartingale for all $s \ge t_0$. Doob's maximal inequality for continuous supermartingales \citep{revuz1999continuous} bounds the tail probability of its supremum:
\begin{equation}
    \mathbb{P}\left(\sup_{s \ge t_0} Y_{s \wedge \tau_\epsilon}^{(i)} \ge \epsilon \,\Big|\, \mathcal{F}_{t_0}\right) = \mathbb{P}\left(\sup_{s \ge t_0} (Y_{s \wedge \tau_\epsilon}^{(i)})^\alpha \ge \epsilon^\alpha \,\Big|\, \mathcal{F}_{t_0}\right) \le \frac{1}{\epsilon^\alpha} \mathbb{E}\left[(Y_{t_0}^{(i)})^\alpha \mid \mathcal{F}_{t_0}\right].
\end{equation}
The event $\left\{\sup_{s \ge t_0} Y_{s \wedge \tau_\epsilon}^{(i)} \ge \epsilon\right\}$ identifies the escape event $\{\tau_\epsilon^{(i)} < \infty\}$ by Lemma \ref{lemma:stopped_bounded}.
\end{proof}

\subsection{Attractor Linkage}

To map the structural binding of decoupled subnetworks, we define the conditions under which independent stochastic trajectories synchronize. For a pair $(i,j)$ of subnetworks with identical architecture the shared manifold is the coincidence set $A_{ij} = \{\boldsymbol{\theta}^{(i)} = \boldsymbol{\theta}^{(j)}\}$, an affine subspace with $Y_t^{(ij)} = \tfrac12\|\boldsymbol{\theta}_t^{(i)} - \boldsymbol{\theta}_t^{(j)}\|_2^2 = d(\boldsymbol{\theta}_t, A_{ij})^2$ and escape time $\tau_\epsilon^{(ij)} = \inf\{s \ge t : Y_s^{(ij)} \ge \epsilon\}$. Every statement of Section \ref{sec:appendix_topological_mechanics} applies to $A_{ij}$ verbatim.

\begin{definition}[Attractor Link Probability]
We define the link probability $\mathcal{P}_{\mathrm{link}}^{(i, j)}(t)$ as the conditional probability that the pair never breaches the boundary of the local basin $\mathcal{N}(A_{ij}, \epsilon)$ from time $t$ onward:
\begin{equation}
    \mathcal{P}_{\mathrm{link}}^{(i, j)}(t) := \mathbb{P}\left( \tau_\epsilon^{(ij)} = \infty \mid \mathcal{F}_t \right) = \mathbb{P}\left( \sup_{s \ge t} Y_s^{(ij)} < \epsilon \,\Big|\, \mathcal{F}_t \right).
\end{equation}
\end{definition}

\begin{lemma}[Maximal Link Bound]
\label{lemma:doob-app}
Let $A_{ij}$ be stochastically attractive of order $\alpha$. Given an instantaneous state inside the local basin ($Y_t^{(ij)} < \epsilon$), the link probability satisfies
\begin{equation}
    \mathcal{P}_{\mathrm{link}}^{(i,j)}(t) \ge 1 - \frac{(Y_t^{(ij)})^\alpha}{\epsilon^\alpha}.
\end{equation}
\end{lemma}

\begin{proof}
By Theorem \ref{thm:supermartingale-app} applied to $A_{ij}$, the stopped process $(Y_{s \wedge \tau_\epsilon}^{(ij)})^\alpha$ is a continuous, non-negative supermartingale for all $s \ge t$. Doob's maximal inequality \citep{revuz1999continuous} gives
\begin{equation}
    \mathbb{P}\left(\sup_{s \ge t} Y_{s \wedge \tau_\epsilon}^{(ij)} \ge \epsilon \,\Big|\, \mathcal{F}_t\right) \le \frac{(Y_t^{(ij)})^\alpha}{\epsilon^\alpha}.
\end{equation}
The event on the left is the escape event $\{\tau_\epsilon^{(ij)} < \infty\}$, and its complement is the link event, which yields the bound.
\end{proof}

\begin{definition}[Attractor Linkage, $\sim_A$]
\label{def:attractor_equiv-app}
Two subnetworks are linked at time $t$, $\boldsymbol{\theta}^{(i)} \sim_A \boldsymbol{\theta}^{(j)}$, if and only if some $s \le t$ has $Y_s^{(ij)} < \epsilon_{\mathrm{in}}$ and $Y_u^{(ij)} < \epsilon$ for all $u \in [s,t]$. A link forms when $Y^{(ij)}$ first falls below an inner radius $\epsilon_{\mathrm{in}} < \epsilon$ and breaks at the next escape from the tube of radius $\epsilon$. By Lemma \ref{lemma:doob-app} a link survives forever with probability at least $1 - (\epsilon_{\mathrm{in}}/\epsilon)^\alpha$. Without the inner radius the bound at formation would be $1 - \epsilon^{-\alpha}\epsilon^\alpha = 0$, since a continuous path enters the tube at $Y = \epsilon$ and a nondegenerate diffusion recrosses the level immediately.
\end{definition}

\begin{proposition}[Linkage Classes]
\label{prop:linkage_classes}
$\sim_A$ is reflexive and symmetric. The class $[i]_t$ of $i$ is the connected component of $i$ in the graph $G_t^\epsilon$ with edge set $\{\{i,j\} : \boldsymbol{\theta}^{(i)} \sim_A \boldsymbol{\theta}^{(j)}\}$. The classes partition $\mathcal{S}$ at every time, and two subnetworks in the same class are joined by a chain of links, each surviving with the probability of Lemma \ref{lemma:doob-app}.
\end{proposition}
\begin{proof}
$Y_t^{(ii)} = 0 < \epsilon$ and $Y_t^{(ij)} = Y_t^{(ji)}$ give reflexivity and symmetry. Connected components of a graph partition its vertex set. They form the transitive closure of $\sim_A$, and the closure is all that the constructions below use. A single-pair relation cannot be transitive at fixed $\epsilon$, since $Y^{(ij)}, Y^{(jk)} < \epsilon$ only gives $Y^{(ik)} < 4\epsilon$ by the triangle inequality, and components resolve this without weakening the tube.
\end{proof}

\subsection{Constructing the Reeb Topology}

Now, we map the links onto a discrete spatial graph, using $\sim_A$ as an adjacency function.

\begin{definition}[The $\epsilon$-Fixed Topological Network]
At any fixed training time $t \in [0, \infty)$ and spatial boundary $\epsilon > 0$, define the instantaneous spatial network $G_t^\epsilon = (\mathcal{S}, E_t^\epsilon)$. The node set $\mathcal{S} = \{1, \dots, N\}$ indexes the decoupled subnetworks. The relation $\sim_A$ of Definition \ref{def:attractor_equiv-app} is the adjacency function, so an undirected edge $e_{ij} \in E_t^\epsilon$ is present if and only if the pair is linked at time $t$.
\end{definition}

This $\epsilon$-fixed network captures instantaneous connectivity. To model stochastic collapse, we track how this network rewires across training time by elevating the discrete spatial networks into a continuous topological space.

\begin{definition}[Continuous Reeb Graph]
We define the macroscopic topological state of the neural network as the Reeb graph $\mathcal{R}$, constructed as the quotient space over the product of the subnetwork indices and time. The quotient is taken with respect to the components of a time-dependent graph rather than the level sets of a scalar function, so $\mathcal{R}$ is a Reeb-type quotient graph, and we keep the shorter name throughout:
\begin{equation}
    \mathcal{R} := \left(\mathcal{S} \times [0, \infty)\right) / \sim_{\mathcal{R}}
\end{equation}
where the identification $(i, t) \sim_{\mathcal{R}} (j, t)$ holds if and only if $j \in [i]_t$, the connected component of $i$ in $G_t^\epsilon$. We take $\sim_{\mathcal{R}}$ to be the smallest closed equivalence relation containing these identifications, so that at a split time both branches share the split vertex and $\mathcal{R}$ is Hausdorff.
\end{definition}

\begin{theorem}[Non-Equilibrium Condensation and Fragmentation]
\label{thm:reeb_transitions}
The Reeb graph $\mathcal{R}$ maps the structural phase transitions of the SDE via two topological mechanisms:
\begin{itemize}
    \item \textbf{Condensation ($\cup$):} As a pair's transverse distance falls below $\epsilon_{\mathrm{in}}$, an edge forms in $G_t^\epsilon$. If the two classes were disjoint, the quotient map fuses them into a single topological merge node on $\mathcal{R}$; otherwise the edge closes a cycle and $\mathcal{R}$ is unchanged.
    \item \textbf{Fragmentation ($\emptyset$):} If an injected variance deviation breaches the basin $\mathcal{N}(A_{ij}, \epsilon)$, the edge in $G_t^\epsilon$ deletes. If it was a bridge of its component, the quotient map splits the node into decoupled branches on $\mathcal{R}$.
\end{itemize}
\end{theorem}

\begin{proof}
We map the SDE dynamics to the topological level sets of $\mathcal{R}$ \citep{edelsbrunner2008topological}.

For \textbf{condensation}, let $Y_t^{(ij)}$ drop below $\epsilon_{\mathrm{in}}$. The adjacency function updates to $e_{ij} = 1$ in $G_t^\epsilon$, and Lemma \ref{lemma:doob-app} bounds the probability that this edge ever deletes by
\begin{equation}
    \mathbb{P}(\tau_\epsilon^{(ij)} < \infty \mid \mathcal{F}_t) \le \frac{(Y_t^{(ij)})^\alpha}{\epsilon^\alpha} \le \Big(\frac{\epsilon_{\mathrm{in}}}{\epsilon}\Big)^\alpha,
\end{equation}
which is small when the inner radius is small against the tube. The quotient projection $\pi_{\mathcal{R}}: \mathcal{S} \times [0, \infty) \to \mathcal{R}$ identifies $(i,t) \sim_{\mathcal{R}} (j,t)$; when $[i]_{t^-} \cap [j]_{t^-} = \emptyset$, this fuses the two components into a single merge vertex $v \in \mathcal{R}$ with $\text{in-degree}(v) \ge 2$.

For \textbf{fragmentation}, let a stochastic large deviation \citep{freidlin2012random} breach the local basin $\mathcal{N}(A_{ij}, \epsilon)$. This exit realizes the stopping time:
\begin{equation}
    \exists s \in (t-\delta, t] \text{ s.t. } Y_s^{(ij)} \ge \epsilon \implies \tau_\epsilon^{(ij)} \le t.
\end{equation}
The edge $e_{ij}$ deletes from $G_t^\epsilon$. When $e_{ij}$ was a bridge, the component disconnects ($[i]_t \cap [j]_t = \emptyset$) and the quotient map partitions the trajectory into a critical split vertex $v \in \mathcal{R}$ with $\text{out-degree}(v) \ge 2$.
\end{proof}

\section{Percolative Collapse and Discrete Scale Invariance}
\label{sec:appendix_dsi_percolation}

We first bridge the geometric symmetries of the non-convex loss landscape with the statistical mechanics of percolation. We then show how the invariant sets generated by network symmetries restrict collapse, breaking continuous scale invariance and causing the macroscopic network to show Discrete Scale Invariance (DSI).

\subsection{Approximate \texorpdfstring{$Q$}{Q}-Symmetry, Affine Trapping, and Transverse Diffusion}

In deep neural networks, intrinsic architectural symmetries, specifically permutation invariance among functionally equivalent hidden neurons, naturally partition the parameter space into distinct geometric subspaces \citep{chen2023stochastic}. Under stochastic gradient dynamics, these subspaces become stochastically attractive and trap the learning trajectory. To ground the subsequent macroscopic topology in the physical architecture of the network, we formalize this geometric property.

\begin{definition}[Approximate $Q$-Symmetry]\label{def:approximate_q_symmetry}
Let $Q \in \mathbb{R}^{d \times d}$ be an orthogonal matrix, $Q Q^\top = I_d$, which covers the coordinate permutations used throughout. This operator represents a parameter transformation, such as the permutation of weights between identical hidden subnetworks. A loss functional $\mathcal{L}(\boldsymbol{\theta})$ exhibits approximate $Q$-symmetry around the affine subspace $A = \{\boldsymbol{\theta} \in \mathbb{R}^d \mid Q\boldsymbol{\theta} = \boldsymbol{\theta}\}$ if, for any arbitrary tolerance $\epsilon > 0$, there exists a spatial radius $\delta > 0$ such that:
\begin{equation}
    |\mathcal{L}(Q\boldsymbol{\theta}) - \mathcal{L}(\boldsymbol{\theta})| < \epsilon \cdot d(\boldsymbol{\theta}, A) \quad \forall \boldsymbol{\theta} \text{ s.t. } d(\boldsymbol{\theta}, A) < \delta,
\end{equation}
where $d(\boldsymbol{\theta}, A) = \inf_{\boldsymbol{a} \in A} \|\boldsymbol{\theta} - \boldsymbol{a}\|_2$ defines the minimal Euclidean distance from the parameter vector to the subspace.
\end{definition}

Treating $A$ as an arbitrary stochastic attractor omits the structural mechanics of the merge process. Deriving the geometry of $A$ directly from the network's linear permutation operations guarantees that the merge geometry is affine and that transverse projections are constant.

\begin{theorem}[Affine Generation from Architectural Permutations]
\label{thm:affine_generation}
If the invariant set $A$ is generated by the architectural permutation of identical sub-components (the symmetric group $S_n$), $A$ is strictly a flat, affine subspace, characterized by a constant orthogonal projection operator $P^\perp$.
\end{theorem}
\begin{proof}
The transposition or permutation of $n$ indistinguishable neural components corresponds to swapping specific blocks of parameters within the weight vector $\boldsymbol{\theta}$. This operation is isomorphic to left-multiplication by a constant block-permutation matrix $Q$. The set of configurations invariant under this structural symmetry constitutes the eigenspace of $Q$ corresponding to the eigenvalue $\lambda = 1$. The null space $\ker(I - Q)$ forms a linear subspace by definition, in particular an affine one. Consequently, the normal bundle is globally parallel, and the orthogonal projection $P^\perp$ mapping any state to its transverse distance vector is globally constant and independent of the local coordinate $\boldsymbol{\theta}$.
\end{proof}

\begin{theorem}[Symmetry-Induced Trapping]
\label{thm:q_symmetry_invariant}
If the differentiable loss function $\mathcal{L}$ satisfies approximate $Q$-symmetry around the affine subspace $A$, then $A$ constitutes a strict invariant set for the continuous deterministic gradient dynamics. \citep{chen2023stochastic}
\end{theorem}
\begin{proof}
Let $\boldsymbol{\theta} \in A$, so $Q\boldsymbol{\theta} = \boldsymbol{\theta}$, and let $\boldsymbol{n}$ be a unit vector orthogonal to $A$. Then $Q(\boldsymbol{\theta} + h\boldsymbol{n}) = \boldsymbol{\theta} + hQ\boldsymbol{n}$ and $d(\boldsymbol{\theta} + h\boldsymbol{n}, A) = h$, so approximate $Q$-symmetry gives, for every $\epsilon > 0$ and $0 < h < \delta$,
\begin{equation}
    \Big|\frac{\mathcal{L}(\boldsymbol{\theta} + hQ\boldsymbol{n}) - \mathcal{L}(\boldsymbol{\theta} + h\boldsymbol{n})}{h}\Big| < \epsilon.
\end{equation}
Letting $h \to 0$ gives $|\nabla\mathcal{L}(\boldsymbol{\theta})^\top(Q - I)\boldsymbol{n}| \le \epsilon$ for every $\epsilon > 0$, hence $\nabla\mathcal{L}(\boldsymbol{\theta})^\top(Q - I)\boldsymbol{n} = 0$. Since $Q$ is orthogonal, $I - Q$ is normal, so its range is $\ker(I - Q)^\perp = A^\perp$ and it maps $A^\perp$ onto $A^\perp$. Hence $\nabla\mathcal{L}(\boldsymbol{\theta})$ is orthogonal to $A^\perp$, that is $P^\perp\nabla\mathcal{L}(\boldsymbol{\theta}) = 0$ and $Q\nabla\mathcal{L}(\boldsymbol{\theta}) = \nabla\mathcal{L}(\boldsymbol{\theta})$. The deterministic drift is tangent to $A$ on $A$.
\end{proof}

The affine geometry of $A$ (Theorem \ref{thm:affine_generation}) removes every curvature term from the transverse dynamics.

\begin{theorem}[Curvature-Free Transverse Diffusion]
\label{thm:curvature_free}
Because the permutation-invariant set $A$ is an affine subspace, the It\^o correction of the transverse distance $Y_t = d(\boldsymbol{\theta}_t, A)^2$ carries no geometric curvature term, and its generator is the flat expression of Equation \eqref{eq:generator}, and $Y_t$ decouples from the tangential coordinates up to the $\boldsymbol{\theta}$-dependence of $\nabla\mathcal{L}$ and $\Sigma_\perp$.
\end{theorem}
\begin{proof}
Let the parameter update follow the stochastic differential equation $d\boldsymbol{\theta}_t = -\nabla \mathcal{L}(\boldsymbol{\theta}_t) dt + G(\boldsymbol{\theta}_t) dW_t$. Applying It\^o's Lemma to the transverse distance function $Y_t$, the differential $dY_t$ picks up a correction governed by the Hessian of the distance function. For a general manifold $M$, the Laplacian of the distance function encodes the mean curvature trace $\mathrm{Tr}(H_M)$, injecting a geometric drift of the form $\propto \frac{1}{2} D \cdot \mathrm{Tr}(H_M) dt$. 

Because $A$ is an affine subspace (Theorem \ref{thm:affine_generation}), all principal curvatures are identically zero ($H_A = 0$) and the Hessian of $Y$ is the constant $2P^\perp$. The normal bundle is globally flat, and the generator of $Y_t$ reduces to Equation \eqref{eq:generator}, which contains only the gradient attraction $-2\boldsymbol{x}^\top\nabla\mathcal{L}$ and the transverse noise trace. The radial part of the transverse process is then a Bessel-type process driven by the transverse noise projection and the gradient attraction. A one-dimensional Fokker--Planck description of $Y_t$ alone additionally requires the transverse coefficients to depend on $\boldsymbol{\theta}$ only through $Y$, a closure that Remark \ref{remark:fp_heuristic} states.
\end{proof}

\begin{remark}[Generalization to Highly Curved Invariant Manifolds]
If one generalizes this framework to neural architectures possessing continuous symmetries (e.g., continuous rotation invariances generating highly curved, non-linear invariant manifolds $M$), the non-zero mean curvature $H_M \neq 0$ introduces a geometric drift proportional to the local concavity of $M$. This drift couples the transverse fluctuations to the tangential motion along the manifold, modifying the local stationary distribution. Consequently, generalizing this topological approach to curved invariant manifolds requires incorporating these geometric potential terms to accurately model the edge formation rates.
\end{remark}

\subsection{Constructing the Percolation}

While the deterministic gradient aligns parallel to $A$, the continuous injection of stochastic gradient noise drives the parameters into these lower-dimensional geometric spaces. 

\begin{definition}[The Macroscopic Topological Graph]
Index the $N$ initially decoupled subnetworks as nodes in a dynamic spatial graph $G_\tau = (\mathcal{S}, E_\tau^\epsilon)$. An undirected edge $e_{ij}$ exists at time $\tau$ if and only if $\boldsymbol{\theta}^{(i)} \sim_A \boldsymbol{\theta}^{(j)}$ in the sense of Definition \ref{def:attractor_equiv-app}, so that the pair has entered the inner tube and has not since left the tube around its shared invariant set $A_{ij}$.
\end{definition}

To map stochastic network dynamics to percolation, we require a monotonically increasing control parameter. Continuous Brownian noise injects transient fragmentations (edges breaking) and condensations (edges forming) across the basin boundary $\epsilon$, which breaks monotonicity for any single realization, and monotonicity is carried instead by the ensemble. Assumption \ref{assum:condensation_phase} asks that the expected number of components $\mathbb{E}[M(\tau)]$ be nonincreasing on the training interval under study, which holds when link formation dominates link breaking, and Lemma \ref{lemma:doob-app} bounds the breaking probability of a formed link by $(\epsilon_{\mathrm{in}}/\epsilon)^\alpha$. This interval is the condensation phase. A task shift or a learning-rate drop ends it (Appendix \ref{sec:appendix_toy_experiments}), and under Adam with weight decay so does a slingshot instability (Remark \ref{remark:adam_diagnostic}).

\begin{remark}[The Fokker--Planck Picture]
\label{remark:fp_heuristic}
Let $p(y, t)$ denote the density of $Y_t^{(ij)}$, so that $\mathcal{P}_e(t) = \int_0^\epsilon p(y, t) dy$ and, with $J(y,t)$ the probability current of the one-dimensional Fokker--Planck equation \citep{risken1996fokker} and a reflecting origin, $\frac{d}{dt}\mathcal{P}_e(t) = -J(\epsilon, t)$. If the transverse dynamics relax on a fast timescale $\tau_{fast}$ to a quasi-stationary law $p_{ss}(y; \sigma(t))$ that deforms on a slow timescale $\tau_{slow} \gg \tau_{fast}$ through a contracting scale $\sigma(t)$, then $\frac{d}{dt}\mathcal{P}_e = \dot\sigma(t) \int_0^\epsilon \partial_\sigma p_{ss}(y;\sigma) dy$, and this is positive exactly when the family is stochastically ordered on the tube, that is, when the distribution function $F_\sigma(y)$ rises as $\sigma$ falls for every $y \le \epsilon$. Two facts fix the status of this picture. First, zero stationary current at $\epsilon$ balances fragmentations against condensations in the mean and does not stop individual crossings, and for this reason Assumption \ref{assum:condensation_phase} concerns the ensemble occupancy rather than individual realizations. Second, a falling mean and a falling variance do not by themselves order the family. In the toy model of Appendix \ref{sec:appendix_toy_experiments}, $Y$ is Gamma-distributed with shape $\mu/\zeta^2 - \tfrac12$ and scale $\zeta^2$. Lowering $\mu$ at fixed $\zeta$ lowers the shape at fixed scale, and a Gamma family is stochastically monotone in its shape, so the distribution function rises at every $y$ and the ordering holds everywhere. Index the family instead by $\zeta$, which the learning rate sets through $\zeta^2 \propto \eta$, at fixed $\mu$. Raising $\zeta$ lowers the shape and raises the scale together, the mean $\mu - \zeta^2/2$ falls, and the two distribution functions cross at exactly one point, since the density ratio $y^{k_2 - k_1}e^{y(1/\theta_1 - 1/\theta_2)}$ decreases and then increases. The distribution function rises below a crossing point $y_\times(\mu,\zeta)$ and falls above it. For $\mu = 1$ and $\zeta^2$ moving from $1$ to $1.5$ the mean halves, $y_\times \approx 4.2$, and $F(5)$ drops from $0.99843$ to $0.99805$. The ordering therefore holds on the tube when $\epsilon < y_\times$ and fails beyond it, which is the regime the tube condition asks for. The quasi-stationary picture itself dissolves at the threshold $\mu \to \zeta^2/2$, where the log-coordinate $x = \ln w$ diffuses in the potential $V(x) = -(\mu - \zeta^2/2)x + \tfrac12 e^{2x}$, whose left wall flattens and whose mixing time diverges. The monotone occupancy is therefore an assumption, and the ensemble fraction of seeds inside the tube verifies it. The spectral geometry of deep SGD supplies the timescale separation behind the picture \citep{li2022what, haochen2021shape, blanc2020implicit}, with fast drift along the large Hessian eigenvalues into a basin followed by slow noise-driven diffusion along the flat invariant manifold.
\end{remark}

\begin{definition}[Merge Fraction, $p$]
Let $M(\tau)$ denote the number of connected components of $G_\tau^\epsilon$. We define the control parameter as the merge fraction
\begin{equation}
    p(\tau) := 1 - \frac{\mathbb{E}[M(\tau)]}{N} \in \Big[0, 1 - \frac{1}{N}\Big],
\end{equation}
which counts the expected number of merges per subnetwork. When the components are $M$ blocks of equal size $i$, $1 - p = M/N = 1/i$, so the merge fraction records the component size directly, which the pairwise edge density $\mathbb{E}|E_\tau^\epsilon|/\binom{N}{2}$ does not, since $M$ cliques of size $i$ have edge density $(i-1)/(N-1)$.
\end{definition}

\begin{proposition}[Reeb Graph Reduction to Percolation Graph]
\label{thm:reeb_to_percolation}
Under Assumption \ref{assum:condensation_phase}, $p(\tau)$ is nondecreasing in $\tau$, and $p(\tau) = 1 - 1/N$ once every realization is a single component. Set $\tau_p := \inf\{\tau : p(\tau) \ge p\}$ and $G_p := G_{\tau_p}$. This yields a percolation process $G_p = (\mathcal{V}, \mathcal{E}_p)$ with control parameter $p \in [0, 1 - 1/N]$. The transient Reeb graph $\mathcal{R}_t$ and the process $G_p$ record the same edges, and only the index differs.
\end{proposition}

\begin{proof}
$p$ is nondecreasing because $\mathbb{E}[M]$ is nonincreasing by assumption, and $M = 1$ in every realization gives $p = 1 - 1/N$. The supermartingale property alone does not make the transverse moments decay, so the progress of $p$ toward $1 - 1/N$ rests on links forming, which Proposition \ref{prop:strict_decay} supplies under strict attractivity, and on links persisting, which Lemma \ref{lemma:doob-app} bounds through $(\epsilon_{\mathrm{in}}/\epsilon)^\alpha$. The edge indicator of a realization is the same object in $\mathcal{R}_t$ and in $G_\tau$, and only the index changes from $\tau$ to $p$. A nondecreasing $p(\tau)$ can be constant on an interval, so $\tau \mapsto p(\tau)$ need not be injective. The generalized inverse $\tau_p$ picks the first time the density reaches $p$ and makes $G_p$ well defined without asking for strict monotonicity.
\end{proof}

Single realizations may delete edges. The percolation process is monotone in its control parameter, not in its realizations, and the relative variance of the next subsection measures those deletions. As the transverse moments collapse, stochastic attractivity binds path measures and $p$ advances from $0$ (fully decoupled) to $1 - 1/N$ (complete collapse of the $N$ symmetric subnetworks onto the invariant set $A$).

\begin{proposition}[Exponential Decay under Strict Attractivity]
\label{prop:strict_decay}
If $A$ is strictly attractive with rate $\kappa$, so that $\mathscr{G}Y^\alpha \le -\kappa Y^\alpha$ on $\mathcal{N}(A,\epsilon)\setminus A$, then for every $t \ge t_0$
\begin{equation*}
    \mathbb{E}\big[(Y_t^{(i)})^\alpha\,\mathds{1}\{\tau_\epsilon > t\} \,\big|\, \mathcal{F}_{t_0}\big] \le e^{-\kappa(t - t_0)}\,(Y_{t_0}^{(i)})^\alpha .
\end{equation*}
\end{proposition}
\begin{proof}
Apply It\^o's lemma to $e^{\kappa t}(Y_t^{(i)})^\alpha$ on $\{t < \tau_\epsilon\}$. The drift is $e^{\kappa t}(\kappa Y^\alpha + \mathscr{G}Y^\alpha) \le 0$, and the stochastic integral is a martingale by the argument of Theorem \ref{thm:supermartingale-app}, so $e^{\kappa(t\wedge\tau_\epsilon)}(Y^{(i)}_{t\wedge\tau_\epsilon})^\alpha$ is a non-negative supermartingale. Restricting the expectation to $\{\tau_\epsilon > t\}$, where $t \wedge \tau_\epsilon = t$, and using non-negativity of the discarded part gives $e^{\kappa t}\,\mathbb{E}[(Y_t^{(i)})^\alpha\mathds{1}\{\tau_\epsilon > t\} \mid \mathcal{F}_{t_0}] \le e^{\kappa t_0}(Y_{t_0}^{(i)})^\alpha$.
\end{proof}

\subsection{Order Parameter and Microtransitions}

To detect macroscopic phase transitions driven by this percolation process, we formalize the physical observables tracking the connectivity of $G_p$.

\begin{definition}[Structural Order Parameter]
Let $C_{\max}(p)$ denote the largest connected component of vertices existing on the realization graph $G_p = G_{\tau_p}$ at density $p$. We define the macroscopic order parameter $\mathcal{O}(p)$ as its fractional size relative to the entire network:
\begin{equation}
    \mathcal{O}(p) = \frac{|C_{\max}(p)|}{N}.
\end{equation}
This formulation isolates the relative size of the largest macroscopic component, mirroring $C_1/N$ in explosive percolation physics \citep{chen2014microtransition}.
\end{definition}

\begin{definition}[Merge Multiplicity]
\label{def:merge_multiplicity}
A condensation event has multiplicity $n$ if it joins $n$ components of equal size $i$ into one component of size $ni$. Proposition \ref{thm:hypercondensation} makes generic condensation pairwise, so $n = 2$ in the generic case. The balanced cascade, in which merged blocks merge again at equal sizes, is a further assumption on the order of the merges, since pairwise merging alone also admits the sequence $1+1 \to 2$, $1+2 \to 3$, $1+3 \to 4$. Nested symmetry permits the balanced order (Appendix \ref{sec:appendix_block_merges}), and we keep $n$ as a free parameter of the cascade.
\end{definition}

In Erd\H{o}s-R\'enyi random graphs, the order parameter grows by absorbing components of vanishing relative size. The binding of symmetric subnetworks forces block jumps. We formalize this divergence from Erd\H{o}s-R\'enyi dynamics via the following theorem.

\begin{proposition}[Non-Erd\H{o}s-R\'enyi Discontinuity]
\label{thm:er_divergence}
Let $G_{ER}(N, p)$ denote an Erd\H{o}s-R\'enyi graph governed by uniform random edge selection. Every jump of its order parameter $\mathcal{O}_{ER}(p)$ at finite $N$ is $o(1)$ with high probability, so $\mathcal{O}_{ER}$ is asymptotically continuous. In contrast, a condensation event of multiplicity $n$ on components of size $i$ in $G_p$ produces a jump $\Delta\mathcal{O} = (n-1)i/N$, and whether this discontinuity survives the thermodynamic limit $N \to \infty$ depends on the size of the merging components relative to $N$ at the moment of the merge (Corollary \ref{cor:microscopic_washout-app}).
\end{proposition}

\begin{proof}
In $G_{ER}(N, p)$, edges arrive one at a time. An edge joining two components of sizes $a \le b$ raises the largest component by at most $a$, since if $b$ is the largest component the new largest is $a+b$, and otherwise $a + b - |C_{\max}| \le a$. Below the critical window all components are $O(\log N)$, inside it the largest are $O(N^{2/3})$, and above it a unique giant component absorbs components of size $O(\log N)$ \citep{stauffer1994introduction}. Every jump of $|C_{\max}|$ is therefore $o(N)$ with high probability, so every jump of $\mathcal{O}_{ER}$ is $o(1)$, which is the statement that $\mathcal{O}_{ER}$ is asymptotically continuous.

Conversely, a condensation event of multiplicity $n$ binds $n$ components of size $i$ simultaneously into a single component of size $ni$. The order parameter experiences a step-like increment at finite $N$:
\begin{equation}
    \Delta \mathcal{O}(p) = \frac{n i - i}{N} = \frac{(n-1)i}{N} > 0.
\end{equation}
Whether $\Delta\mathcal{O}(p)$ remains bounded away from $0$ as $N \to \infty$ is governed by the scaling of $i$ with $N$; see Corollary \ref{cor:microscopic_washout-app}.
\end{proof}

\begin{corollary}[Regime Dependence of the Discontinuity]
\label{cor:microscopic_washout-app}
Let $i = i(N)$ denote the size of the merging components at a given microtransition.
\begin{itemize}
    \item \textbf{Macroscopic regime:} If $i(N) = \Theta(N)$, i.e.\ $i(N)/N \to c \in (0,1]$ as $N \to \infty$, then $\Delta\mathcal{O}(p) \to (n-1)c > 0$: the jump is a genuine discontinuity that survives the thermodynamic limit.
    \item \textbf{Microscopic regime:} If $i(N) = O(1)$, i.e.\ the size of each merging component is bounded independently of $N$, then $\Delta\mathcal{O}(p) = O(1/N) \to 0$: the jump vanishes as $N \to \infty$, and the transition is asymptotically continuous.
\end{itemize}
\end{corollary}
\begin{proof}
Immediate from $\Delta \mathcal{O}(p) = (n-1)i(N)/N$ (Proposition \ref{thm:er_divergence}) by taking $N \to \infty$ under each scaling assumption on $i(N)$.
\end{proof}

\begin{remark}[Relation to Achlioptas-Type Explosive Percolation and Finite-Size Scaling]
\citet{riordan2011explosive} show that all Achlioptas processes in which an edge is chosen competitively among a bounded number of candidates at each step have continuous phase transitions in the thermodynamic limit, contrary to earlier numerical claims of discontinuity \citep{achlioptas2009explosive}. Corollary \ref{cor:microscopic_washout-app} is consistent with this result, since the microscopic regime ($i=O(1)$) reproduces exactly the finite-$N$ illusion-of-discontinuity behavior that Achlioptas-type processes are now known to exhibit. Our mechanism differs in the macroscopic regime, where merges are driven by symmetry-forced identification of already-macroscopic components. While determining whether practical architectures reside in the macroscopic $\Theta(N)$ or microscopic $O(1)$ regime remains an active theoretical frontier, this distinction ultimately governs only the true asymptotic limit ($N \to \infty$).
\end{remark}
To detect these discontinuities across ensemble trajectories, we utilize the relative variance of the order parameter following \citep{chen2014microtransition}.

\begin{definition}[Relative Variance and Microtransitions]
Define the relative variance $R_v(p)$ of the order parameter $\mathcal{O}(p)$ over the ensemble of stochastic trajectories as the variance normalized by the squared expectation:
\begin{equation}
    R_v(p) = \frac{\mathbb{E}\left[(\mathcal{O}(p) - \mathbb{E}[\mathcal{O}(p)])^2\right]}{\mathbb{E}[\mathcal{O}(p)]^2}.
\end{equation}
Discontinuous block merges manifest as sharp, localized peaks in $R_v(p)$. We define these peaks as \emph{microtransitions}, which precede complete collapse at $p = 1 - 1/N$.
\end{definition}

\begin{theorem}[Bounded Peak of Relative Variance at Microtransitions]
\label{thm:rv_divergence}
If the ensemble law of the order parameter near a microtransition density $p_i$ is a two-point mixture of the pre-jump value $\mathcal{O}_1$ and the post-jump value $\mathcal{O}_1 + \Delta\mathcal{O}$, with weight $w$ on the merged state, then
\begin{equation}
    R_v(w) = \frac{w(1-w)\,\Delta\mathcal{O}^2}{(\mathcal{O}_1 + w\,\Delta\mathcal{O})^2},
\end{equation}
with $R_v(0) = R_v(1) = 0$ and a unique maximum at $w^* = \mathcal{O}_1/(2\mathcal{O}_1 + \Delta\mathcal{O})$ of height $\Delta\mathcal{O}^2/\big(4\mathcal{O}_1(\mathcal{O}_1 + \Delta\mathcal{O})\big)$. For a block merge of multiplicity $n$ we have $\Delta\mathcal{O} = (n-1)\mathcal{O}_1$, so the peak occurs at $w^* = 1/(n+1)$ with height $(n-1)^2/(4n)$, independent of $i$ and $N$. For pairwise merges the height equals $1/8$.
\end{theorem}

\begin{proof}
Near $p_i$, stochastic fluctuations cause different network realizations to either undergo the block jump or remain in the unmerged state. Model the ensemble distribution as a mixture of two states with values $\mathcal{O}_1$ and $\mathcal{O}_2 = \mathcal{O}_1 + \Delta \mathcal{O}$, weighted by probabilities $1-w$ and $w$ respectively. Then $\mathbb{E}[\mathcal{O}] = \mathcal{O}_1 + w \Delta \mathcal{O}$ and
\begin{equation}
    \mathrm{Var}(\mathcal{O}) = w(1-w)(\Delta \mathcal{O})^2,
\end{equation}
which gives the stated $R_v(w)$. Set $a = \mathcal{O}_1$ and $b = \mathcal{O}_1 + \Delta\mathcal{O}$. Then
\begin{equation}
    R_v(w) = \frac{w(1-w)(b-a)^2}{(a + w(b-a))^2}, \qquad R_v'(w) = \frac{(b-a)^2\,[a - (a+b)w]}{(a + w(b-a))^3},
\end{equation}
which vanishes at $w^* = a/(a+b)$, where $R_v(w^*) = (b-a)^2/(4ab)$. Substituting $b = na$ gives $w^* = 1/(n+1)$ and $(n-1)^2/(4n)$. Since $\mathcal{O} \ge 1/N$, the peak is always finite: the first merge out of isolated vertices has $\mathcal{O}_1 = 1/N$ and $\Delta\mathcal{O} = (n-1)/N$, and gives the same height. The peak flags the microtransition, and its height does not. In the fixed-multiplicity model every peak of a cascade has the same height, and each peak is located at the $w^*$-quantile of the hitting law.
\end{proof}

\subsection{Pairwise Condensation and Discrete Scale Invariance}

In standard continuous percolation, cluster sizes follow universal power laws in $p_c - p$ and edges are drawn uniformly. Permutation symmetry restricts which merges can happen.

\begin{proposition}[Generic Pairwise Condensation]
\label{thm:hypercondensation}
Let $A$ be generated by the symmetric group $S_n$ acting on $n$ identical subnetworks of dimension $d_0$. The fixed subspaces of the subgroups $S_k \subset S_n$, $2 \le k \le n$, are nested, $A_{S_2} \supset A_{S_3} \supset \dots \supset A_{S_n} = A$, with codimension $(k-1)d_0$. A trajectory reaching $A_{S_k}$ from a generic initial condition passes through $A_{S_{k-1}}$ first, so condensation events are pairwise for generic initializations, and a simultaneous $n$-body merge requires the $n-1$ independent pairwise transverse distances to vanish at the same time, a codimension-$(n-1)d_0$ coincidence.
\end{proposition}

\begin{proof}
$A_{S_k} = \{\boldsymbol{\theta}^{(1)} = \dots = \boldsymbol{\theta}^{(k)}\}$ has codimension $(k-1)d_0$, and $A_{S_k} \subset A_{S_{k-1}}$ since equality of $k$ blocks implies equality of the first $k-1$. Each pairwise coincidence set $A_{ij}$ is invariant and, under stochastic attractivity, trapping for its pair by Theorem \ref{thm:supermartingale-app}, so the pairwise mechanism operates independently of the higher coincidences. Reaching $A_{S_k}$ means $Y^{(ij)} \to 0$ for $k-1$ independent pairs. Their hitting times are distinct for generic transverse dynamics, and for $n \ge 3$ the event that they coincide has codimension $(n-1)d_0 > d_0$ in the space of transverse configurations. The largest component therefore grows by pairwise merges, and the hierarchy $1 \to 2 \to 4 \to \cdots$ of Definition \ref{def:merge_multiplicity} is the idealization in which merged pairs merge pairwise again at equal sizes.
\end{proof}

Traditional scale invariance models self-similarity through a continuous Lie group of dilations, where observables remain invariant or scale homogeneously under arbitrary continuous rescaling transformations $x \to \lambda x$ for any real $\lambda > 0$. The infinitesimal generators of these scaling transformations form a continuous Lie algebra, enforcing continuous translation symmetry in logarithmic space ($\ln x \to \ln x + \ln \lambda$). 

However, when underlying microscopic constraints or discrete architectural symmetries govern the system, continuous translation invariance in log-space is broken down to a discrete subgroup. The system ceases to be scale-invariant for arbitrary continuous factors, preserving self-similarity only under a discrete geometric sequence of preferred magnification ratios. This gives rise to Discrete Scale Invariance (DSI) \citep{sornette1998discrete}.

\begin{definition}[Discrete Scale Invariance (DSI)]
A system property or observable $f(x)$ exhibits Discrete Scale Invariance if it is invariant under scaling by a discrete set of preferred magnification factors rather than a continuous Lie group, satisfying scaling relations under discrete ratios $\lambda^{k}$ \citep{sornette1998discrete}.
\end{definition}

\begin{definition}[Hitting Laws]
\label{def:ensemble_transition}
Let $\{\omega\}$ index independent noise realizations of the SGF process, and let $T_i^{(\omega)}$ denote the time at which the mean component size $N/M(t)$ of realization $\omega$ first reaches $i$. We write $T_i$ for the random variable, call its law the hitting law of size $i$, and write $p_i = 1 - 1/i$ for the corresponding merge fraction. In the balanced cascade of Definition \ref{def:merge_multiplicity} every component at a level has the same size, so $T_i$ is the time at which the largest component reaches $i$ and $\mathcal{O}$ jumps. In asynchronous coalescence with the constant kernel the size distribution at time $t$ is geometric with mean $N/M(t)$, so the largest of $M$ components exceeds the mean by a factor of order $\log M$, which varies slowly and cancels from the ratio $T_{ni}/T_i$ in the limit.
\end{definition}

\begin{remark}[Why the Scaling Concerns the Hitting Laws]
No functional of the ensemble mean growth curve $m(p) := \mathbb{E}_\omega[|C_{\max}^{(\omega)}(p)|]$ determines $p_i$, since inverting $m$ at level $i$ does not return $\mathbb{E}[T_i]$. Take $C(p) = 1 + \mathds{1}\{T \le p\}$ with $F_T(p) = 1/(1-p) - 1$ on $[0,\tfrac12]$. Then $m(p) = 1/(1-p)$ and $m^{-1}(2) = \tfrac12$, while $\mathbb{E}[T] = 1 - \log 2$. Inverting a mean and averaging an inverse are different operations, and the cascade is a statement about the second.
\end{remark}

\begin{definition}[Symmetric Coalescence Process]
\label{def:coalescence}
Read the components of $G_\tau$ as clusters with mass equal to their size. Two components of sizes $a$ and $b$ merge when the transverse distance of one of their constituent pairs enters the tube, and Assumption \ref{assum:kernel} states that this happens at rate $K(a,b)/N$ with $K$ homogeneous of degree $\gamma_K < 1$. This is the Marcus--Lushnikov process with kernel $K$ \citep{marcus1968, lushnikov1978}, whose mean-field limit is the Smoluchowski coagulation equation \citep{smoluchowski1916, aldous1999}.
\end{definition}

\begin{proof}[Proof of Lemma \ref{lem:dsi_selfsim}]
For the constant kernel, the number $M(t)$ of components decreases by one at total rate $M(M-1)/(2N)$, and in the mean-field limit $N \to \infty$ with $M/N$ fixed the law of large numbers gives $\dot M = -M(M-1)/(2N)$ \citep[Section 2]{aldous1999}. The mean component size $m = N/M$ then satisfies $\dot m = \tfrac12(1 - 1/M) \to \tfrac12$, so $m(t) = 1 + t/2$ grows linearly, the hitting time of size $i$ is $T_i \to 2(i-1)$, and $T_{ni}/T_i \to n$. For a kernel homogeneous of degree $\gamma_K < 1$ that admits a dynamical scaling solution, which is proved for the constant and additive kernels and conjectured in general \citep{aldous1999}, the Smoluchowski equation has solutions $c(x,t) = s(t)^{-2}\varphi(x/s(t))$ with $s(t) \propto t^{1/(1-\gamma_K)}$ \citep{vandongen1985, aldous1999}, so the size scale reaches $i$ at $T_i \propto i^{1-\gamma_K}$ and $T_{ni}/T_i \to n^{1-\gamma_K} = \lambda_t$. In the hierarchical cascade, level $k$ has $M_k = Nn^{-k}$ components of size $n^k$, each linking at rate $M_k K(n^k, n^k)/N = n^{k(\gamma_K - 1)}$, so the level duration is $D_k = n^{k(1-\gamma_K)}\tilde D_k = \lambda_t^k \tilde D_k$, where $\tilde D_k$ is the completion time of a level with unit link rate, a sum of $M_k(1 - 1/n)$ independent exponential waiting times whose relative fluctuation is $O(M_k^{-1/2})$. Summing the levels, $T_{n^{k+1}} = \sum_{j=0}^{k} D_j = \tilde D_0 + \lambda_t\sum_{j=0}^{k-1}\lambda_t^{j}\tilde D_{j+1} = \tilde D_0 + \lambda_t T'_{n^k}$ with $T'_{n^k}$ equal in law to $T_{n^k}$ as $N \to \infty$ at fixed $k$ (at finite $N$ the level durations differ in their numbers of summands), and $\tilde D_0$ is a fraction of order $\lambda_t^{-k}$ of the total. The merge fraction is $p = 1 - M/N$, and at the hitting time of size $i$ in the hierarchy $M = N/i$, so $1 - p_i = 1/i$ and $1 - p_{ni} = (1 - p_i)/n$ exactly. The mean-field statements concern the typical component size, and the largest component follows the same scaling up to logarithmic factors in $M_k$, which the hierarchy removes.
\end{proof}

Lemma \ref{lem:dsi_selfsim} gives a single factor $\lambda_t > 1$ with $T_{ni} \overset{d}{=} \lambda_t T_i$ for every $i$, so that the hitting law of size $ni$ is the hitting law of size $i$ dilated by $\lambda_t$, up to the stated corrections, and in the density variable $1 - p_{ni} = (1 - p_i)/n$.

\begin{theorem}[Discrete Scale Invariance]
\label{thm:dsi_proof}
Assume the coalescence process of Assumption \ref{assum:kernel} has the self-similar scaling limit $s(t) \propto t^{1/(1-\gamma_K)}$ that its dynamical scaling clause supplies, and that the merges are balanced. Then for every $i$ and every functional $\Phi$ of a law that commutes with affine maps of the line, so the mean, every quantile, and the $w^*$-quantile $q_i$ that locates the $R_v$ peak of Theorem \ref{thm:rv_divergence},
\begin{equation}
    \frac{\Phi(T_{ni})}{\Phi(T_i)} = \lambda_t = n^{1-\gamma_K},
\end{equation}
up to the corrections of Lemma \ref{lem:dsi_selfsim}. The peak times lock into the geometric cascade $q_{n^k i} = \lambda_t^{k} q_i$ asymptotically in $k$ (Corollary \ref{cor:peak_times} gives the finite-level ratios), so the $R_v$ peaks of the microtransitions are equally spaced in $\log t$ with spacing $\log\lambda_t$, and the peak densities satisfy $1 - q^{(p)}_{n^k i} = n^{-k}(1 - q^{(p)}_i)$. The peaks precede the macroscopic transition and are equally spaced in $\log(1-p)$ with spacing $\log n$ \citep{chen2014microtransition}.
\end{theorem}

\begin{proof}
Lemma \ref{lem:dsi_selfsim} gives $T_{ni} \overset{d}{=} \lambda_t T_i$. A functional that commutes with affine maps satisfies $\Phi(\lambda_t T) = \lambda_t\Phi(T)$, and means and quantiles do, since $\mathbb{E}[\lambda_t T] = \lambda_t\mathbb{E}[T]$ and $\mathbb{P}(\lambda_t T \le \lambda_t z) = \mathbb{P}(T \le z)$. Hence $\Phi(T_{ni}) = \lambda_t\Phi(T_i)$, and induction on $k$ gives the cascade. The density statement is $p = 1 - M/N$ with $M = N/i$ at the hitting time of size $i$. The $R_v$ peak of Theorem \ref{thm:rv_divergence} occurs where the merged fraction $w = \mathbb{P}(T_i \le t)$ equals $w^*$, so $q_i$ is a quantile of $T_i$ and the identity covers it. Continuous translational invariance in $\log t$ is broken to the discrete subgroup generated by $\log\lambda_t$, which is DSI.
\end{proof}

\begin{corollary}[Geometric Peak Times]
\label{cor:peak_times}
Under Assumption \ref{assum:kernel}, consecutive $R_v$ peak times satisfy $t_{ni}/t_i = \lambda_t = n^{1-\gamma_K}$ up to the finite-size corrections of Lemma \ref{lem:dsi_selfsim}, so $\log t_{n^k i}$ is linear in $k$ with slope $\log \lambda_t$, and the measured ratio fixes $\gamma_K = 1 - \log_n\lambda_t$ once $n$ is known. For the constant kernel the exact finite-size ratios are $(ni-1)/(i-1)$.
\end{corollary}
\begin{proof}
The peak time $t_i$ is the $w^*$-quantile of $T_i$, and Lemma \ref{lem:dsi_selfsim} scales every quantile by $\lambda_t$. For the constant kernel, $T_{n^k} = 2(n^k - 1)$ from the proof of the lemma.
\end{proof}

It is essential to note that DSI generalizes traditional continuous scale invariance by capturing hierarchical, multi-scale organizational structures. While standard continuous power laws assume scale-free behavior across all intervals, DSI systems reveal preferred magnification scales governed by underlying discrete microscopic symmetries, such as pairwise permutation constraints.

\begin{remark}[Empirical Neural Scaling Laws and Stagewise Development]
\label{remark:scaling_laws_stagewise}
Empirical scaling laws in deep learning, such as power-law relationships governing test loss relative to compute budgets and parameter counts \citep{kaplan2020scaling, hoffmann2022training}, are conventionally modeled as smooth, continuous power functions. However, finer-grained analyses of training dynamics reveal that network development is inherently stagewise and non-monotonic, characterized by sudden behavioral shifts and localized phase transitions corresponding to changes in loss landscape degeneracy \citep{wei2022deep, hoogland2024loss}. Interpreting parameter collapse through symmetry-induced percolation suggests that empirical scaling laws might capture coarse-grained averages over an underlying DSI sequence of symmetry-breaking microtransitions.
\end{remark}

\section{Conditional Extension to Adam and AdamW}
\label{sec:appendix_adam}

The percolation and DSI results of Appendix \ref{sec:appendix_dsi_percolation} are derived for vanilla SGD. We state conditions under which the trapping argument carries over to Adam \citep{kingma2015adam} and AdamW \citep{loshchilov2019decoupled}. The extension requires three departures from the SGD setting. The state must be lifted to include the first- and second-moment estimates $(m_t, v_t)$. The admissible symmetry group narrows from general orthogonal $Q$ to coordinate permutations, since elementwise squaring in $v_t$ does not commute with sign flips or general rotations. The gradient noise model must accommodate the heavy tails documented for attention-based architectures \citep{zhang2020adaptive}. We handle the last point by analyzing the clipped recursion that practitioners run, and separate the resulting approximation error into a deterministic cross-sectional component, controlled by equivariance and the Lipschitz geometry of Proposition \ref{prop:affine_invariance}, and a temporal component, controlled by martingale-difference variance bounds for the two exponential filters. When the trapping inequality \eqref{eq:adam_condition} and the confinement condition of Corollary \ref{cor:adam_nonescape} hold, trapping and the cascade of Theorem \ref{thm:dsi_proof} carry over on an annulus around the joint invariant set up to the block escape time. Throughout, $\epsilon_{\mathrm{opt}}$ is Adam's numerical constant and $\epsilon$ the tube radius.

\medskip
\noindent\textbf{Standing assumption for this appendix.} The sample gradients $\nabla\ell(\cdot; x_i, y_i)$ are uniformly $L$-Lipschitz in $\boldsymbol{\theta}$. No bound on their magnitude is assumed, since Assumption \ref{assum:heavy_tail} permits none, and the boundedness hypothesis of Proposition \ref{prop:affine_invariance} belongs to the SGD sections and is not used here. For AdamW we assume in addition that $\|\boldsymbol{\theta}_t\|_\infty \le \Theta$ on every coordinate over the horizon, so that the decoupled decay step is bounded by $\eta\lambda\Theta$ and every step bound below holds for AdamW with $\Delta_{\max}$ replaced by $\Delta_{\max} + \eta\lambda\Theta$. Lipschitzness of the mini-batch gradient and of $\nabla\mathcal{L}$ follow by averaging, exactly as in the first display of the proof of Proposition \ref{prop:affine_invariance}, which uses Lipschitzness alone. We fix a horizon of $T$ steps and a confidence level $\delta \in (0,1)$, and every high-probability statement below holds uniformly over $t_0 \le t \le T$ and over the $n$ coordinates of the block under analysis.

\subsection{State Lift, Symmetry Restriction, and the Gradient Noise Model}

At iteration $t$, Adam and AdamW maintain the joint state $\xi_t = (\boldsymbol{\theta}_t, m_t, v_t) \in \mathbb{R}^{3d}$, updated as
\begin{align*}
m_t &= \beta_1 m_{t-1} + (1-\beta_1) g_t, \\
v_t &= \beta_2 v_{t-1} + (1-\beta_2) g_t^{\circ 2}, \\
\boldsymbol{\theta}_t &= \boldsymbol{\theta}_{t-1} - \eta \left( \frac{\hat{m}_t}{\sqrt{\hat{v}_t}+\epsilon_{\mathrm{opt}}} + \lambda \boldsymbol{\theta}_{t-1} \mathbbm{1}_{\text{AdamW}} \right),
\end{align*}
where $g_t^{\circ 2}$ denotes the elementwise square and $\hat{m}_t, \hat{v}_t$ are the bias-corrected estimates. This recursion is not Markov in $\boldsymbol{\theta}_t$ alone, so any invariant set must be stated jointly over $(\boldsymbol{\theta}, m, v)$.

\begin{definition}[Joint Invariant Set]
\label{def:joint_invariant}
Let $P_\pi \in \mathbb{R}^{d \times d}$ be a coordinate permutation matrix. Define the joint invariant set
\begin{equation*}
    \tilde{A} := \{(\boldsymbol{\theta}, m, v) \in \mathbb{R}^{3d} \mid P_\pi\boldsymbol{\theta} = \boldsymbol{\theta},\ P_\pi m = m,\ P_\pi v = v\}.
\end{equation*}
\end{definition}

\begin{lemma}[Joint Equivariance of the Adam Recursion]
\label{lemma:joint_equivariance}
If the gradient map is $P_\pi$-equivariant, $g_t(P_\pi\boldsymbol{\theta}) = P_\pi g_t(\boldsymbol{\theta})$, then for every $t$,
\begin{equation*}
    m_t(P_\pi\boldsymbol{\theta}) = P_\pi m_t(\boldsymbol{\theta}), \qquad v_t(P_\pi\boldsymbol{\theta}) = P_\pi v_t(\boldsymbol{\theta}).
\end{equation*}
The AdamW decoupled weight-decay term imposes no further restriction.
\end{lemma}
\begin{proof}
We proceed by induction on $t$. The claim holds trivially at $t=0$ since $m_0=v_0=0$. Assume $m_{t-1}(P_\pi\boldsymbol{\theta}) = P_\pi m_{t-1}(\boldsymbol{\theta})$ and $v_{t-1}(P_\pi\boldsymbol{\theta}) = P_\pi v_{t-1}(\boldsymbol{\theta})$. For the first moment, linearity of the recursion and $P_\pi$-equivariance of $g_t$ give $m_t(P_\pi\boldsymbol{\theta}) = \beta_1 P_\pi m_{t-1}(\boldsymbol{\theta}) + (1-\beta_1)P_\pi g_t(\boldsymbol{\theta}) = P_\pi m_t(\boldsymbol{\theta})$. This step holds for any $Q$ with $QQ^\top=I$. For the second moment, write $(P_\pi g)_i = g_{\pi(i)}$ for the permutation $\pi$. Then $\big((P_\pi g)^{\circ 2}\big)_i = g_{\pi(i)}^2 = \big(g^{\circ 2}\big)_{\pi(i)} = \big(P_\pi(g^{\circ 2})\big)_i$, since permuting coordinates and squaring elementwise commute exactly, because squaring acts identically on every coordinate regardless of label. Consequently $v_t(P_\pi\boldsymbol{\theta}) = \beta_2 P_\pi v_{t-1}(\boldsymbol{\theta}) + (1-\beta_2)P_\pi\big(g_t(\boldsymbol{\theta})^{\circ 2}\big) = P_\pi v_t(\boldsymbol{\theta})$. This commutation fails for a general orthogonal $Q$, since $(Qg)_i^{2} = \big(\sum_j Q_{ij}g_j\big)^2$ mixes coordinates before squaring. We therefore work with the coordinate-permutation group throughout, which is also the group generated by $S_n$ used in Theorem \ref{thm:affine_generation}. The bias correction scales by a deterministic scalar and preserves equivariance. The AdamW term $\lambda\boldsymbol{\theta}_{t-1}$ is linear in $\boldsymbol{\theta}$ and imposes no additional restriction. Clipping acts coordinatewise by the same map on every coordinate, so it commutes with $P_\pi$ as well, and the clipped recursion below inherits the lemma.
\end{proof}

\begin{remark}
Lemma \ref{lemma:joint_equivariance} restricts the admissible symmetry class of the Approximate $Q$-Symmetry condition (Definition \ref{def:approximate_q_symmetry}) from general orthogonal or symmetric $Q$ to coordinate permutations. Neuron-permutation invariance, the source of the $S_n$-symmetry used throughout Section \ref{sec:main_dsi_percolation}, already lies in this restricted class, so Theorem \ref{thm:affine_generation} and Proposition \ref{thm:hypercondensation} require no modification.
\end{remark}

Having fixed the admissible symmetry class, we turn to the statistical model of the gradient noise itself.

\begin{assumption}[Heavy-Tailed Gradient Noise]
\label{assum:heavy_tail}
There exist $p \in (1,2]$ and $\sigma > 0$ such that $\mathbb{E}[|g_{t,i}|^p \mid \mathcal{F}_{t-1}] \le \sigma^p$ almost surely for every coordinate $i$ and every $t$. The conditional form is the one the clipping-bias and tracking bounds use.
\end{assumption}

Assumption \ref{assum:heavy_tail} is the standard model for gradient noise in attention-based architectures \citep{zhang2020adaptive}, and permits infinite variance when $p<2$. The following representation grounds $\sigma^p$ as a computable functional of the gradient's characteristic function, making Assumption \ref{assum:heavy_tail} checkable in practice.

\begin{lemma}[Fractional Moment via the Characteristic Function]
\label{lemma:fractional_moment}
Let $\phi_t(u) := \mathbb{E}[e^{iug_{t,i}}]$. For $p \in (1,2)$,
\begin{equation*}
    \mathbb{E}|g_{t,i}|^p = \frac{2\,\Gamma(p+1)\sin(\pi p/2)}{\pi} \int_0^\infty \frac{1 - \mathrm{Re}\,\phi_t(u)}{u^{p+1}}\,du.
\end{equation*}
\end{lemma}
\begin{proof}
For $p\in(0,2)$ and any $x\in\mathbb{R}$, $|x|^p = c_p^{-1}\int_0^\infty u^{-p-1}(1-\cos(ux))\,du$ with $c_p := \pi/(2\Gamma(p+1)\sin(\pi p/2))$ \citep{samorodnitsky1994stable}. Taking expectations and exchanging with the $u$-integral by Tonelli's theorem, valid since the integrand is non-negative, gives $\mathbb{E}|g_{t,i}|^p = c_p^{-1}\int_0^\infty u^{-p-1}(1-\mathbb{E}[\cos(ug_{t,i})])\,du$, and $\mathbb{E}[\cos(ug_{t,i})]=\mathrm{Re}\,\phi_t(u)$ gives the stated form.
\end{proof}

\begin{remark}
The restriction $p\in(1,2)$ excludes the finite-variance case $p=2$, where the identity of Lemma \ref{lemma:fractional_moment} degenerates as $\sin(\pi p/2)\to 0$. There, $\mathbb{E}[g_{t,i}^2]=-\phi_t''(0)$ directly. Convergence of the integral requires $1-\mathrm{Re}\,\phi_t(u)=O(u^p)$ as $u\to 0$, the standard local regularity condition replacing an ordinary second derivative when $p<2$.
\end{remark}

\subsection{Clipping, Filters, and the Step Bound}

Under Assumption \ref{assum:heavy_tail} with $p<2$, $g_t$ may have infinite variance, so neither the momentum nor the second-moment recursion admits moment estimates directly. Practice resolves this by clipping, and we analyze the clipped recursion.

\begin{definition}[Clipped Gradient]
\label{def:truncated_gradient}
For a fixed threshold $\tau > 0$ define $\mathrm{clip}_\tau(x) := \mathrm{sign}(x)\min(|x|, \tau)$ coordinatewise and $\hat{g}_t := \mathrm{clip}_\tau(g_t)$. The map $\mathrm{clip}_\tau$ is $1$-Lipschitz, fixes $[-\tau,\tau]$ pointwise, and satisfies $|x - \mathrm{clip}_\tau(x)| = (|x|-\tau)_+ \le |x|\,\mathds{1}\{|x| > \tau\}$.
\end{definition}

\begin{lemma}[Clipping Bias and Coincidence]
\label{lemma:truncation_bias}
Under Assumption \ref{assum:heavy_tail}, $\big|\mathbb{E}[g_{t,i} \mid \mathcal{F}_{t-1}] - \mathbb{E}[\hat{g}_{t,i} \mid \mathcal{F}_{t-1}]\big| \le \sigma^p\tau^{1-p}$ for every coordinate, and on a run of $T$ steps the clipped and unclipped recursions produce identical iterates on an event of probability at least $1 - dT\sigma^p\tau^{-p}$.
\end{lemma}
\begin{proof}
The bias equals $\mathbb{E}[(g_{t,i} - \mathrm{clip}_\tau(g_{t,i})) \mid \mathcal{F}_{t-1}]$, bounded in absolute value by $\mathbb{E}[|g_{t,i}|\mathds{1}\{|g_{t,i}|>\tau\} \mid \mathcal{F}_{t-1}]$. By H\"older's inequality with exponents $p, p/(p-1)$ and Markov's inequality applied to $|g_{t,i}|^p$,
\begin{equation*}
    \mathbb{E}\big[|g_{t,i}|\mathds{1}\{|g_{t,i}|>\tau\}\big] \le \left(\mathbb{E}|g_{t,i}|^p\right)^{1/p}\mathbb{P}(|g_{t,i}|>\tau)^{(p-1)/p} \le \sigma\left(\frac{\sigma^p}{\tau^p}\right)^{(p-1)/p}=\sigma^p\tau^{1-p}.
\end{equation*}
Clipping is inactive when $|g_{s,i}| \le \tau$ for all $s \le T$ and all $i$. Markov's inequality and a union bound over the $dT$ coordinates bound the complementary probability by $dT\sigma^p\tau^{-p}$, and on that event the two recursions receive identical inputs and produce identical iterates.
\end{proof}

The coincidence bound is informative for $\tau \gg (dT)^{1/p}\sigma$. The analysis concerns clipped Adam, the algorithm that is run in practice, and Lemma \ref{lemma:truncation_bias} transfers its conclusions to unclipped Adam on the event that clipping is inactive.

Since $\hat{g}_t \in [-\tau,\tau]^d$ almost surely, $\hat m_t, \hat v_t$ have finite moments of every order, with $|\hat m_{t,i}| \le \tau$ and $\hat v_{t,i} \in [0,\tau^2]$. We write the bias-corrected filters as explicit weighted sums,
\begin{equation*}
    \hat m_t = \sum_{s=1}^{t} w^{(1,t)}_{t-s}\,\hat g_s,\quad w^{(1,t)}_k := \frac{(1-\beta_1)\beta_1^k}{1-\beta_1^{t}}, \qquad
    \hat v_t = \sum_{s=1}^{t} w^{(2,t)}_{t-s}\,\hat g_s^{\circ2},\quad w^{(2,t)}_k := \frac{(1-\beta_2)\beta_2^k}{1-\beta_2^{t}},
\end{equation*}
each family of weights summing to one over $0 \le k \le t-1$. The transfer-function view of these filters fixes their memory scales.

\begin{definition}[Shift Operator and $z$-Transform]
\label{def:z_transform}
For a sequence $\{x_t\}_{t\ge 0}$, define the shift operator $(Sx)_t := x_{t-1}$, with $x_{-1}:=0$, and the one-sided $z$-transform \citep{oppenheim2010discrete}
\begin{equation*}
    X(z) := \mathcal{Z}\{x_t\} := \sum_{t=0}^\infty x_t z^{-t},
\end{equation*}
defined on the region $|z|>\rho$ for the smallest $\rho\ge 0$ for which the series converges absolutely.
\end{definition}

\begin{lemma}[Shift Property]
\label{lemma:shift_property}
$\mathcal{Z}\{Sx\}(z) = z^{-1}X(z)$.
\end{lemma}
\begin{proof}
$\mathcal{Z}\{Sx\}(z)=\sum_{t\ge0}x_{t-1}z^{-t}=\sum_{t\ge1}x_{t-1}z^{-t}$, since the $t=0$ term vanishes because $x_{-1}=0$, giving $z^{-1}\sum_{s\ge0}x_sz^{-s}=z^{-1}X(z)$ after reindexing $s=t-1$.
\end{proof}

Let $\hat G(z):=\mathcal Z\{\hat g_t\}$. Since $|\hat g_t|\le\tau$, the series converges for all $|z|>1$.

\begin{lemma}[Momentum and Second-Moment Filters]
\label{lemma:momentum_filter}
$\hat M(z) = \dfrac{1-\beta_1}{1-\beta_1 z^{-1}}\hat G(z)$ and $\hat V(z) = \dfrac{1-\beta_2}{1-\beta_2 z^{-1}}\mathcal{Z}\{\hat g_t^{\circ 2}\}$, with poles at $z=\beta_1$ and $z = \beta_2$ and effective memories $\tau_1:=1/(1-\beta_1)$, $\tau_2 := 1/(1-\beta_2)$.
\end{lemma}
\begin{proof}
Applying $\mathcal Z\{\cdot\}$ to $\hat m_t = \beta_1\hat m_{t-1}+(1-\beta_1)\hat g_t$ and using Lemma \ref{lemma:shift_property}, $\hat M(z)=\beta_1z^{-1}\hat M(z)+(1-\beta_1)\hat G(z)$. Solving for $\hat M(z)$ gives the stated form. The second recursion is linear in the sequence $\hat g_t^{\circ2}$, whose transform exists since the sequence is bounded by $\tau^2$, and the same computation applies. The poles correspond to impulse responses $(1-\beta_j)\beta_j^k$ with decay times $1/(1-\beta_j)$.
\end{proof}

\begin{lemma}[Step Bound]
\label{lemma:step_bound}
Assume $\beta_1^2 < \beta_2$ and $t \ge t_0$ with $\beta_1^{t_0}, \beta_2^{t_0} \le \tfrac12$. Then every coordinate of every Adam step satisfies
\begin{equation*}
    |\Delta\theta_{t,i}| \le \Delta_{\max} := 2\eta C_\beta, \qquad C_\beta := \frac{1-\beta_1}{\sqrt{(1-\beta_2)(1-\beta_1^2/\beta_2)}},
\end{equation*}
with no lower bound on $\hat v$.
\end{lemma}
\begin{proof}
By Cauchy--Schwarz,
\begin{equation*}
    |m_{t,i}| = \Big|\sum_{k}(1-\beta_1)\beta_1^k\,\hat g_{t-k,i}\Big| \le \Big(\sum_k \frac{(1-\beta_1)^2\beta_1^{2k}}{(1-\beta_2)\beta_2^k}\Big)^{1/2}\Big(\sum_k (1-\beta_2)\beta_2^k\,\hat g_{t-k,i}^2\Big)^{1/2} = C_\beta\sqrt{v_{t,i}},
\end{equation*}
the first factor summing a geometric series with ratio $\beta_1^2/\beta_2 < 1$. Bias correction multiplies the ratio $m/\sqrt v$ by $\sqrt{1-\beta_2^t}/(1-\beta_1^t) \le 2$ for $t \ge t_0$. Hence $|\Delta\theta_{t,i}| = \eta|\hat m_{t,i}|/(\sqrt{\hat v_{t,i}}+\epsilon_{\mathrm{opt}}) \le \eta|\hat m_{t,i}|/\sqrt{\hat v_{t,i}} \le 2\eta C_\beta$. The weight-decay term is handled separately below.
\end{proof}

\subsection{Temporal Tracking}

Appendix \ref{app:sde_and_invariance} samples mini-batches with replacement, so the clipped gradient at step $s$ is conditionally independent of the past given $\boldsymbol{\theta}_s$. Define the martingale differences
\begin{equation*}
    \xi_s := \hat g_s - \mathbb{E}[\hat g_s \mid \mathcal{F}_{s-1}], \qquad \zeta_s := \hat g_s^{\circ 2} - \mathbb{E}[\hat g_s^{\circ 2} \mid \mathcal{F}_{s-1}],
\end{equation*}
with $|\xi_{s,i}| \le 2\tau$ and $|\zeta_{s,i}| \le \tau^2$, and the targets
\begin{equation*}
    m^*_i(\boldsymbol{\theta}) := \mathbb{E}[\hat g_{t,i}^2 \mid \boldsymbol{\theta}_t = \boldsymbol{\theta}], \qquad \mathbb{E}[\hat g_s \mid \mathcal{F}_{s-1}] = \nabla\mathcal{L}(\boldsymbol{\theta}_s) + b_s,\quad |b_{s,i}| \le \sigma^p\tau^{1-p},
\end{equation*}
the bias bound coming from Lemma \ref{lemma:truncation_bias}.

\begin{lemma}[Lipschitz Targets]
\label{lemma:lipschitz_targets}
$m^*_i$ is $L_m$-Lipschitz with $L_m = 2\tau L$, and $\nabla\mathcal{L}$ is $L$-Lipschitz with $|\partial_i\mathcal{L}| \le \sigma$.
\end{lemma}
\begin{proof}
For a fixed mini-batch realization, $|\mathrm{clip}_\tau(g_i(\boldsymbol{\theta}_1))^2 - \mathrm{clip}_\tau(g_i(\boldsymbol{\theta}_2))^2| \le 2\tau\,|\mathrm{clip}_\tau(g_i(\boldsymbol{\theta}_1)) - \mathrm{clip}_\tau(g_i(\boldsymbol{\theta}_2))| \le 2\tau L\|\boldsymbol{\theta}_1 - \boldsymbol{\theta}_2\|$, using $|a^2-b^2| \le 2\tau|a-b|$ on $[-\tau,\tau]$, the $1$-Lipschitz clip, and the $L$-Lipschitz mini-batch gradient, itself an average of $L$-Lipschitz sample gradients. Averaging over the mini-batch law preserves the bound. Lipschitzness of $\nabla\mathcal{L}$ follows from the standing assumption by averaging, and $|\partial_i\mathcal{L}| = |\mathbb{E} g_{t,i}| \le (\mathbb{E}|g_{t,i}|^p)^{1/p} \le \sigma$.
\end{proof}

\begin{lemma}[Window Drift Bound]
\label{lemma:window_drift}
For $t \ge t_0$,
\begin{gather*}
    \Big|\sum_{s} w^{(2,t)}_{t-s}\big(m^*_i(\boldsymbol{\theta}_s)-m^*_i(\boldsymbol{\theta}_t)\big)\Big| \le 2L_m\sqrt{d}\,\Delta_{\max}\,\frac{\beta_2}{1-\beta_2}, \\
    \Big\|\sum_{s} w^{(1,t)}_{t-s}\big(\nabla\mathcal{L}(\boldsymbol{\theta}_s)-\nabla\mathcal{L}(\boldsymbol{\theta}_t)\big)\Big\| \le 2L\sqrt{d}\,\Delta_{\max}\,\frac{\beta_1}{1-\beta_1}.
\end{gather*}
\end{lemma}
\begin{proof}
Lemma \ref{lemma:step_bound} and the telescoping sum of steps give $\|\boldsymbol{\theta}_s - \boldsymbol{\theta}_t\| \le \sqrt{d}\,(t-s)\Delta_{\max}$, the factor $\sqrt d$ converting the coordinatewise bound into a Euclidean one. Lipschitz continuity (Lemma \ref{lemma:lipschitz_targets}) turns this into $|m^*_i(\boldsymbol{\theta}_s) - m^*_i(\boldsymbol{\theta}_t)| \le L_m\sqrt{d}\,(t-s)\Delta_{\max}$. Summing against the weights, $\sum_k k\,w^{(2,t)}_k = \frac{(1-\beta_2)}{1-\beta_2^t}\sum_k k\beta_2^k \le \frac{\beta_2}{(1-\beta_2)(1-\beta_2^t)} \le \frac{2\beta_2}{1-\beta_2}$ for $t \ge t_0$. The momentum bound is the same computation with $(L, \beta_1)$.
\end{proof}

\begin{lemma}[Filtered Noise Variance]
\label{lemma:filtered_noise_variance}
For $t \ge t_0$,
\begin{equation*}
    \mathrm{Var}\Big[\sum_s w^{(2,t)}_{t-s}\,\zeta_{s,i}\Big] \le 4\,\frac{1-\beta_2}{1+\beta_2}\,\tau^4, \qquad
    \mathrm{Var}\Big[\sum_s w^{(1,t)}_{t-s}\,\xi_{s,i}\Big] \le 16\,\frac{1-\beta_1}{1+\beta_1}\,\tau^2.
\end{equation*}
\end{lemma}
\begin{proof}
Martingale differences are orthogonal in $L^2$, so the variance of the weighted sum equals $\sum_k (w^{(2,t)}_k)^2\,\mathbb{E}[\zeta_{t-k,i}^2] \le \tau^4 \frac{(1-\beta_2)^2}{(1-\beta_2^t)^2}\sum_k \beta_2^{2k} \le 4\tau^4\,\frac{1-\beta_2}{1+\beta_2}$. The same computation with $|\xi_{s,i}| \le 2\tau$ gives the momentum bound. Neither stationarity nor decorrelation of the noise is used, since the martingale structure of fresh mini-batches replaces the correlation-time hypothesis that would otherwise be required.
\end{proof}

The variance bounds hold at a fixed $t$. The supermartingale argument below needs the tracking errors controlled at every step up to the exit time, so we pass from variances to exponential tail bounds and take a union over the horizon.

\begin{lemma}[Uniform Tracking over the Horizon]
\label{lemma:uniform_tracking}
Set
\begin{equation*}
    \delta' := \tau^2\sqrt{8\,\tfrac{1-\beta_2}{1+\beta_2}\,\log\tfrac{2nT}{\delta}}, \qquad
    \delta'' := \tau\sqrt{32\,\tfrac{1-\beta_1}{1+\beta_1}\,\log\tfrac{2nT}{\delta}}.
\end{equation*}
With probability at least $1-\delta$, simultaneously for every $t_0 \le t \le T$ and every coordinate $i$ of the block,
\begin{gather*}
    \Big|\sum_s w^{(2,t)}_{t-s}\,\zeta_{s,i}\Big| \le \delta', \\ \text{and, on a second event of probability at least } 1-\delta, \qquad \Big|\sum_s w^{(1,t)}_{t-s}\,\xi_{s,i}\Big| \le \delta''.
\end{gather*}
\end{lemma}
\begin{proof}
For fixed $t$ and $i$, $\sum_s w^{(2,t)}_{t-s}\zeta_{s,i}$ is a sum of martingale differences with increments bounded by $w^{(2,t)}_{t-s}\tau^2$, so the Azuma--Hoeffding inequality gives $\mathbb{P}(|\sum_s w^{(2,t)}_{t-s}\zeta_{s,i}| \ge x) \le 2\exp\big(-x^2/(2\tau^4\sum_k (w^{(2,t)}_k)^2)\big) \le 2\exp\big(-x^2(1+\beta_2)/(8\tau^4(1-\beta_2))\big)$, using $\sum_k (w^{(2,t)}_k)^2 \le 4\frac{1-\beta_2}{1+\beta_2}$ from the proof of Lemma \ref{lemma:filtered_noise_variance}. Setting the right side to $\delta/(nT)$ gives $x = \delta'$, and a union bound over the $nT$ pairs $(t,i)$ gives the first event. The same computation with increments bounded by $2\tau w^{(1,t)}_{t-s}$ and $\sum_k (w^{(1,t)}_k)^2 \le 4\frac{1-\beta_1}{1+\beta_1}$ gives the second with $x = \delta''$. Boundedness of the clipped gradients makes the exponential bound available, and the horizon enters only through $\sqrt{\log(nT/\delta)}$.
\end{proof}

\begin{corollary}[Temporal Tracking Bound]
\label{cor:temporal_tracking}
On the first event of Lemma \ref{lemma:uniform_tracking}, for every $t_0 \le t \le T$ and every coordinate $i$ of the block,
\begin{equation*}
    |\hat v_{t,i} - m^*_i(\boldsymbol{\theta}_t)| \le 2L_m\sqrt{d}\,\Delta_{\max}\tfrac{\beta_2}{1-\beta_2} + \delta'.
\end{equation*}
\end{corollary}
\begin{proof}
Since $\mathbb{E}[\hat g_{s,i}^2 \mid \mathcal{F}_{s-1}] = m^*_i(\boldsymbol{\theta}_s)$,
\begin{equation*}
    \hat v_{t,i}-m^*_i(\boldsymbol{\theta}_t) = \sum_s w^{(2,t)}_{t-s}\big(m^*_i(\boldsymbol{\theta}_s)-m^*_i(\boldsymbol{\theta}_t)\big) + \sum_s w^{(2,t)}_{t-s}\,\zeta_{s,i},
\end{equation*}
and the weights sum to one. Combine Lemma \ref{lemma:window_drift} with Lemma \ref{lemma:uniform_tracking}.
\end{proof}

\begin{lemma}[Momentum Tracking]
\label{lemma:momentum_tracking}
On the second event of Lemma \ref{lemma:uniform_tracking}, for every $t_0 \le t \le T$ and every coordinate $i$ of the block,
\begin{equation*}
    |\hat m_{t,i} - \partial_i\mathcal{L}(\boldsymbol{\theta}_t)| \le r_{\max} + \delta'', \qquad r_{\max} := 2L\sqrt{d}\,\Delta_{\max}\frac{\beta_1}{1-\beta_1} + \sigma^p\tau^{1-p}.
\end{equation*}
\end{lemma}
\begin{proof}
Decompose
\begin{equation*}
    \hat m_t - \nabla\mathcal{L}(\boldsymbol{\theta}_t) = \sum_s w^{(1,t)}_{t-s}\big(\nabla\mathcal{L}(\boldsymbol{\theta}_s) - \nabla\mathcal{L}(\boldsymbol{\theta}_t)\big) + \sum_s w^{(1,t)}_{t-s}\, b_s + \sum_s w^{(1,t)}_{t-s}\,\xi_s,
\end{equation*}
and bound the three terms by Lemma \ref{lemma:window_drift}, by Lemma \ref{lemma:truncation_bias} with weights summing to one, and by Lemma \ref{lemma:uniform_tracking}.
\end{proof}

\begin{remark}[Momentum Lag]
The momentum is the numerator of the update, and controlling $\hat v$ alone leaves it untouched. The exact recursion lets $\hat m_t$ point against the current gradient after a sign change. From $m_0 = 0$ with $g_1 = -1$ and $g_2 = a \in (0, \beta_1)$, $\hat m_2 = (\beta_1 g_1 + g_2)/(1+\beta_1) = (a - \beta_1)/(1+\beta_1) < 0$ while $g_2 > 0$. Lemma \ref{lemma:momentum_tracking} bounds this lag by the gradient variation over one memory window, $L\sqrt d\,\Delta_{\max}\beta_1/(1-\beta_1)$, and this quantity enters the transverse drift below.
\end{remark}

\subsection{Cross-Sectional Homogeneity}

Corollary \ref{cor:temporal_tracking} bounds the error accrued from a single coordinate's own second-moment estimate drifting away from its target over time. A second source of error arises across coordinates within a symmetric block. Let $A$ be generated by the full permutation group on an $S_n$-symmetric block, as in Theorem \ref{thm:affine_generation}. Every coordinate in such a block projects onto the same point of $A$; for units with several coordinates this holds role by role, over the coordinates the permutation exchanges.

\begin{lemma}[Common Projection Point]
\label{lemma:common_projection}
For every coordinate $i$ in the block, $\pi_A(\boldsymbol\theta)_i = \bar\theta^{\mathrm{blk}} := \frac{1}{n}\sum_{k\in\mathrm{block}}\theta_k$.
\end{lemma}
\begin{proof}
$A$ restricted to the block is $\{\theta\in\mathbb R^n: \theta_i=\theta_j\ \forall i,j\}$, since $A$ is the fixed-point set of the full permutation group on the block (Theorem \ref{thm:affine_generation}). The orthogonal projection of $\theta$ onto this subspace minimizes $\sum_k(\theta_k-c)^2$ over the diagonal point $c\mathbf 1$. Differentiating in $c$ gives $c=\bar\theta^{\mathrm{blk}}$.
\end{proof}

\begin{assumption}[Non-Degeneracy]
\label{assum:nondegeneracy}
$\hat{v}_i(\boldsymbol{\theta}_t) \ge v_{\min} > 0$ for every coordinate $i$ under analysis, for all $t \le \tau_\epsilon^{\mathrm{blk}}$.
\end{assumption}

\begin{lemma}[Pathwise Block Homogeneity]
\label{lemma:block_homogeneity}
Let $\tau_\epsilon^{\mathrm{blk}}:=\min_{k\in\mathrm{block}}\tau_\epsilon^{(k)}$. For $t\le\tau_\epsilon^{\mathrm{blk}}$ and any coordinates $i,j$ of the block exchanged by the permutation, almost surely,
\begin{equation*}
    |\hat v_{t,i}-\hat v_{t,j}| \le 4\sqrt{2}\,L\tau\sqrt{\epsilon}, \qquad \|E_t\|_\infty \le \kappa\cdot 4\sqrt{2}\,L\tau\sqrt{\epsilon},\quad \kappa:=\frac{1}{2\sqrt{v_{\min}}(\sqrt{v_{\min}}+\epsilon_{\mathrm{opt}})^2},
\end{equation*}
where $E_t$ denotes the deviation of $P(\boldsymbol\theta_t)|_{\mathrm{block}}=\mathrm{diag}(1/(\sqrt{\hat v_t}+\epsilon_{\mathrm{opt}}))$ from $c_t$, the diagonal matrix that assigns to each coordinate the average of $1/(\sqrt{\hat v_t}+\epsilon_{\mathrm{opt}})$ over the coordinates the permutation exchanges with it (one value per coordinate role, common to both units), and $\kappa$ is evaluated using Assumption \ref{assum:nondegeneracy}. If the pair enters the tube at step $t_{\mathrm{enter}}$, gradients from before entry add at most $2\tau^2\beta_2^{\,t - t_{\mathrm{enter}}}$ to the first bound, which falls below it after a burn-in of $\tau_2\ln\!\big(\tau/(2\sqrt{2}\,L\sqrt{\epsilon})\big)$ steps.
\end{lemma}
\begin{proof}
By Lemma \ref{lemma:common_projection} and Lemma \ref{lemma:stopped_bounded}, $|\theta_{t,i}-\theta_{t,j}| \le \sqrt{Y_t^{(i)}}+\sqrt{Y_t^{(j)}} \le 2\sqrt{\epsilon}$ for $t \le \tau_\epsilon^{\mathrm{blk}}$. Let $P_{ij}$ swap coordinates $i$ and $j$. Equivariance gives $g_{s,i}(\boldsymbol{\theta}) = g_{s,j}(P_{ij}\boldsymbol{\theta})$ for every mini-batch, and $\|P_{ij}\boldsymbol{\theta}-\boldsymbol{\theta}\| = \sqrt{2}\,|\theta_i - \theta_j|$, so
\begin{equation*}
    |g_{s,i} - g_{s,j}| = |g_{s,j}(P_{ij}\boldsymbol{\theta}_s) - g_{s,j}(\boldsymbol{\theta}_s)| \le L\sqrt{2}\,|\theta_{s,i}-\theta_{s,j}| \le 2\sqrt{2}\,L\sqrt{\epsilon}.
\end{equation*}
Clipping is $1$-Lipschitz, so the same bound holds for $\hat g$ with no boundary case. A hard truncation $x\,\mathds{1}\{|x|\le\tau\}$ would separate the inputs $\tau \pm h$ by $\tau - h$, while $\mathrm{clip}_\tau$ separates them by $h$. Then $|\hat g_{s,i}^2-\hat g_{s,j}^2| \le 2\tau|\hat g_{s,i}-\hat g_{s,j}| \le 4\sqrt{2}\,L\tau\sqrt{\epsilon}$ for every $s\le t$, and averaging with the weights $w^{(2,t)}$ gives the bound on $\hat v$. Each entry of $c_t$ is a convex combination of the $1/(\sqrt{\hat v_{t,k}}+\epsilon_{\mathrm{opt}})$ over exchanged coordinates. The map $v\mapsto1/(\sqrt v+\epsilon_{\mathrm{opt}})$ has derivative bounded by $\kappa$ on $[v_{\min},\infty)$, and the mean value theorem transfers the bound to $\|E_t\|_\infty$.
\end{proof}

\subsection{The Reduced Transverse Dynamics}

The update on the block is $\Delta\boldsymbol{\theta}_t = -\eta(c_t + E_t)\hat m_t - \eta\lambda\boldsymbol{\theta}_{t-1}\mathds{1}_{\mathrm{AdamW}}$, and the transverse coordinate $\boldsymbol{x}_t = P^\perp\boldsymbol{\theta}_t$ follows by applying $P^\perp$. The term $E_t = O(\sqrt\epsilon)$, bounded through the history of $|\theta_i - \theta_j|$ on the tube, multiplies $\hat m_t = O(1)$, which is the order of $c_t\nabla_\perp\mathcal{L}$ at the tube boundary, and its sign follows that of $(\hat v_{t,i} - \hat v_{t,j})\,\partial_j\mathcal{L}$. We therefore freeze it.

\begin{assumption}[Quasi-Static Preconditioner]
\label{assum:quasistatic}
$1 - \beta_2$ is small enough that $\hat v$ on the block is constant over the trapping horizon, so $c_t = c$ and $E_t = E$.
\end{assumption}

Assumption \ref{assum:quasistatic} describes the regime in which $\tau_2$ exceeds the transverse relaxation time. The quantity that checks it on a run is the relative variation $\max_{t_0 \le t \le T}|\hat v_{t,k} - \hat v_{t_0,k}|/\hat v_{t_0,k}$ over the block, and under it $E$ is a constant to be measured rather than bounded. Remark \ref{remark:adam_diagnostic} reports this relative variation on the grokking task.

\begin{proposition}[Preconditioned SGF Reduction]
\label{thm:adam_sgf_reduction}
Under Assumptions \ref{assum:heavy_tail}, \ref{assum:nondegeneracy} and \ref{assum:quasistatic}, and reading the exponentially filtered momentum noise as white noise with its long-run covariance, for $t_0 \le t \le \tau_\epsilon^{\mathrm{blk}}$ the transverse coordinate of the block evolves as the Euler--Maruyama step, at the resolution of one iteration, of
\begin{equation}
\label{eq:adam_reduced}
    d\boldsymbol{x}_t = -\eta\, c\, \nabla_\perp\mathcal{L}(\boldsymbol{\theta}_t)\,dt - \eta\lambda\,\boldsymbol{x}_t\,\mathds{1}_{\mathrm{AdamW}}\,dt + \boldsymbol{f}\,dt + \boldsymbol{\varrho}_t\,dt + \eta\, c\,(\Sigma^\tau_\perp(\boldsymbol{\theta}_t))^{1/2}\,dW_t,
\end{equation}
where $\boldsymbol{f} := -\eta P^\perp E\,\nabla\mathcal{L}(\boldsymbol{\theta}_t)$ satisfies $\|\boldsymbol{f}\| \le \eta\sqrt{n}\,\|E\|_\infty\,\sigma$, $\Sigma^\tau_\perp := P^\perp\,\mathrm{Cov}(\hat g_t \mid \boldsymbol{\theta}_t)\,P^\perp$ is the transverse covariance of the clipped mini-batch gradient, and, on the second event of Lemma \ref{lemma:uniform_tracking}, for every $t_0 \le t \le T$,
\begin{equation*}
    \|\boldsymbol{\varrho}_t\| \le \rho_{\max} := \eta\sqrt{n}\,\big(\|c\|_\infty + \|E\|_\infty\big)\,(r_{\max} + \delta'').
\end{equation*}
\end{proposition}
\begin{proof}
Substitute $\hat m_t = \nabla\mathcal{L}(\boldsymbol{\theta}_t) + \boldsymbol{r}_t + \boldsymbol{\xi}^m_t$ from Lemma \ref{lemma:momentum_tracking}, with $\boldsymbol{r}_t$ the deterministic lag-plus-bias and $\boldsymbol{\xi}^m_t := \sum_s w^{(1,t)}_{t-s}\xi_s$ the filtered martingale noise, into $P^\perp\Delta\boldsymbol{\theta}_t = -\eta P^\perp(c + E)\hat m_t - \eta\lambda\boldsymbol{x}_{t-1}\mathds{1}_{\mathrm{AdamW}}$. Since $c$ is common to both units it commutes with the swap and with $P^\perp$, so the drift $-\eta c\nabla_\perp\mathcal{L}$ and the forcing $\boldsymbol{f}$ come out of $\nabla\mathcal{L}$, and the bound on $\boldsymbol{f}$ uses $|\partial_i\mathcal{L}| \le \sigma$ (Lemma \ref{lemma:lipschitz_targets}). The residual $\boldsymbol{\varrho}_t := -\eta P^\perp(c+E)\boldsymbol{r}_t - \eta P^\perp E\boldsymbol{\xi}^m_t$ collects the deterministic errors and the part of the filtered noise that the deviation $E$ multiplies, and Lemma \ref{lemma:momentum_tracking} bounds $\|\boldsymbol{r}_t + \boldsymbol{\xi}^m_t\|$ coordinatewise by $r_{\max} + \delta''$ on the stated event, which gives $\rho_{\max}$. The remaining noise $-\eta c P^\perp\boldsymbol{\xi}^m_t$ enters as the diffusion of \eqref{eq:adam_reduced} through the white-noise reading stated in the proposition. Weight decay contributes $-\eta\lambda\boldsymbol{x}$ exactly, since $P^\perp$ is linear and $P^\perp\boldsymbol{\theta} = \boldsymbol{x}$.
\end{proof}

\begin{remark}[The White-Noise Reading]
\label{remark:white_noise}
The reading is a modeling step, since the recursion does not imply it, and Theorem \ref{thm:adam_trapping} is stated for the reduced SDE while the tracking bounds are discrete, so the Euler--Maruyama error between the two is not controlled here. The filtered noise $\boldsymbol{\xi}^m_t = \sum_s w^{(1,t)}_{t-s}\xi_s$ is an exponential average of martingale differences, so most of it at step $t$ is carried over from step $t-1$ and only the fresh increment $w^{(1,t)}_0\,\xi_t$ is unpredictable given $\mathcal{F}_{t-1}$. Its per-step variance is smaller than its long-run variance by the factor $\frac{1-\beta_1}{1+\beta_1}$, and on time scales beyond $\tau_1$ its increments decorrelate and the long-run covariance equals that of its input, $\eta^2 c\,\Sigma^\tau_\perp c$, which is the covariance \eqref{eq:adam_reduced} carries. The step that turns SGD into the SGF of Appendix \ref{app:sde_and_invariance} swaps independent increments for Gaussian ones, while this one swaps correlated increments for independent ones, so it is a modeling step of the same kind and a stronger one. The statements below take it as a hypothesis. The filtered noise also enters $\rho_{\max}$ through $c\,\delta''$, so it is counted once as forcing and once as diffusion, which errs on the safe side. A route that drops the reading altogether bounds the whole increment pathwise on the annulus, using Lemma \ref{lemma:uniform_tracking}, and expands $Y^\alpha$ to second order with a remainder. It yields a discrete supermartingale inequality without the concavity term on the right of \eqref{eq:adam_condition}, with the pathwise second-order piece $\|\Delta\boldsymbol{x}\|^2 \le n\Delta_{\max}^2$ for Adam, and $\le (\sqrt{n}\,\Delta_{\max} + \eta\lambda\|\boldsymbol{x}\|)^2$ for AdamW since Lemma \ref{lemma:step_bound} excludes the decay step, in place of $\eta^2\operatorname{tr}(c\,\Sigma^\tau_\perp c)$. Dropping the concavity term discards a favorable contribution, so the condition it produces is stronger than \eqref{eq:adam_condition} and needs no white-noise reading.
\end{remark}

This reduced transverse dynamics has the structural form of the vanilla-SGD flow, with the drift and the diffusion preconditioned by $c$ and hence the covariance by $c(\cdot)c$, a forcing $\boldsymbol{f}$ off $A$, and a residual. We now state the condition under which the trapping argument survives. Every statement below carries the white-noise reading of Remark \ref{remark:white_noise} as a hypothesis.

\begin{theorem}[Conditional Trapping Criterion for the Reduced Adam Dynamics]
\label{thm:adam_trapping}
Fix $\alpha \in (0,1]$ and $0 < r_0 < \sqrt{\epsilon}$, let $\rho_\parallel \le \rho_{\max}$ be a constant with $|\boldsymbol{x}_t^\top\boldsymbol{\varrho}_t| \le \rho_\parallel\sqrt{Y_t}$ for $t_0 \le t \le T$ on the tracking event of Lemma \ref{lemma:uniform_tracking}, and suppose that on the annulus $\mathcal{N}_{r_0,\epsilon} := \{r_0^2 \le Y < \epsilon\}$
\begin{equation}
\label{eq:adam_condition}
    -2\eta\,\boldsymbol{x}^\top c\,\nabla_\perp\mathcal{L} - 2\eta\lambda Y\,\mathds{1}_{\mathrm{AdamW}} + 2\,\boldsymbol{x}^\top\boldsymbol{f} + 2\sqrt{Y}\,\rho_\parallel + \eta^2 \operatorname{tr}(c\,\Sigma^\tau_\perp c) \;\le\; 2(1-\alpha)\,\eta^2\,\frac{\boldsymbol{x}^\top c\,\Sigma^\tau_\perp c\,\boldsymbol{x}}{Y}.
\end{equation}
The conclusion also holds with $2\,\boldsymbol{x}^\top\boldsymbol{f} + 2\sqrt{Y}\,\rho_\parallel$ replaced by $2\sqrt{Y}\,(\|\boldsymbol{f}\| + \rho_{\max})$.
Let $\tau_{\mathrm{ann}} := \inf\{t \ge t_0 : Y_t \notin \mathcal{N}_{r_0,\epsilon}\}$. Then $Y^\alpha_{t \wedge \tau_{\mathrm{ann}}}$ is a non-negative supermartingale under \eqref{eq:adam_reduced}.
\end{theorem}
\begin{proof}
It\^o's lemma for $\phi_\alpha(Y) = Y^\alpha$ under \eqref{eq:adam_reduced} gives
\begin{multline*}
    \mathscr{G} Y^\alpha = \alpha Y^{\alpha-1}\Big[-2\eta\,\boldsymbol{x}^\top c\,\nabla_\perp\mathcal{L} - 2\eta\lambda Y\,\mathds{1}_{\mathrm{AdamW}} + 2\boldsymbol{x}^\top(\boldsymbol{f} + \boldsymbol{\varrho}_t) \\ + \eta^2 \operatorname{tr}(c\,\Sigma^\tau_\perp c) - 2(1-\alpha)\,\eta^2\,\frac{\boldsymbol{x}^\top c\,\Sigma^\tau_\perp c\,\boldsymbol{x}}{Y}\Big].
\end{multline*}
By the definition of $\rho_\parallel$, $2\boldsymbol{x}^\top\boldsymbol{\varrho}_t \le 2\sqrt{Y}\rho_\parallel$, so \eqref{eq:adam_condition} makes the bracket non-positive on the annulus, and Cauchy--Schwarz bounds $2\boldsymbol{x}^\top(\boldsymbol{f} + \boldsymbol{\varrho}_t) \le 2\sqrt{Y}(\|\boldsymbol{f}\| + \rho_{\max})$ for the cruder form. Preconditioning drift and noise amplitude by the same $c$ scales the drift at first order and the diffusion at second order in $c$, so the sign of the bracket is decided by \eqref{eq:adam_condition}, and $c > 0$ alone does not determine it. For $dX = -cX\,dt + cX\,dW$ the generator of $X^2$ equals $(c^2 - 2c)X^2$, which is positive for $c > 2$, while the generator of $\log X^2$ equals $-(2c + c^2)$, which is negative for every $c$. The concavity term on the right of \eqref{eq:adam_condition} carries the $c^2$ contribution to the favorable side. In transverse dimension one with $\Sigma^\tau_\perp = \zeta^2 Y$ the two $c^2$ terms combine to $(2\alpha - 1)\,\eta^2 c^2\zeta^2 Y$, so for $\alpha \le \tfrac12$ the noise helps for every $c$, while at $\alpha = 1$ the condition reads $c \le 2$ in the example above and excludes $c = 3$. On the annulus $Y \ge r_0^2$, the integrand of the stochastic integral, $2\alpha Y^{\alpha-1}\eta c\,\boldsymbol{x}^\top(\Sigma^\tau_\perp)^{1/2}$, is bounded, so the integral is a genuine martingale \citep{karatzas1991brownian} and the conditional expectation annihilates it, leaving the non-positive drift.
\end{proof}

\begin{corollary}[Conditional Non-Escape and Confinement]
\label{cor:adam_nonescape}
Over the horizon $t_0 \le t \le T$,
\begin{equation*}
    \mathbb{P}\big(Y \text{ exits the annulus through } \epsilon \text{ before reaching the core and before } T \mid \mathcal{F}_{t_0}\big) \le \epsilon^{-\alpha}\,\mathbb{E}[Y_{t_0}^\alpha \mid \mathcal{F}_{t_0}] + 2\delta.
\end{equation*}
Let $\tilde Y := \|\boldsymbol{x} - \boldsymbol{x}^*\|^2$, where $\boldsymbol{x}^*$ is the zero of the transverse drift, and suppose that $\mathscr{G}\tilde Y \le -2\kappa\tilde Y + C$ on $\{\tilde Y < \epsilon\}$ for constants $\kappa > 0$ and $C \ge 0$. Then
\begin{equation*}
    \mathbb{P}\big(\tilde Y \text{ exits the tube before } T \mid \mathcal{F}_{t_0}\big) \le \frac{\tilde Y_{t_0} + C\,(T - t_0)}{\epsilon} + 2\delta .
\end{equation*}
\end{corollary}
\begin{proof}
Let $\tau_{\mathrm{bad}}$ be the first step $t \ge t_0$ at which either tracking bound of Lemma \ref{lemma:uniform_tracking} fails. Both bounds are $\mathcal{F}_t$-measurable, so $\tau_{\mathrm{bad}}$ is a stopping time, and Lemma \ref{lemma:uniform_tracking} gives $\mathbb{P}(\tau_{\mathrm{bad}} \le T) \le 2\delta$. Before $\tau_{\mathrm{ann}} \wedge \tau_{\mathrm{bad}}$ the residual bound $\rho_{\max}$, and with it $\rho_\parallel \le \rho_{\max}$, holds at every step, so \eqref{eq:adam_condition} holds and Theorem \ref{thm:adam_trapping} makes $Y^\alpha_{t \wedge \tau_{\mathrm{ann}} \wedge \tau_{\mathrm{bad}}}$ a non-negative supermartingale without conditioning on any event. Doob's maximal inequality \citep{revuz1999continuous} bounds $\mathbb{P}(\sup_{t \le T} Y^\alpha_{t\wedge\tau_{\mathrm{ann}}\wedge\tau_{\mathrm{bad}}} \ge \epsilon^\alpha)$ by the first term. An exit through $\epsilon$ before $T$ either happens before $\tau_{\mathrm{bad}}$, which the stopped process records, or after it, which costs $\mathbb{P}(\tau_{\mathrm{bad}} \le T) \le 2\delta$. A single-step tail bound does not give uniformity over the horizon, and with Chebyshev in place of Azuma--Hoeffding the union bound would cost a factor $T$ rather than $\log T$. For the confinement bound, the generator of $\tilde Y_t + C(T-t)$ is $\mathscr{G}\tilde Y - C \le -2\kappa\tilde Y \le 0$ on the tube, and $\tilde Y_t + C(T - t) \ge 0$ for $t \le T$. On the tube the integrand of the stochastic integral, $2\eta c\,(\boldsymbol{x} - \boldsymbol{x}^*)^\top(\Sigma^\tau_\perp)^{1/2}$, is bounded, so the integral is a martingale and the stopped process is a non-negative supermartingale on $[t_0, T]$. Doob's maximal inequality bounds $\mathbb{P}(\sup_{t_0 \le t \le T}\tilde Y_{t \wedge \tau} \ge \epsilon)$ by $(\tilde Y_{t_0} + C(T - t_0))/\epsilon$, and the tracking failure adds $2\delta$.
\end{proof}

The first bound covers the first exit of the annulus. A path that enters the core $\{Y < r_0^2\}$ can return to the annulus and exit through $\epsilon$, each such excursion costs at most $(r_0^2/\epsilon)^\alpha$, and Theorem \ref{thm:adam_trapping} does not bound their number. The second bound covers the whole horizon. Its constant $C$ is the supremum over the tube of $\eta^2\operatorname{tr}(c\,\Sigma^\tau_\perp c) + 2\sqrt{\epsilon}(\|\boldsymbol{f}\| + \rho_\parallel)$, the terms that do not vanish at $\boldsymbol{x}^*$, and $\kappa$ is the inward rate of the terms linear in $\tilde Y$. A nondegenerate diffusion is confined near $\boldsymbol{x}^*$ rather than trapped at it, so the second bound supplies the link survival of Lemma \ref{lemma:doob-app}, Theorem \ref{thm:adam_dsi_main} assumes it, and Remark \ref{remark:adam_diagnostic} measures $C$.

Trajectories that reach the core $\{Y < r_0^2\}$ lie at a distance $O(r_0)$ from $A$, where the forcing $\boldsymbol{f}$ and the residual set the scale. The transverse drift $-\eta c\nabla_\perp\mathcal{L} + \boldsymbol{f}$ vanishes at a point off $A$, and the trap is that point rather than $A$ itself. This is the direct analogue, under Adam and AdamW, of Corollary \ref{cor:prob_non_escape-app}, with the invariant set replaced by a shifted trap and the tube by an annulus. Remark \ref{remark:adam_diagnostic} reports the measurement of these quantities on the grokking task.

\begin{assumption}[Macroscopic Merges]
\label{assum:block_scaling}
The merging components at the observed microtransitions satisfy $i(N) = \Theta(N)$.
\end{assumption}

\begin{theorem}[Conditional Cascade Transfer to Adam and AdamW]
\label{thm:adam_dsi_main}
Under Assumptions \ref{assum:heavy_tail}, \ref{assum:nondegeneracy}, \ref{assum:quasistatic}, \ref{assum:block_scaling}, \ref{assum:condensation_phase} and \ref{assum:kernel}, and the confinement condition $\mathscr{G}\tilde Y \le -2\kappa\tilde Y + C$ of Corollary \ref{cor:adam_nonescape}, on $t \le \tau_\epsilon^{\mathrm{blk}}$ the Adam- or AdamW-trained trajectory admits a percolation process $G_p$ whose $R_v$ peaks form the cascade of Theorem \ref{thm:dsi_proof} with the same factor $\lambda_t$, with escape probability bounded as in Corollary \ref{cor:adam_nonescape}.
\end{theorem}
\begin{proof}
The confinement bound of Corollary \ref{cor:adam_nonescape} supplies the link-survival probability that Lemma \ref{lemma:doob-app} provides in the SGD case. Assumption \ref{assum:condensation_phase} supplies the monotone control parameter of Proposition \ref{thm:reeb_to_percolation}. Proposition \ref{thm:er_divergence} and Corollary \ref{cor:microscopic_washout-app} apply verbatim, and Assumption \ref{assum:block_scaling} places the merges in the macroscopic regime. Assumption \ref{assum:kernel} is a statement about the link rates of the process as run, so the frozen preconditioner $c$, which is a constant of the block's dynamics, never enters the ratio of Theorem \ref{thm:dsi_proof}.
\end{proof}

\begin{remark}[Measurement of the Adam Conditions on the Grokking Task]
\label{remark:adam_diagnostic}
We measured the quantities of Theorem \ref{thm:adam_trapping} on the AdamW-trained modular-arithmetic Transformer of Section \ref{sec:experiments}, the grokking task (Appendix \ref{sec:appendix_empirical_suite}), for two fixed permutation pairs of feed-forward hidden units, $(3,7)$ and $(11,42)$, and for the closest pair among all $512$ units, at every $100$ epochs over $30000$ epochs, which is $301$ checkpoints and $60$ multiples of $\tau_2$ at two steps per epoch. The block of a unit consists of its input weights, its bias and its output weights ($257$ coordinates), and $Y$ is the transverse distance of the pair. At each checkpoint we computed the exact full-batch gradient (the training set has $202$ examples), $\Sigma^\tau_\perp$ from $32$ minibatch gradients, $c$ and $E$ from the bias-corrected $\hat v$, and the deterministic lag $\boldsymbol{r}_t$ as the difference between the $\beta_1$-average of the exact gradient over the $30$ steps preceding the checkpoint and the gradient at the checkpoint, so that $\rho$ is $\eta\|P^\perp[(cI + E)\boldsymbol{r}_t + E\,\boldsymbol{\xi}^m_t]\|$ as Proposition \ref{thm:adam_sgf_reduction} defines it. We recorded $\hat v$ on the block every $10$ steps and flagged an epoch as a loss spike when the training loss exceeded five times the median of the preceding $50$ epochs. Two further realizations with different seeds were measured in the same way, one with the same two pairs and one with $50$ pairs of units drawn at random, together with the distribution of $Y$ over all pairs of live units, defined as units whose norm exceeds half the median norm. Where the realizations differ, the numbers of the other realizations are given in parentheses. The condition, drift and tube statistics are from three further realizations (seeds $1000$ to $1002$), computed with the role-wise preconditioner $c$ of Lemma \ref{lemma:block_homogeneity} and with $\Sigma^\tau_\perp$ and $\nabla\mathcal{L}$ measured with dropout active, as in training.
 
\emph{Preconditioner.} Relative to a single scalar, the deviation $\|E\|_\infty/c$ has median $1.46$ and $1.61$ for the two pairs and reaches $15$ and $35$ at the spikes. Most of it is role-wise: split as in Lemma \ref{lemma:block_homogeneity}, the spread of $c$ across coordinate roles is $0.80$ to $0.91$ of its mean, while the cross-unit part $E$ is $0.75$ to $1.03$ of it and nearly independent of $\sqrt{Y}$ (log-log slope $0.07$ to $0.14$), since slingshots recur before the burn-in of Lemma \ref{lemma:block_homogeneity} completes. $E$ is therefore measured rather than bounded (Assumption \ref{assum:quasistatic}), and the forcing $\boldsymbol{f}$ has the order of the drift at the tube boundary (Proposition \ref{thm:adam_sgf_reduction}). The medians of $\|\boldsymbol{f}\|$ and $\rho$ are $4.1\times10^{-4}$ and $4.8\times10^{-4}$ for $(3,7)$ and $7.8\times10^{-4}$ and $7.6\times10^{-4}$ for $(11,42)$. Replacing $\boldsymbol{r}_t$ by the full lag $\hat m - \nabla\mathcal{L}$ raises $\rho$ by a factor of only $1.1$ to $1.6$, so the lag on this run is deterministic, since minibatches of $128$ out of $202$ examples carry little noise, and Lemma \ref{lemma:momentum_tracking} is the operative bound while the noise term of Lemma \ref{lemma:uniform_tracking} is small. The relative variation of $\hat v$ over windows of length $\tau_2$ has median $1.6$ and $2.4$ on the two blocks, which is a factor of $1.6$ to $1.8$ in $c$ over a typical window, and $90$th percentile $290$ and $600$, which is a factor of $17$ to $25$ in $c$ on the windows that contain a spike. Assumption \ref{assum:quasistatic} therefore holds to within a factor of two inside a condensation phase, each spike opens a new window, and $\hat v$ on the block spans $10^{-9.2}$ to $10^{-3.6}$ over the run. The empirical autocorrelation time of $\hat g^{\circ2}$ is one step, $\sqrt{\hat v}$ on the block exceeds $\epsilon_{\mathrm{opt}}$ by two orders of magnitude, and $\max|g|/\mathrm{std}(g) \approx 77$.
 
\emph{Condition.} Condition \eqref{eq:adam_condition} holds at $50$ to $62\%$ of checkpoints for the pairs $(3,7)$ and $(11,42)$, at $52$ to $58\%$ averaged over the $50$ random pairs and at $62$ to $65\%$ for the closest pair, with phase averages between $0.47$ and $0.62$ from memorization to the end of training. With the noise measured with dropout off, the same seeds give $64$ to $78\%$, and replacing the role-wise $c$ by a scalar changes every fraction by at most $0.02$. The cruder form with $2\sqrt{Y}(\|\boldsymbol{f}\| + \rho_{\max})$ holds at $14$ to $23\%$, because $\boldsymbol{f}$ is nearly orthogonal to $\boldsymbol{x}$ and the Cauchy--Schwarz step discards most of the margin. The fractions are the same for every $\alpha$ between $0.05$ and $1$, since the diffusion terms lie about two and a half orders of magnitude below the drift. Dropout raises the transverse noise $31$- to $46$-fold, and the noise still vanishes at coincidence, scaling as $Y^{0.96}$ to $Y^{1.04}$ for live pairs.
 
\emph{Drift.} $\cos(\boldsymbol{x}, \boldsymbol{D})$ with $\boldsymbol{D} = -\eta c\nabla_\perp\mathcal{L} - \eta\lambda\boldsymbol{x} + \boldsymbol{f}$ is negative at $63$ to $74\%$ of checkpoints for the two pairs and the $50$ random pairs. The inward pull comes from weight decay: evaluated without dropout, the gradient term alone points outward, $\boldsymbol{x}^\top\nabla_\perp\mathcal{L} < 0$, at $89$ to $95\%$ of checkpoints on contracting and expanding segments alike, while $\boldsymbol{D}$ points inward at $96$ to $97\%$ of checkpoints on contracting segments and at $57$ to $65\%$ on expanding ones.
 
\emph{Contraction and restarts.} The training loss shows $44$ spike events over the run, with mean spacing $679$ epochs, which is $1.4\,\tau_2$, and median duration $28$ epochs. These are the slingshot instabilities of Adam with weight decay \citep{thilak2022slingshot}, and each one starts a new condensation phase. Within a phase $Y$ evolves geometrically. Of the $45$ inter-spike segments per pair, $60\%$ contract ($53\%$ averaged over the $50$ random pairs), and the segments that expand are those on which the full drift $\boldsymbol{D}$ points inward less often. Among the contracting segments the median e-fold time of $Y$ is $1226$ steps for $(3,7)$ and $802$ steps for $(11,42)$ ($1390$ steps over the $50$ pairs), against $1/(2\eta\lambda) = 333$ steps for weight decay alone, and the fastest phases reach e-fold times of $400$ steps. The largest single-phase contractions are $1.7$ decades for each pair ($2.5 \to 0.053$ over epochs $16863$ to $17642$ and $5.7 \to 0.11$ over epochs $24519$ to $25293$), and the second realization contains a phase of $3.6$ decades. Over the whole run contraction leads expansion by $2.5$ decades for $(3,7)$ ($13.0$ against $10.5$) and by $1.2$ decades for $(11,42)$ ($17.2$ against $16.0$), a net drift toward coincidence that the restarts distribute across the run. Among the $344$ live units at the end of the $50$-pair realization, $114$ live pairs lie below $Y = 10^{-2}$, the $0.1\%$ quantile of $Y$ over live pairs is $7\times10^{-3}$, and random live pairs reach $Y \approx 10^{-3}$ within a single phase (pair $(54,325)$ at epochs $5163$ and $11930$, with unit norms $0.4$ and $2.8$). The closest pair among all units at every checkpoint consists of two units with norms $2\times10^{-3}$ to $5\times10^{-2}$ against a median unit norm of $0.3$ to $0.8$, so the closest-pair statistic tracks units that weight decay has driven toward the origin, and the live-pair statistics above are the ones that measure coincidence.
 
\emph{Tube occupancy.} Over the three further realizations, the sampled pairs enter the tube $\{Y < 10^{-2}\}$ progressively: $0$ to $12\%$ by epoch $2000$, $25$ to $77\%$ by the second variance peak and $62$ to $100\%$ by the end of training, fastest between the first two peaks, with a median stay of $80$ to $605$ epochs per entry. Most of these entries are units that weight decay drives together toward the origin; among live units the $\epsilon$-graph reaches a merge fraction of $0.17$ to $0.24$, with the largest cluster holding $4$ to $17\%$ of them. Condition \eqref{eq:adam_condition} holds more often the closer the pair: at $0.62$ to $0.66$ of checkpoints inside the tube and $0.49$ to $0.50$ for $Y \ge 1$, with the drift inward at $0.71$ to $0.76$ and $0.61$ to $0.62$. Within $100$ epochs after a spike the condition holds at $0.47$ to $0.56$ of checkpoints, against $0.52$ to $0.59$ in quiet windows.
 
\emph{Confinement constant.} The constant $C$ of Corollary \ref{cor:adam_nonescape} is the sum of the diffusion term and the forcing-plus-residual term at the tube boundary, and at typical checkpoints the latter dominates, $2\sqrt{\epsilon}(\|\boldsymbol{f}\| + \rho) \approx 2\sqrt{\epsilon}\times10^{-3}$, with the diffusion term four orders of magnitude smaller. The bound certifies confinement over horizons $T \lesssim \epsilon/C$, which at this value of $C$ is a few hundred steps for $\epsilon$ of order $10^{-2}$. The reinflation of $Y$ at each spike occurs beyond this horizon.
 
\emph{Reading.} Condition \eqref{eq:adam_condition} holds at most checkpoints inside phases of length about $1.4\,\tau_2$, at $0.3$ to $0.4$ of the weight-decay contraction rate in the median and at $0.8$ in the fastest phases, and the slingshot instabilities restart the condensation phase on that timescale. We treat each inter-spike window as a condensation phase in the sense of Assumption \ref{assum:condensation_phase}, with the spikes as its boundaries. The repulsion phases, on which $\boldsymbol{D}$ points outward, are the complement of the annulus in time. The order parameter of the grokking task is the effective rank of the embedding, and the feed-forward block measured here is the composition through which the network reads the embedding, so its coincidence sets are the ones that nest above the embedding's.
\end{remark}

\section{Explicit Theoretical Verification and Toy Models}
\label{sec:appendix_toy_experiments}
 
We construct two distinct experimental environments. First, we impose the topological mechanics in a controlled kinematic setting and check that the detection pipeline recovers the imposed cascade. Second, we deploy an empirical neural network optimized via Stochastic Gradient Descent (SGD) to verify that these topological phase transitions—both condensation and reactive fragmentation—govern non-convex empirical loss landscapes.
 
\begin{figure*}[htbp]
    \centering
    \includegraphics[width=\textwidth]{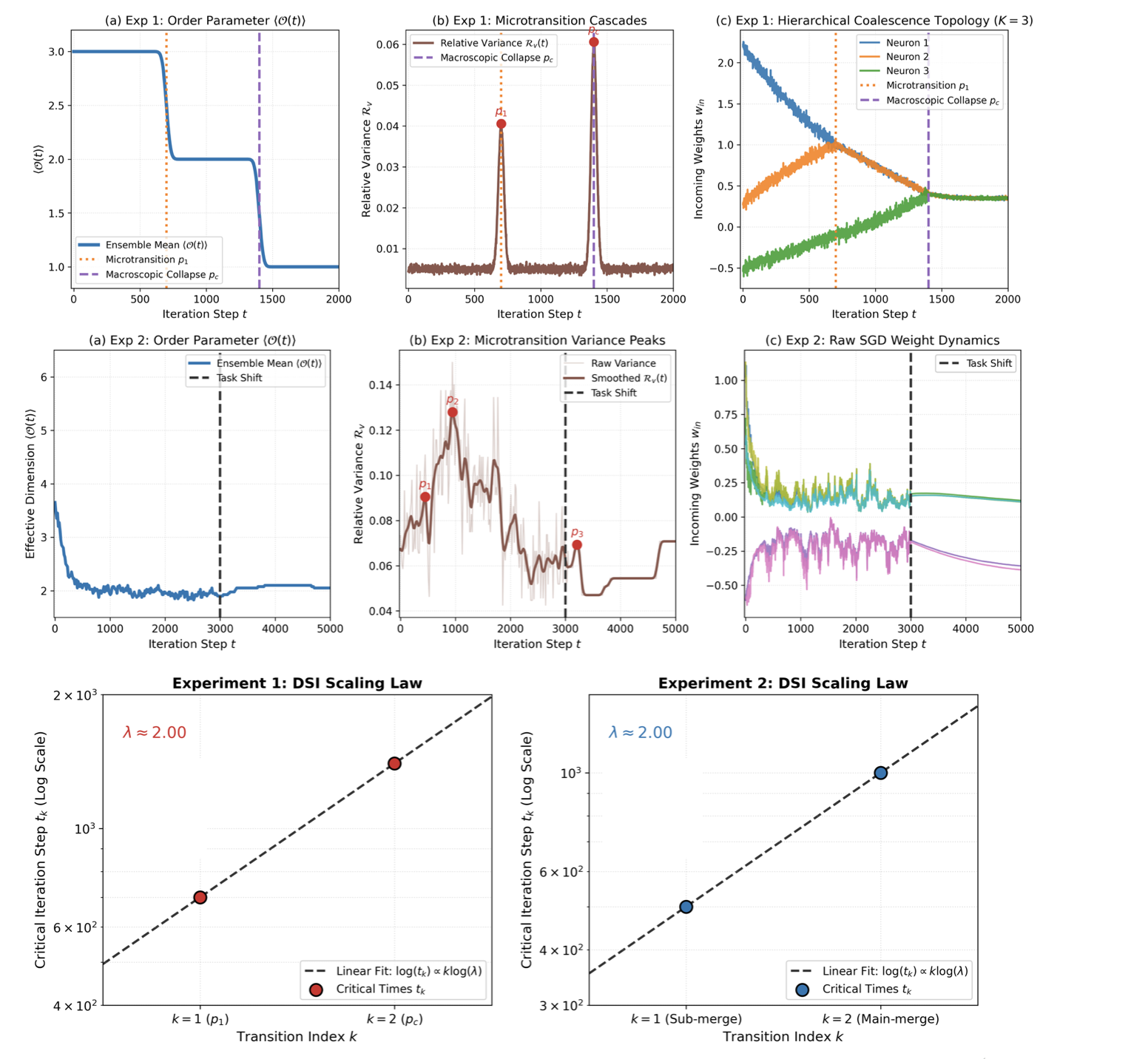}
    \caption{Experimental validation of symmetry-induced topological collapse and the scaling of peak times. \textbf{Top Row (Experiment 1):} Kinematic simulation of hierarchical coalescence ($N=3$), displaying discrete jumps in the order parameter (a), peaks in relative variance exactly at the imposed merge points (b), and forced continuous spline trajectories (c). \textbf{Middle Row (Experiment 2):} Empirical SGD dynamics ($N=6$), showing macroscopic dimensionality collapse (a), distinct precursory variance peaks $p_1, p_2$ during the condensation phase, and a reactive fragmentation peak $p_3$ following a task shift at $t=3000$ (b), alongside the raw stochastic weight trajectories (c). \textbf{Bottom Row:} Log-linear progressions of peak times: Experiment 1 (Left) recovers the imposed ratio $\lambda_t = 2.00$, and Experiment 2 (Right) gives the two-peak ratio $\lambda_t \approx 2.00$; the finite-$N$ prediction for the first two merges of $N=6$ units is $2.5$ (Section \ref{sec:appendix_toy_experiments}). }
    \label{fig:toy_experiments}
\end{figure*}
 
\subsection{Generalization of the Stochastic Collapse Threshold}
 
To derive the exact mathematical condition under which a subnetwork irreversibly collapses into an invariant set, we analyze the local dynamics of continuous stochastic gradient flow (SGF) \citep{li2017stochastic}. 
 
Consider a network with multiplicative parameter interactions, such as adjacent layers connected by weights $w_{in}$ and $w_{out}$. Assuming the subnetwork learns a target signal $\mu$ under a Mean Squared Error (MSE) objective, the local loss is $\mathcal{L}(w_{in}, w_{out}) = \frac{1}{4}(w_{in}w_{out} - \mu)^2$. Under continuous gradient descent, the difference of squared weights is a conserved quantity: $\frac{d}{dt}(w_{out}^2 - w_{in}^2) = 0$. Assuming a balanced initialization, the parameters satisfy $w_{in} = w_{out} = w$. The transverse dynamics consequently collapse onto a one-dimensional effective loss landscape:
\begin{equation}
    \mathcal{L}(w) = \frac{1}{4}(w^2 - \mu)^2.
\end{equation}
 
Near the structurally collapsed state at the origin ($w=0$), the gradient is $\mathcal{L}'(w) = w^3 - \mu w$, establishing $w=0$ as a stationary invariant manifold where the first derivative vanishes. The local curvature is given by the second derivative $\mathcal{L}''(0) = -\mu$, defining a deterministic repulsive drift away from the manifold. 
 
Simultaneously, the gradient noise covariance matrix $D(w)$ inherently scales with the magnitude of the active weights \citep{blanc2020implicit}. Near the manifold, the transverse diffusion term scales quadratically, $D(w) \approx \frac{1}{2}\zeta^2 w^2$, yielding purely multiplicative noise. Substituting the exact gradient and the local noise approximation into the continuous SDE $dw_t = -\nabla_{w}\mathcal{L}(w_t)dt + \sqrt{2D(w_t)}dW_t$ yields:
\begin{equation}
    dw_t = -(w_t^3 - \mu w_t)dt + \zeta w_t dW_t.
\end{equation}
 
To determine whether the network undergoes collapse, we evaluate the Fokker-Planck equation for localized stationarity ($\partial_t p = 0$). Forcing the probability current to vanish yields the stationary density \citep{chen2023stochastic}
\begin{equation}
    p_{ss}(w) \propto w^{2\mu/\zeta^2 - 2}\, e^{-w^2/\zeta^2}.
\end{equation}
The global normalization partition function diverges when the exponent of $w$ forces a non-integrable singularity at the origin, requiring $\frac{2\mu}{\zeta^2} - 2 \le -1$.
 
\textbf{Theoretical Threshold:} The invariant manifold becomes strictly stochastically attractive—collapsing the transverse probability mass into a Dirac delta distribution $\delta(w)$—if and only if the multiplicative noise variance overcomes the deterministic drift:
\begin{equation}
    \mu \le \frac{\zeta^2}{2}.
\end{equation}
Because the empirical SGD diffusion amplitude scales proportionally with the learning rate $\eta$, maintaining a sufficiently high learning rate explicitly satisfies this threshold, permanently trapping the parameters within the lower-dimensional invariant subspace \citep{chen2023stochastic}.
 
Definition \ref{def:stoch-attac} reads this threshold off directly. With $\boldsymbol{x} = w$, $\nabla\mathcal{L} = w^3 - \mu w$ and $\eta\Sigma_\perp = \zeta^2 w^2$, condition \eqref{eq:attractivity} becomes $2\mu w^2 - 2w^4 + \zeta^2 w^2 \le 2(1-\alpha)\zeta^2 w^2$, that is $\mu \le \tfrac{\zeta^2}{2}(1 - 2\alpha) + w^2$, and $\alpha \downarrow 0$ recovers $\mu \le \zeta^2/2$ near the origin. At $\alpha = 1$ the condition reads $2\mu + \zeta^2 \le 0$ and fails for every $\mu > 0$, because the generator of $w^2$ is $(2\mu + \zeta^2)w^2 - 2w^4 > 0$ near the origin even as $w \to 0$ almost surely, and for this reason the definition is stated for $Y^\alpha$ rather than for $Y$. Above the threshold, $Y = w^2$ is Gamma-distributed with shape $\mu/\zeta^2 - \tfrac12$ and scale $\zeta^2$ (Remark \ref{remark:fp_heuristic}).
 
\subsection{Derivation of the Scaling of Peak Times}
 
Once parameters are trapped in the invariant manifold, Definition \ref{def:merge_multiplicity} and Proposition \ref{thm:hypercondensation} make condensation pairwise, $n = 2$, in the generic case. Lemma \ref{lem:dsi_selfsim} then gives $t_{2i}/t_i = \lambda_t = 2^{1-\gamma_K}$ for the $R_v$ peak times (Theorem \ref{thm:dsi_proof}, Corollary \ref{cor:peak_times}). For the constant kernel the mean component size grows linearly in time, $\bar C = 1 + t/2$, so $T_{2^k} = 2(2^k - 1)$ and $\lambda_t = 2$ asymptotically.
 
\textbf{Prediction under the constant kernel:} For pairwise merges the sequential peak times $t_k$ detected by the macroscopic relative variance $\mathcal{R}_v(t)$ satisfy the log-linear progression $t_{k+1} = 2 t_k$ asymptotically, with exact finite-size ratios $(2i-1)/(i-1)$ for the times at which the largest component reaches sizes $i$ and $2i$. This follows from Assumption \ref{assum:kernel} with $\gamma_K = 0$ through Lemma \ref{lem:dsi_selfsim}. For $N$ units the expected waiting time between the merges that take the component count from $M$ to $M-1$ is $2N/(M(M-1))$ in the same units, so for $N = 6$ the expected times of the first five merges are $0.4, 1.0, 2.0, 4.0, 10.0$ and the first two merges are in ratio $2.5$. A measured $\lambda_t \ne 2$ at large $N$ reflects $\gamma_K \ne 0$, $n \ne 2$, or both, and the ratio alone does not separate them.
 
\subsection{Idealized Hierarchical Coalescence}
 
\textbf{Purpose:} In non-convex empirical loss landscapes, topological phase transitions are tightly coupled with geometric distortions, complicating the isolation of pure structural scaling. We impose the topological variable in a controlled environment and check that the detection pipeline recovers it.
 
\textbf{Construction:} We construct a synthetic kinematic system of $N=3$ decoupled parameters. We impose sequential, binary hierarchical coalescence ($n=2$) by routing continuous spline trajectories injected with decaying Gaussian noise, with merge times fixed at $t_1 = 700$ and $t_c = 1400$. This tests whether the $\mathcal{R}_v$ pipeline locates imposed merges and returns the imposed ratio of their times.
 
\textbf{Results:} The binding of parameters 1 and 2 at $t_1 = 700$, followed by parameter 3 at $t_c = 1400$ (mapped in Figure \ref{fig:toy_experiments}, Top Row, Panel c), causes the macroscopic order parameter $\langle \mathcal{O}(t) \rangle$ (Figure \ref{fig:toy_experiments}, Top Row, Panel a) to collapse via discrete jumps. Driven by the two-point ensemble states at these thresholds, the relative variance $\mathcal{R}_v(t)$ (Figure \ref{fig:toy_experiments}, Top Row, Panel b) registers localized peaks exactly at the imposed merge points. Plotting these transition times on a logarithmic scale against their sequential index $k$ (Figure \ref{fig:toy_experiments}, Bottom Left) returns $\lambda_t = 1400/700 = \boldsymbol{2.00}$, the imposed ratio. This checks the pipeline, and the theoretical content of the prediction is tested by the SGD simulation below.
 
\subsection{Empirical Simulation (\texorpdfstring{$N=6$}{N=6}) with Task Shift and Reverse Transition}
 
\textbf{Theoretical Justification:} Having checked the pipeline, we show that the noise characteristics generated by standard neural network optimization satisfy these topological constraints without external kinematic forcing. Furthermore, the theoretical threshold ($\mu \le \zeta^2/2$) dictates that the collapse is dynamically reversible: dropping the noise variance while increasing the signal curvature should render the invariant manifold repulsive, forcing reactive fragmentation, which ends the condensation phase of Assumption \ref{assum:condensation_phase}.
 
\textbf{Empirical Verification:} We deploy a PyTorch environment, initializing a $N=6$ shallow neural network with a GELU activation function, optimized on an MSE objective by label-noise SGD \citep{blanc2020implicit}: each step uses the full batch of $32$ inputs $x \sim \mathcal{N}(0,1)$ with fresh Gaussian label noise of standard deviation $2.0$. The network computes $\sum_{k=1}^{6} w_{out,k}\,\mathrm{gelu}(w_{in,k}x)$ for the target $0.5\sin 2x$, with learning rate $0.2$, weight decay $10^{-3}$, gradient-norm clipping at $5$ and initialization standard deviation $0.8$, and the order parameter counts average-linkage clusters of the $w_{in,k}$ at distance $0.25$. We manipulate the theoretical stochastic collapse condition by separating the training into two distinct phases. In the condensation phase ($t < 3000$), a high learning rate injects multiplicative noise to satisfy $\mu \le \zeta^2/2$. At $t=3000$, we execute a task shift: we inject a high-frequency target signal (increasing the deterministic curvature $\mu$) while simultaneously dropping the learning rate (decreasing the noise amplitude $\zeta^2$): the target gains $0.8\cos 5x$, the learning rate falls to $0.01$ and the label-noise standard deviation to $0.1$. This tests both forward condensation and reverse reactive fragmentation.
 
\textbf{Results:} During the initial high learning rate phase, the aggressive stochastic diffusion neutralizes the deterministic gradient drift (Figure \ref{fig:toy_experiments}, Middle Row, Panel c). This forces the independent parameters to collapse into shared, lower-dimensional invariant subspaces. The SGD noise shifts the merge thresholds across ensemble realizations, creating a distinct multi-phase precursory cascade. We detect sequential peaks in the smoothed relative variance $\mathcal{R}_v(t)$ (Figure \ref{fig:toy_experiments}, Middle Row, Panel b) strictly within the high learning rate regime. Extracting these transition indices and computing their scaling progression (Figure \ref{fig:toy_experiments}, Bottom Right) yields the two-peak ratio $\boldsymbol{\lambda_t \approx 2.00}$. Two facts qualify this number. With a tube radius of $\epsilon = 0.25$ on weights initialized with standard deviation $0.8$, about a third of the pairs lie inside the tube at initialization, so the component count starts near $4$ and the first detected peak partly records links that were present at initialization. And for six units the constant kernel predicts a ratio of $2.5$ for the first two merges, with $2$ the asymptotic value at large $N$ (Section \ref{sec:appendix_toy_experiments}, prediction paragraph), so the two peaks of this run measure the asymptotic ratio only approximately. The large-$N$ test of the prediction is the cascade over many levels, which the constant kernel fixes at $(2i-1)/(i-1)$ level by level.
 
Following the task shift at $t=3000$, the drastic reduction in learning rate and the addition of the new target signal immediately violate the trapping threshold. The invariant manifold becomes stochastically repulsive. The macroscopic topology registers a reverse transition: the relative variance $\mathcal{R}_v(t)$ peaks as the previously bound parameters fragment to map the new target dimensionality.

\textbf{Separating the two interventions.} Past the threshold the escape rate from the manifold scales as $\eta(\mu - \zeta^2/2)$, so the curvature and the noise act on different clocks. We repeat the protocol over $40$ realizations with common random numbers in four arms (Table \ref{tab:task_shift}). Raising the curvature at the original learning rate fragments $30\%$ of seeds and raises $\mathcal{R}_v$ $1.61$-fold within a few hundred steps. Lowering the noise through the learning rate fragments seeds in the same direction on a $20$-fold slower clock, and the combined protocol, which runs on that clock, is still building its response at the end of the $2000$-step window.

\begin{table}[htbp]
    \centering
    \caption{Task-shift ablation ($N=6$, $40$ realizations, common random numbers; interventions at $t=3000$). Fragmenting seeds end with at least half a cluster more, on average over the last $300$ steps, than just before the shift.}
    \label{tab:task_shift}
    \resizebox{\linewidth}{!}{%
    \begin{tabular}{lcccc}
    \toprule
    Arm & Clusters before $\to$ after & Fragmenting seeds & Post/pre $\max\mathcal{R}_v$ & Maximum at \\
    \midrule
    Target and noise (Fig.~\ref{fig:toy_experiments}) & $1.93 \to 2.00$ & $0.12$ & $1.19$ & $t = 4995$ \\
    Noise only & $1.93 \to 2.08$ & $0.10$ & $0.87$ & $t = 3045$ \\
    Target only & $1.93 \to 2.11$ & $0.30$ & $1.61$ & $t = 4155$ \\
    Control & $1.93 \to 1.95$ & $0.03$ & $1.05$ & $t = 4155$ \\
    \bottomrule
    \end{tabular}}
\end{table}

\subsection{Block Merges Require Nested Symmetry}
\label{sec:appendix_block_merges}
 
Proposition~\ref{thm:main_er_discontinuity} separates two ways in which the largest component can grow, namely block merges, in which two already-macroscopic components coincide at once, and single attachment, in which the largest component absorbs one subnetwork at a time. The cascade of Lemma~\ref{lem:dsi_selfsim} is a property of the hitting laws and holds under either mode. What the two modes decide is the sharpness of the individual peaks, since the two-point mixture of Theorem~\ref{thm:main_rv_divergence} needs a macroscopic jump, which block merges produce and single attachment, the Erd\H{o}s--R\'enyi regime, does not. We test which of the two SGD produces, and what the architecture has to supply for block merges to occur, with two networks built from the same units, the same target, and the same optimizer.
 
\textbf{Construction.} Both networks are built from units $a\,\mathrm{gelu}(bx+c)$ with scalar parameters, trained on the target $\mu\,\mathrm{gelu}(x)$ with $\mu=1$ on $256$ standard normal inputs, by SGD with learning rate $\eta=0.1$, weight decay $\lambda = 10^{-2}$, batch size $8$, for $6000$ SGD steps and $32$ seeds. The weight decay supplies the inward drift of Definition~\ref{def:stoch-attac} along the directions in which the gradient is conserved between coincident units. The \emph{flat} layer sums $N$ such units and carries only the pairwise permutation symmetry of Proposition~\ref{thm:hypercondensation}. The \emph{binary tree} of depth $d$ has $N=2^d$ leaf units, every internal node computes $\mathrm{gelu}(u_L h_L + u_R h_R + c)$ from its two children, and the root is linear. Swapping any two sibling subtrees, together with all of their parameters, is an exact symmetry at every level of the tree, so the coincidence set of two siblings at level $k$ is an affine subspace of the form of Theorem~\ref{thm:affine_generation}, and a merge at level $k$ joins two identical blocks of $2^{k-1}$ leaves. Two leaves are linked when the parameters along their paths to the root, $\phi_i$ and $\phi_j$, satisfy $\tfrac12\|\phi_i-\phi_j\|_2^2<\epsilon$ with $\epsilon=10^{-2}$, so that a link between leaves in different subtrees requires the subtrees and their ancestors to coincide, as a block merge requires. Both graphs start with at most one or two links out of $\binom{N}{2}$, with initial mean cluster counts $15.9$, $15.5$, $31.9$ and $30.3$ for tree and flat at $N=16$ and $N=32$. We record the largest component of the $\epsilon$-graph at every step and report the growth of its running maximum per seed, so that each level is counted once and crossings of the tube boundary are not counted as repeated merges.
 
\textbf{Results.} Figure~\ref{fig:block_merges} and Table~\ref{tab:block_merges} summarize the two architectures at $N=16$ and $N=32$. Only the tree reaches full collapse, in $47\%$ of seeds at $N=32$ and $19\%$ at $N=16$, against $0\%$ for the flat layer at either size, whose cluster count plateaus near $2.3$. The tree produces merges that the flat layer does not. At $N=32$ its record growth contains three jumps of $+16$, the coincidence of two half-trees, and twelve of $+8$, the coincidence of two quarter-trees. The flat layer at $N=32$ shows a handful of jumps between $+8$ and $+12$ (nine events in $498$), which are attachments of a group of units that have already collapsed to a common point rather than merges of symmetric blocks, and none reaches $+16$. The share of record-growth events that exactly double the largest component, the signature of a merge of two equal blocks, is $0.22$ and $0.11$ for the tree at $N=16$ and $N=32$, against $0.11$ and $0.03$ for the flat layer. In both architectures the majority of record-growth events are single attachments ($0.60$ to $0.74$), because a subtree need not finish collapsing internally before its ancestors align with those of its sibling, so block merges occur in the tree but do not dominate its growth. Nested symmetry forces every merge across subtrees to be a block merge, since leaves in different subtrees link only when their shared ancestors coincide; within a subtree, leaves can still attach one at a time. Nested symmetry therefore makes macroscopic block merges possible, and by Proposition~\ref{thm:main_er_discontinuity} the jump survives only when the merging components are already macroscopic, which the architecture permits and the dynamics realize in part. At the first level the ensemble is two-point (all seeds at $|C_{\max}| \in \{1, 2\}$ at $N = 32$), and the $\mathcal{R}_v$ peak of $|C_{\max}|/N$ has height $0.125$ ($95\%$ bootstrap interval $[0.124, 0.169]$) at merged weight $w = 0.34$, at the $w^*$-quantile of the hitting law, as Theorem~\ref{thm:main_rv_divergence} predicts. At higher levels the ensemble occupies more than two states, reflecting the single attachments, and the peaks are correspondingly taller ($0.21$ to $0.37$).
 
\textbf{Clock.} Panel (c) reports the per-seed hitting times $\mathbb{E}[T_{2^k}]$ of $|C_{\max}|\ge 2^k$ by level. For the tree at $N=32$ they are $748$, $1659$, $1922$, $2200$ and $2701$ steps, with consecutive ratios $2.2$, $1.16$, $1.14$ and $1.23$, against $3$, $2.33$, $2.14$ and $2.07$ for a constant kernel under a time-homogeneous clock (Lemma~\ref{lem:dsi_selfsim}). The first level is slow because the leaves must converge from a random initialization, and every higher level is paced by the same quantity, the contraction $e^{-\eta\lambda t}$ of the conserved sibling differences under weight decay, so that successive levels close at a nearly constant spacing $\Delta t \approx \ln(\Delta_0/\sqrt{2\epsilon})/(\eta\lambda)$ rather than at geometrically growing intervals. The level structure is the pairwise hierarchy of Definition~\ref{def:merge_multiplicity}, and in the balanced hierarchy the cascade is geometric with factor $2$ in the density variable, since $1-p$ halves at each level. The clock, however, is exponential rather than the power law of a time-homogeneous kernel, so the ratios in $t$ reflect the clock rather than the multiplicity. A kernel exponent fitted to these ratios through $\gamma_K = 1 - \log_2 \lambda_t$ would read $0.74$, attributing to component size a variation that comes from training time. This is the time-inhomogeneity that Remark~\ref{remark:empirical_variance_mass} names, realized in a setting where the level structure and the clock can both be read off, and it places the weight-decay-paced toy outside the time-homogeneous hypothesis of Theorem~\ref{thm:main_dsi} while confirming its level structure.

\begin{table}[htbp]
    \centering
    \caption{Growth of the record largest component under SGD for the two architectures ($32$ seeds, $6000$ SGD steps). ``Exact doubling'' is the share of record-growth events with $|C_{\max}|_{\mathrm{after}} = 2\,|C_{\max}|_{\mathrm{before}}$, and ``macroscopic'' is the share with jump $\ge N/4$. Hitting ratios omit the first level for the flat layer at $N=32$, where one pair is inside the tube at initialization ($T_2=18$).}
    \label{tab:block_merges}
    \resizebox{\linewidth}{!}{%
    \begin{tabular}{lcccc}
    \toprule
    & \textbf{Tree $N=16$} & \textbf{Flat $N=16$} & \textbf{Tree $N=32$} & \textbf{Flat $N=32$} \\
    \midrule
    Full collapse (share of seeds) & $19\%$ & $0\%$ & $47\%$ & $0\%$ \\
    Final mean cluster count & $3.6$ & $2.3$ & $1.7$ & $2.3$ \\
    Single-attachment share & $0.74$ & $0.68$ & $0.65$ & $0.60$ \\
    Exact-doubling share & $0.22$ & $0.11$ & $0.11$ & $0.03$ \\
    Macroscopic jumps ($\ge N/4$) share & $0.09$ & $0.05$ & $0.05$ & $0.02$ \\
    Largest jump observed & $+4$ & $+5$ & $+16$ & $+12$ \\
    Hitting ratios $\mathbb{E}[T_{2i}]/\mathbb{E}[T_i]$ & $1.85,\,1.22$ & $3.8,\,1.5$ & $2.2,\,1.16,\,1.14,\,1.23$ & $1.7,\,1.2$ \\
    \bottomrule
    \end{tabular}}
\end{table}

\section{Extended Empirical Validation Across Diverse Datasets and Architectures}
\label{sec:appendix_empirical_suite}
 
Having recovered the pairwise cascade in controlled kinematic and low-dimensional SGD environments, we now generalize the framework to standard deep learning paradigms. We deploy unconstrained neural networks across a diverse suite of datasets, spanning tabular data, geometric manifolds, vision tasks, and algorithmic grokking.
 
\subsection{Methodological Extensions for Empirical Regimes}
 
Transitioning from isolated theoretical toy models to highly non-convex, unconstrained deep learning environments requires adapting our measurement observables. All empirical experiments were run on a Tesla T4 GPU and require about one day of compute. Architectures and optimization settings are listed in Table \ref{tab:hyperparams}.

\begin{table}[htbp]
    \centering
    \caption{Settings of the empirical suite. Categorical features are one-hot encoded (first level dropped) and tabular features standardized; models are trained on the full dataset, since the observable is the training dynamics. UCI runs use $20$ initializations and the other tasks $15$. UCI peak detection uses macro $\sigma = 20$, local $\sigma = 1.5$ and minimum peak distance $8$; the other tasks use macro $\sigma = E/5$, local $\sigma = E/50 + 1$ and distance $E/20$ for $E$ epochs; the prominence is $0.05$ throughout.}
    \label{tab:hyperparams}
    \small
    \renewcommand{\arraystretch}{1.15}
    \begin{tabularx}{\linewidth}{@{}>{\raggedright\arraybackslash}p{2.0cm}>{\raggedright\arraybackslash}X>{\raggedright\arraybackslash}p{1.9cm}>{\raggedright\arraybackslash}X>{\raggedright\arraybackslash}X@{}}
    \toprule
    Task & Model & Monitored matrix & Optimization & Data \\
    \midrule
    Heart Disease, German Credit, Abalone, Digits & $d_{\mathrm{in}} \to 64 \to 32 \to C$, GELU & first layer ($64 \times d_{\mathrm{in}}$) & SGD, lr $0.05$, batch $32$, $100$ epochs & OpenML \texttt{heart-statlog}, \texttt{credit-g}, \texttt{abalone} (rings $\le 8$, $9$--$10$, $\ge 11$); scikit-learn digits \\
    \addlinespace
    Moons, Swiss Roll & $d_{\mathrm{in}} \to 64 \to 32 \to 2$, GELU & first layer ($64 \times d_{\mathrm{in}}$) & SGD, lr $0.1$, batch $64$, $100$ epochs & $2000$ samples, noise $0.1$ \\
    \addlinespace
    MNIST, FashionMNIST & $784 \to 128 \to 64 \to 10$, GELU & first layer ($128 \times 784$) & SGD, lr $0.05$, batch $128$, $50$ epochs & $5000$-image subset \\
    \addlinespace
    Modular arithmetic & 1 layer, $d = 128$, 4 heads, FF $512$, GELU, dropout $0.1$, no LayerNorm & token embedding ($18 \times 128$) & AdamW, lr $3 \times 10^{-3}$, $\lambda = 0.5$, batch $128$, $30000$ epochs & $(a + b) \bmod 17$, input $[a, b, =]$, $70/30$ split ($202/87$) \\
    \bottomrule
    \end{tabularx}
\end{table} 
 
\textbf{Soft Dimensionality Collapse (Spectral Effective Rank):} In high-dimensional empirical networks, parameter constraints rarely force weights to absolute zero (hard topological collapse). Instead, structural merges manifest as spectral concentration, where the energy of the weight matrix collapses into a smaller subset of principal components. We map the macroscopic order parameter $\langle \mathcal{O}(t) \rangle$ to the Effective Rank of the target layer's weight matrix. Computing the Singular Value Decomposition (SVD), we extract the singular values $\sigma_k$, normalize them into a probability distribution $p_k = \sigma_k / \sum \sigma_i$, and compute the exponential of the Shannon entropy: 
\begin{equation}
    \mathcal{O}(t) = \exp \left( -\sum p_k \ln p_k \right).
\end{equation}
This continuous metric reliably tracks the macroscopic dimensionality collapse as the top singular values diverge and absorb the network's capacity.
 
\textbf{Signal Processing, Detrended Log-Variance, and Ablation Controls:} Real-world SGD exhibits massive baseline heteroskedasticity; the raw variance is known to decay as the network settles into a basin \citep{mandt2017stochastic,mori2022logarithmic}. To isolate the microscopic variance peaks ($\mathcal{R}_v(t)$) from this global decay, we operate strictly in log-variance space, $\log(\text{Var}[\mathcal{O}(t)])$. We subsequently subtract a global linear secant to detrend the signal, preventing artificial boundary artifacts during macroscopic smoothing. This isolates the structural fluctuations from the standard optimization noise. To establish a baseline, we construct a spectral null model testing the null hypothesis that observed cascades are artifacts of this optimization noise. We compute $1000$ phase-randomized spectral surrogates which preserve the empirical power spectrum of the variance trajectory while destroying localized temporal structures to calculate empirical False Positive Rates (FPR), the fraction of surrogates whose peak sequence fits a log-linear law at least as well as the data; with $1000$ surrogates the smallest attainable value is $1/1001$. Additionally, we compute $500$ bootstrap resamples per configuration to ensure statistical stability, aggregating variance across $20$ independent initializations for the UCI tabular suite and $15$ independent initializations for high-dimensional vision and algorithmic grokking datasets. A log-linear fit through $k$ peaks has $k-2$ residual degrees of freedom, so for three peaks the fit has one and the FPR is the primary evidence for the cascade. We report $R^2$ for the longer cascades only.

\textbf{Graded evidence:} We report the FPR as a continuous measure of evidence rather than applying a fixed significance level, following the ASA statements on $p$-values \citep{wasserstein2016asa, wasserstein2019moving}. A $p$-value does not measure the size or importance of an effect, and conclusions should not rest on whether it crosses a specific threshold. Heart Disease ($4.9\%$) and FashionMNIST ($10.7\%$) therefore differ by a factor of two in how often spectrally matched noise reproduces the fit, not by a qualitative change in evidence, and the negative control German Credit ($80.2\%$) marks the other end of the scale.
 
\textbf{Spectral Amplitude of Block Merges:} In the idealized $N=3$ kinematic model (Appendix \ref{sec:appendix_toy_experiments}), exact counting gives an integer jump in $|C_{\max}|$. Mapping unconstrained networks through the continuous Effective Rank changes the amplitude of each jump and leaves its location alone: the cascade factor $\lambda$ is a property of the hitting laws (Theorem \ref{thm:dsi_proof}), the measured ratio $\lambda_t$ a property of the time map (Corollary \ref{cor:peak_times}), and neither depends on the observable used to detect the peaks, since two monotone observables of the same process share every transition time. What the observable sets is the size of the step. The effective rank therefore serves as a detector of transition times; the peak height of Theorem \ref{thm:rv_divergence} is a property of $|C_{\max}|/N$, and Appendix \ref{sec:appendix_block_merges} tests it there.
 
\begin{proposition}[Spectral Step of a Coincidence Merge]
\label{thm:spectral_relaxation}
Let the order parameter be the Spectral Effective Rank $\mathcal{O} = \exp(H)$ with $H = -\sum_k p_k\ln p_k$ and $p_k = \sigma_k/\sum_j\sigma_j$. Let the weight matrix have $N$ mutually orthogonal columns of equal norm, and let $n$ of them coincide. Before the merge $\mathcal{O}_{pre} = N$. After it the singular values are $N-n$ of the original size and one of $\sqrt{n}$ times that size, so with $S := N - n + \sqrt{n}$
\begin{equation*}
    \mathcal{O}_{post} = \exp\!\Big(\frac{(N-n)\ln S + \sqrt{n}\,\ln(S/\sqrt{n})}{S}\Big), \qquad \frac{\mathcal{O}_{pre}}{\mathcal{O}_{post}} \ne n^{n/N}.
\end{equation*}
With the squared normalization $p_k = \sigma_k^2/\sum_j\sigma_j^2$ the spectral mass is conserved through the merge and the step is exactly $\mathcal{O}_{pre}/\mathcal{O}_{post} = n^{n/N}$.
\end{proposition}
 
\begin{proof}
Orthogonal columns of equal norm $\|c\|$ give $N$ singular values equal to $\|c\|$ and $p_k = 1/N$, so $\mathcal{O}_{pre} = N$. When $n$ columns become identical, their span has dimension one and the single nonzero singular value on it is $\sqrt{n}\|c\|$, while the remaining $N-n$ singular values are unchanged. The $\sigma$-mass $\sum_k\sigma_k$ falls from $N\|c\|$ to $S\|c\|$, so it is not conserved, and evaluating $H$ at the post-merge spectrum gives the displayed formula. Under the squared normalization the masses are $\sigma_k^2$, which sum to $N\|c\|^2$ before and after, the merged component carries $p = n/N$ and each other component $1/N$, and $H_{pre} - H_{post} = (n/N)\ln n$, which is the stated step. For $N = 16$ and $n = 2$ the two normalizations give steps of $1.072$ and $1.091$. Discrete counting registers a component reduction by the factor $n$, the spectral observable registers the step above, and the transition times are those of the underlying process and do not move, so $n$ is read from the height of the effective-rank step given $N$, and the peak spacing is unaffected by the choice of observable.
\end{proof}
 
\begin{remark}[Reading the Measured Ratios]
\label{remark:empirical_variance_mass}
Corollary \ref{cor:peak_times} gives $\lambda_t = n^{1-\gamma_K}$, so the measured ratio constrains $n$ and the kernel exponent $\gamma_K$ jointly, and $\gamma_K = 1 - \log_n\lambda_t$ once $n$ is fixed; the fitted $\gamma_K$ is therefore descriptive rather than a test of the kernel. It treats training time as the coalescence clock, so a link rate that varies with training time enters it alongside the size dependence of the kernel. For $\lambda_t$ near $1$ the corrections $\lambda_t^{-k}$ of Lemma \ref{lem:dsi_selfsim} are large at the levels observed, about $0.3$ at $k = 5$ for $\lambda_t = 1.28$, so those ratios carry finite-level corrections of that size.
\begin{itemize}
    \item \textbf{Modular Arithmetic (Transformer Grokking):} $\lambda_t = 2.11$ is pairwise ($n=2$) with $\gamma_K = -0.08$, a constant kernel within the resolution of three peaks, or three-body ($n=3$) with $\gamma_K = 0.32$. The ratio alone does not separate the two.
    
    \item \textbf{UCI Abalone:} $\lambda_t = 1.28$, which for pairwise merges gives $\gamma_K = 0.64$.
 
    \item \textbf{UCI Heart Disease:} $\lambda_t = 1.71$, which for pairwise merges gives $\gamma_K = 0.23$.
 
    \item \textbf{FashionMNIST (Vision Benchmark):} $\lambda_t = 1.57$, which for pairwise merges gives $\gamma_K = 0.35$.
    
    \item \textbf{Conjecture on Network and Task Complexity:} We cautiously hypothesize that highly deep neural models trained on complex representations form distinct macroscopic sub-network symmetries, yielding clean log-linear cascades. Conversely, shallower networks on low-dimensional tabular data generate variance fluctuations frequently indistinguishable from baseline stochastic gradient noise.
\end{itemize}
\end{remark}
 
\subsection{UCI Tabular Datasets}
 
We evaluate shallow Multi-Layer Perceptrons (MLPs) optimized via SGD on a suite of standard tabular datasets from the UCI Machine Learning Repository. Across all datasets, the macroscopic order parameter (Effective Rank) and the training loss exhibit smooth, continuous decay. However, the detrended variance breaks this continuity, revealing discrete microtransitions. Table \ref{tab:uci_results} aggregates the empirical peak-time fits and spectral null-model evaluations across the suite.
 
Heart Disease gives the strongest tabular evidence ($\text{FPR} = 4.9\%$). Conversely, German Credit serves as a critical negative control. Although the network exhibits a 4-peak variance sequence ($\lambda_t = 1.99, R^2 = 0.86$), spectral null testing yields a False Positive Rate of $80.2\%$. This negative result establishes the necessity of the signal processing pipeline; without it, standard optimization noise mimics geometric scaling artifacts.
 
\begin{table}[htbp]
\centering
\caption{Empirical peak-time cascades and statistical null-model evaluations for UCI tabular datasets (aggregated over $20$ independent initializations). A cascade fit needs at least three peaks, so none is reported for Digits.}
\label{tab:uci_results}
\begin{tabular}{lcccc}
\toprule
\textbf{Dataset} & \textbf{Peaks} & $\lambda_t$ & $R^2$ & \textbf{FPR (\%)} \\
\midrule
Heart Disease    & 4 & 1.71 & 0.98 & 4.9\%  \\
Abalone          & 6 & 1.28 & 0.97 & 15.6\% \\
German Credit    & 4 & 1.99 & 0.86 & 80.2\% \\
Digits           & 1 & N/A  & N/A  & N/A    \\
\bottomrule
\end{tabular}
\end{table}
 
\begin{figure*}[htbp]
    \centering
    \includegraphics[width=\textwidth]{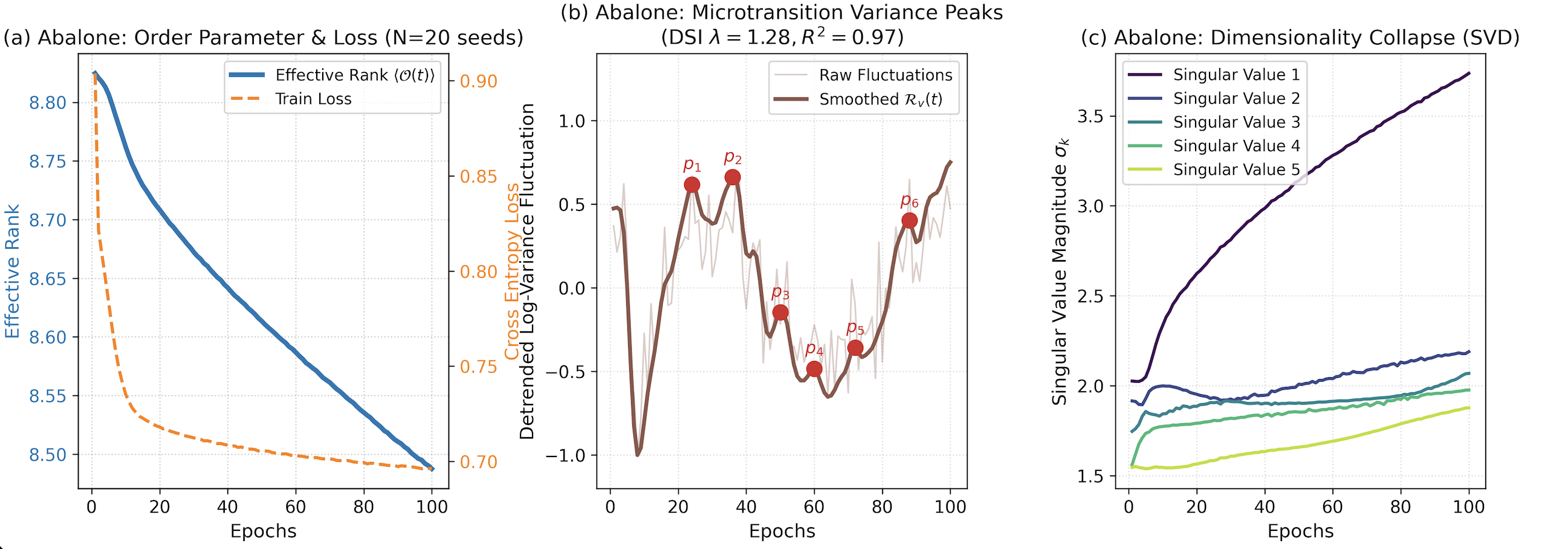}
    \caption{\textbf{Abalone:} The macroscopic continuous decay of effective rank (a) is underscored by a structured 6-peak cascade yielding a log-linear fit of $R^2 = 0.97$ with a fractional ratio $\lambda_t = 1.28$ (b). The collapse is driven by a dominant top singular value (c).}
    \label{fig:uci_abalone}
\end{figure*}
 
\begin{figure*}[htbp]
    \centering
    \includegraphics[width=\textwidth]{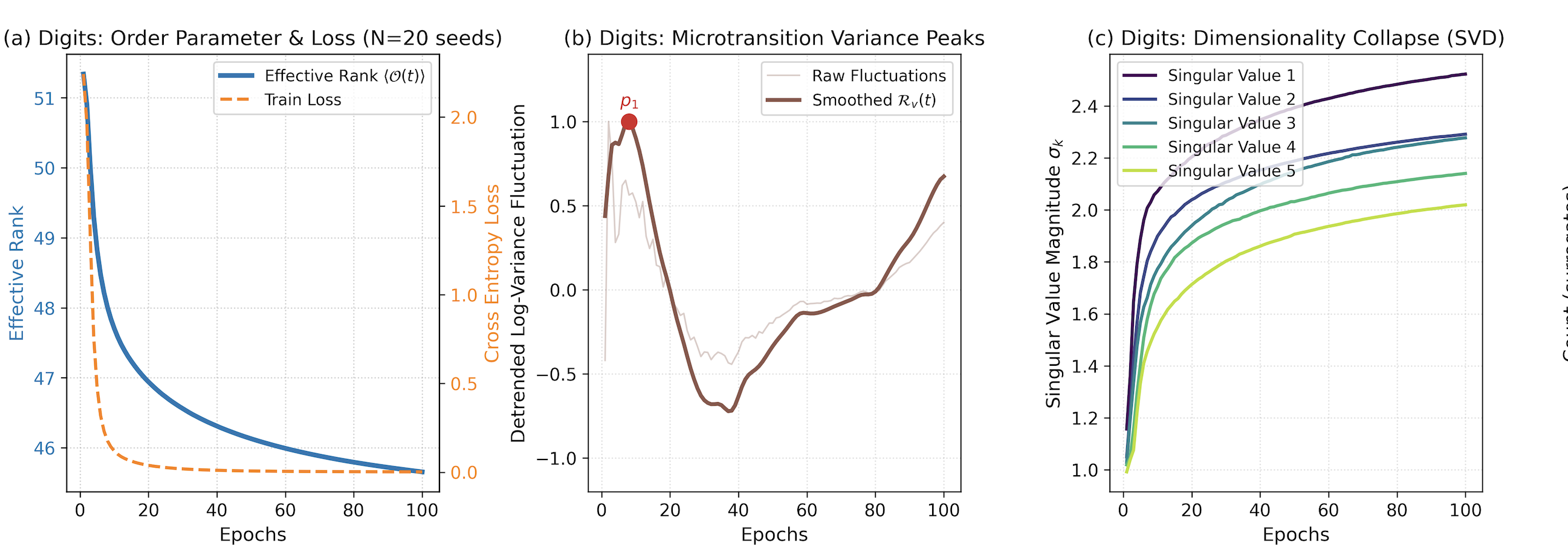}
    \caption{\textbf{Digits:} Variance fluctuations map the topological condensation of the network during early-stage SGD optimization.}
    \label{fig:uci_digits}
\end{figure*}
 
\begin{figure*}[htbp]
    \centering
    \includegraphics[width=\textwidth]{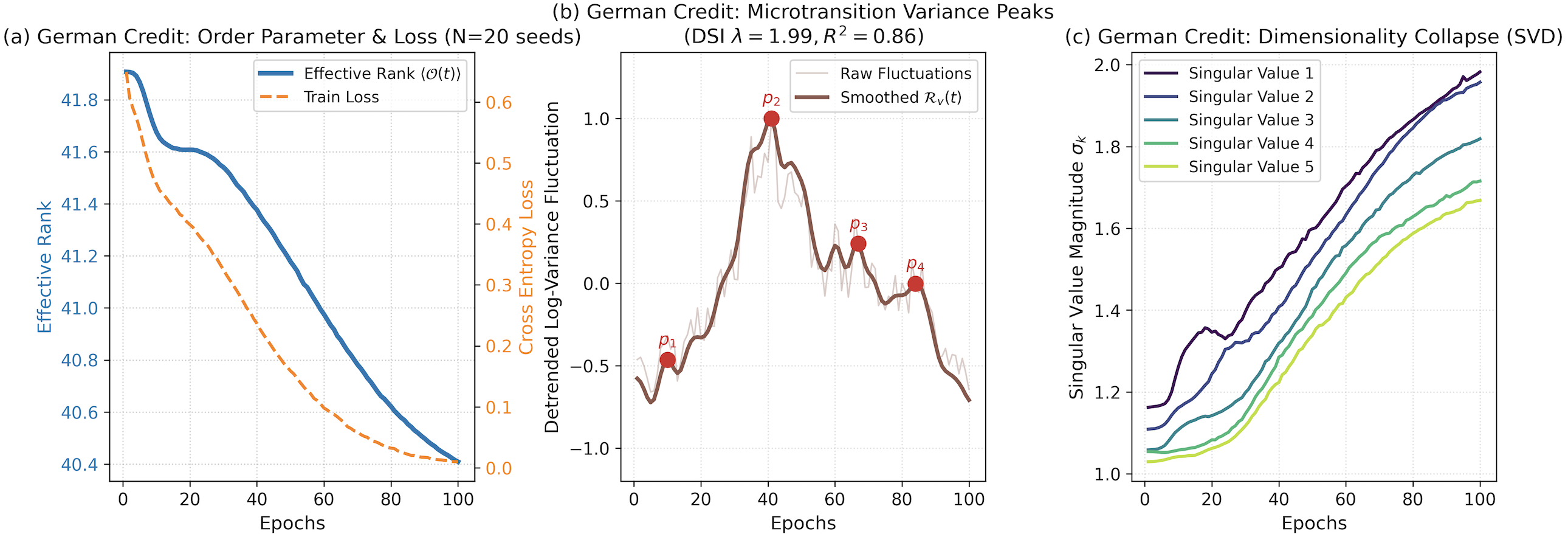}
    \caption{\textbf{German Credit:} The network exhibits a 4-peak cascade ($\lambda_t = 1.99, R^2 = 0.86$), but yields a spectral null False Positive Rate of $80.2\%$, acting as a negative control.}
    \label{fig:uci_german}
\end{figure*}
 
\begin{figure*}[htbp]
    \centering
    \includegraphics[width=\textwidth]{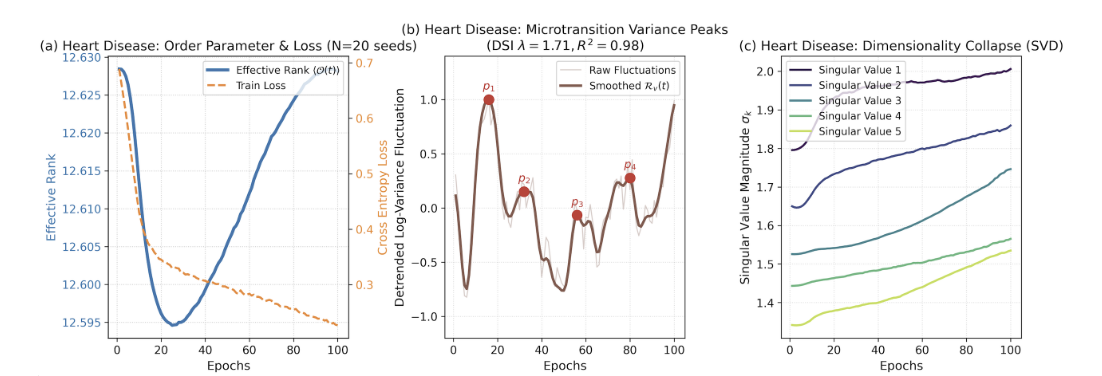}
    \caption{\textbf{Heart Disease:} Dimensionality collapse (c) is mirrored by a sequential 4-peak variance cascade ($\lambda_t = 1.71, R^2 = 0.98, \text{FPR} = 4.9\%$) mapping the internal topological constraints.}
    \label{fig:uci_heart}
\end{figure*}

\subsection{Geometric Manifolds, Vision, and Algorithmic Grokking}
 
To verify that these scaling laws are not strictly an artifact of tabular data or shallow MLPs, we extend the empirical validation to geometric manifolds (Moons, Swiss Roll), vision benchmarks (MNIST, FashionMNIST), and algorithmic tasks (Modular Arithmetic). 
 
Critically, the Modular Arithmetic task utilizes a Transformer architecture optimized via AdamW with explicit weight decay. This algorithmic setting inherently induces grokking, a dynamic phenomenon wherein deep neural networks initially overfit the training distribution before abruptly acquiring perfect generalization after extensive continued optimization \citep{power2022grokking}. Mechanistic analyses reveal that this delayed performance spike masks a representational phase transition; the network progressively dismantles high-rank, dense memorization circuits to construct low-dimensional, structured algorithmic manifolds \cite{nanda2022progress, liu2022towards}. Driven by explicit regularization, specifically the weight decay utilized in our AdamW optimizer, and inherent stochastic gradient noise, the optimization trajectory continuously penalizes complex parameterizations, forcing the weights to systematically escape sharp memorizing minima \cite{thilak2022slingshot}. Weight decay enters condition \eqref{eq:adam_condition} of Theorem \ref{thm:adam_trapping} through the inward term $-\eta\lambda\boldsymbol{x}$, and Remark \ref{remark:adam_diagnostic} reports the measurements. The transverse drift toward permutation coincidence of feed-forward units is inward at $63$ to $74\%$ of checkpoints across three realizations, the trapping condition holds at $50$ to $62\%$, and $Y$ contracts by up to $3.6$ decades between slingshot instabilities, which reset the condensation phase about once per $1.4\,\tau_2$. Consequently, we expect this prolonged macroscopic dimensionality reduction to happen through a sequence of discrete topological collapses. The Transformer's subsequent delayed generalization is consistent with this hypothesis. 
 
Across $15$ independent initializations, the final train loss converges to $0.005 \pm 0.004$, validation loss drops to $0.030 \pm 0.033$, and effective rank collapses to $10.93 \pm 0.99$. The detrended variance isolates a log-linear 3-peak cascade with $\lambda_t = 2.11$. With three peaks the fit has one degree of freedom, so the evidence for the cascade is the spectral null, which gives a False Positive Rate of at most $0.1\%$ (no surrogate matched the data). We test on addition modulo $17$ for $30000$ epochs with AdamW weight-decay coefficient $\lambda = 0.5$ and learning rate $0.003$; the three peaks lie at epochs $4500$, $9500$ and $20000$. Fourier pruning of the embedding, the progress measure of \citet{nanda2022progress}, completes by epoch $2850$ in three probe seeds, so the cascade runs after it at fixed Fourier support. In the same seeds the embedding columns coalesce at a tube radius of about half their norm in halving levels, the first reached at epochs $3960$ to $5470$, while coincidence of the units on which the permutation symmetry acts is measured on the feed-forward block (Remark \ref{remark:adam_diagnostic}).
 
The high-dimensional vision benchmarks mirror this discretized capacity reduction. FashionMNIST exhibits consecutive variance peaks highlighting structural scaling, yielding a 3-peak cascade with $\lambda_t = 1.57$ and a spectral null FPR of $10.7\%$. Conversely, MNIST produces a 2-peak sequence ($p_1, p_2$) mapping early-stage topological condensation without forming a full extended cascade.
 
Evaluations on geometric manifolds map the collapse phase as the MLP parameterizes spatial embeddings. Moons isolates a single primary microtransition ($p_1$) navigating the non-linear 2D geometry. Swiss Roll similarly identifies a primary structural microtransition ($p_1$) charting the 3D spatial embedding. These initial peaks capture the epoch of spatial alignment consistent with discrete topological condensation during early-stage optimization.
 
\begin{figure*}[htbp]
    \centering
    \includegraphics[width=\textwidth]{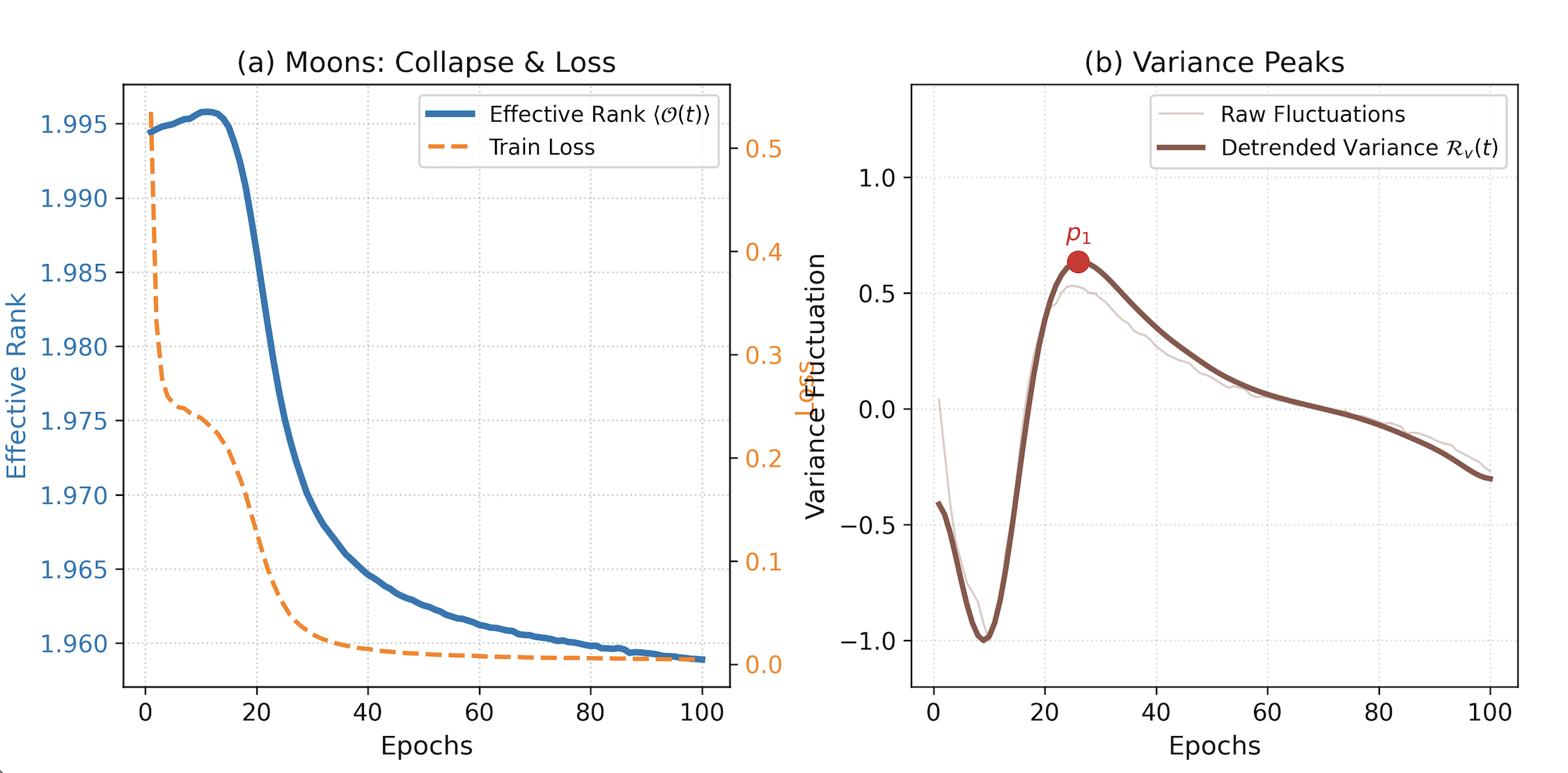}
    \caption{\textbf{Moons (Geometric Manifold):} Collapse phase mapped by variance fluctuations as the MLP parameterizes the 2D non-linear manifold.}
    \label{fig:suite_moons}
\end{figure*}
 
\begin{figure*}[htbp]
    \centering
    \includegraphics[width=\textwidth]{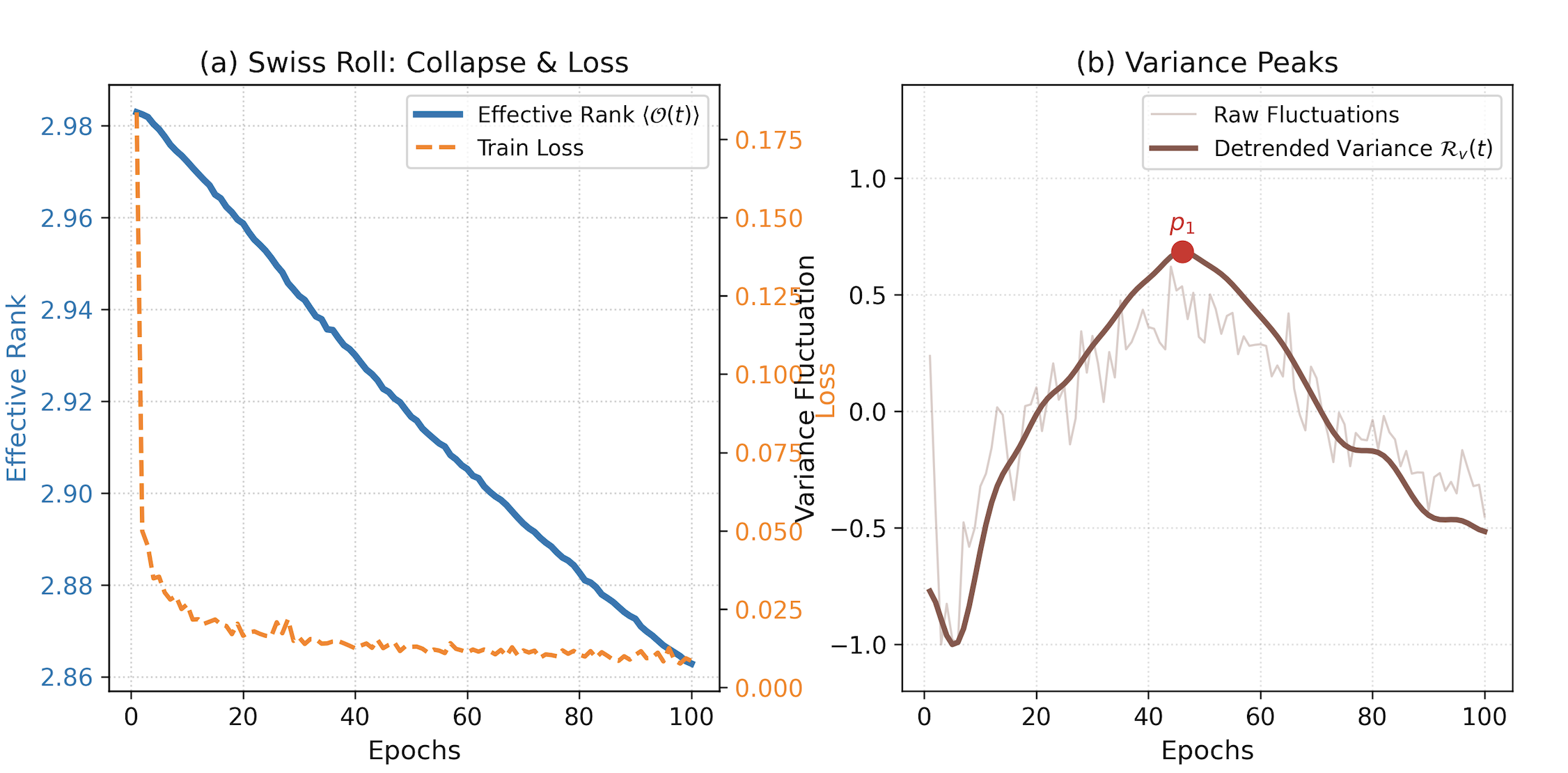}
    \caption{\textbf{Swiss Roll (Geometric Manifold):} Sequential structural microtransitions navigate the 3D spatial embedding of the data.}
    \label{fig:suite_swiss}
\end{figure*}
 
\begin{figure*}[htbp]
    \centering
    \includegraphics[width=\textwidth]{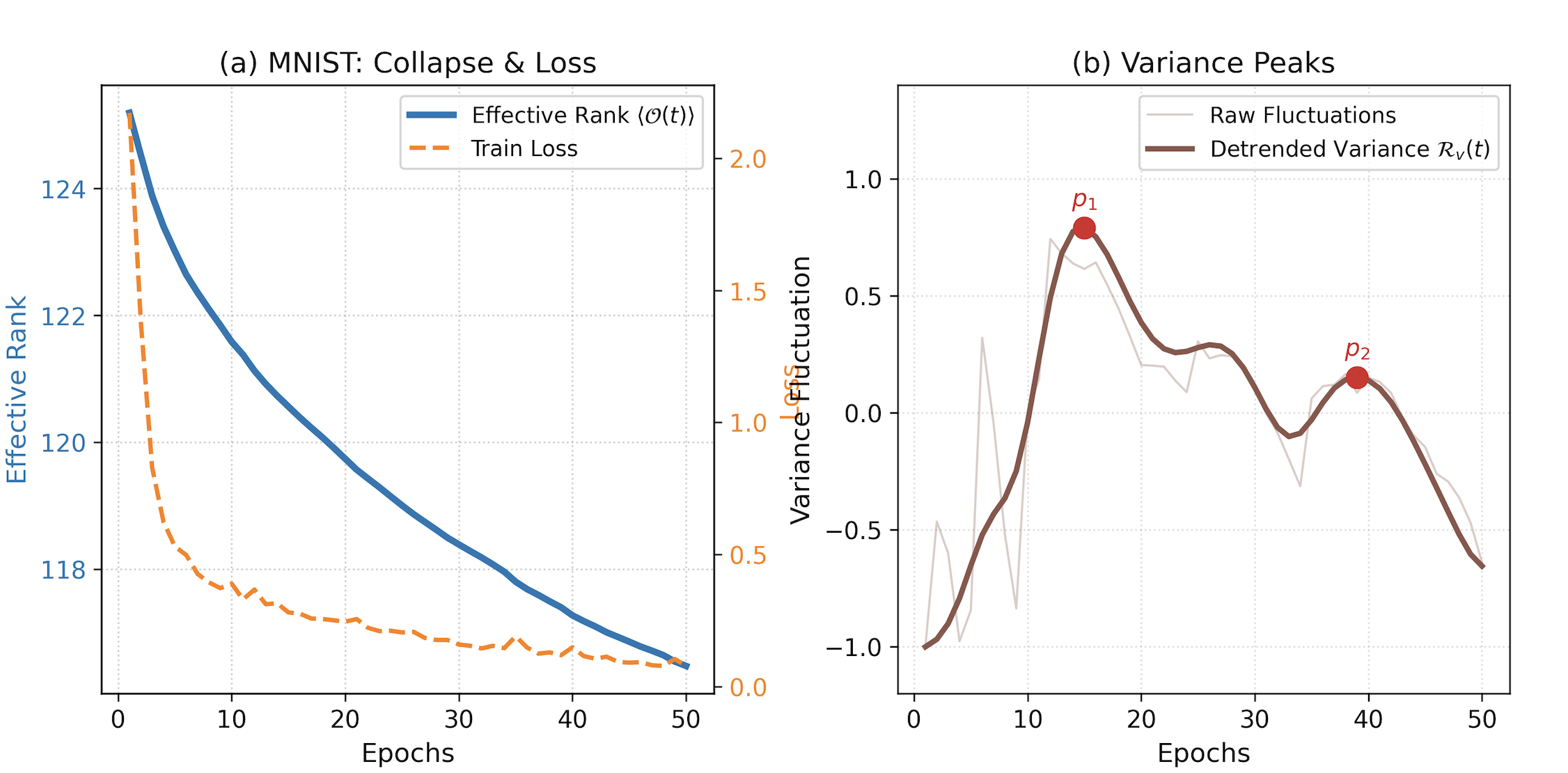}
    \caption{\textbf{MNIST (Vision):} Variance fluctuations produce a 2-peak sequence ($p_1, p_2$) mapping the topological condensation of the network during early-stage SGD optimization.}
    \label{fig:suite_mnist}
\end{figure*}
 
\begin{figure*}[htbp]
    \centering
    \includegraphics[width=\textwidth]{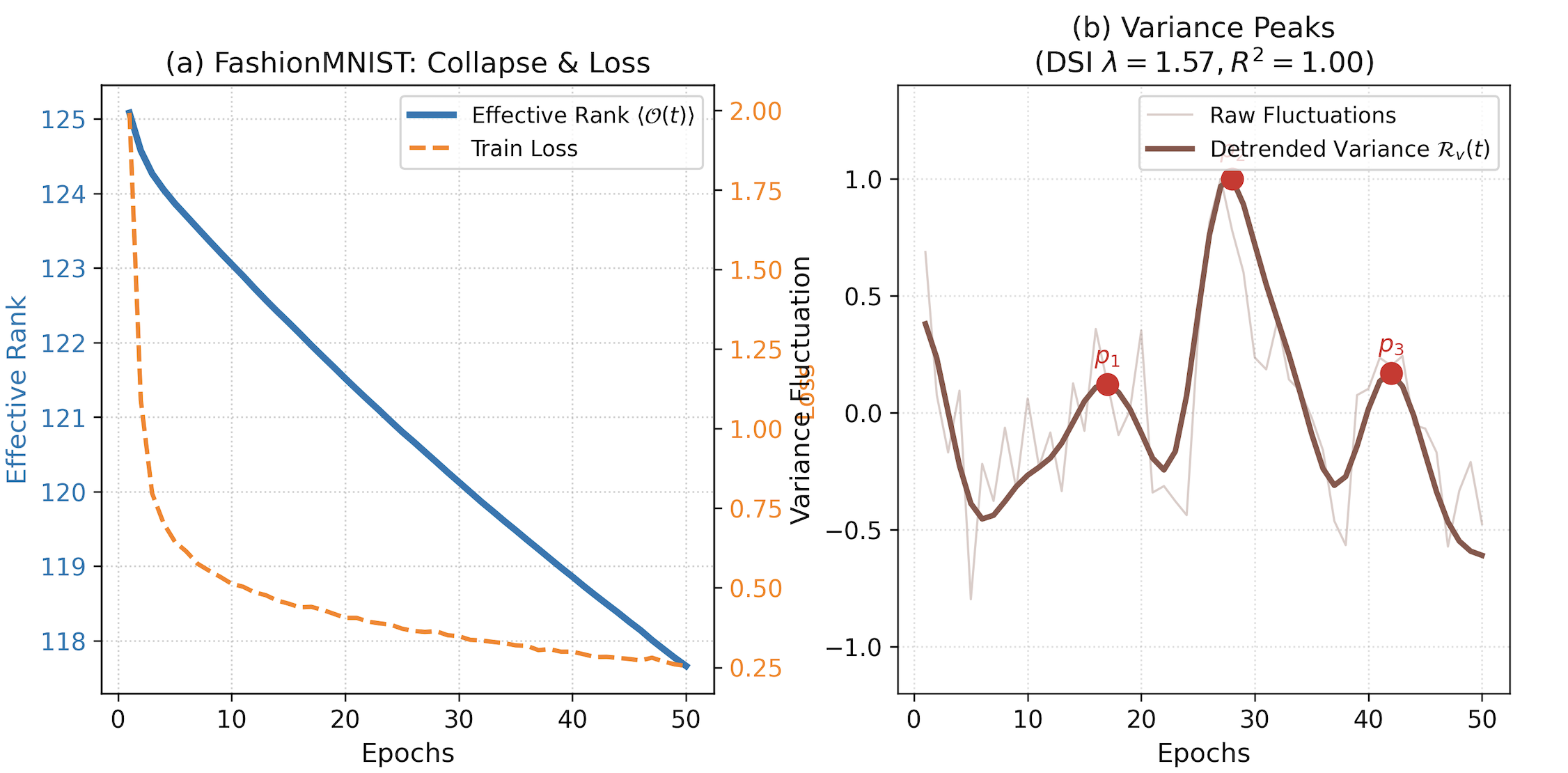}
    \caption{\textbf{FashionMNIST (Vision):} Detection of consecutive variance peaks highlighting the discretized nature of capacity reduction ($\lambda_t = 1.57, \text{FPR} = 10.7\%$).}
    \label{fig:suite_fashionmnist}
\end{figure*}
 
\begin{figure*}[htbp]
    \centering
    \includegraphics[width=\textwidth]{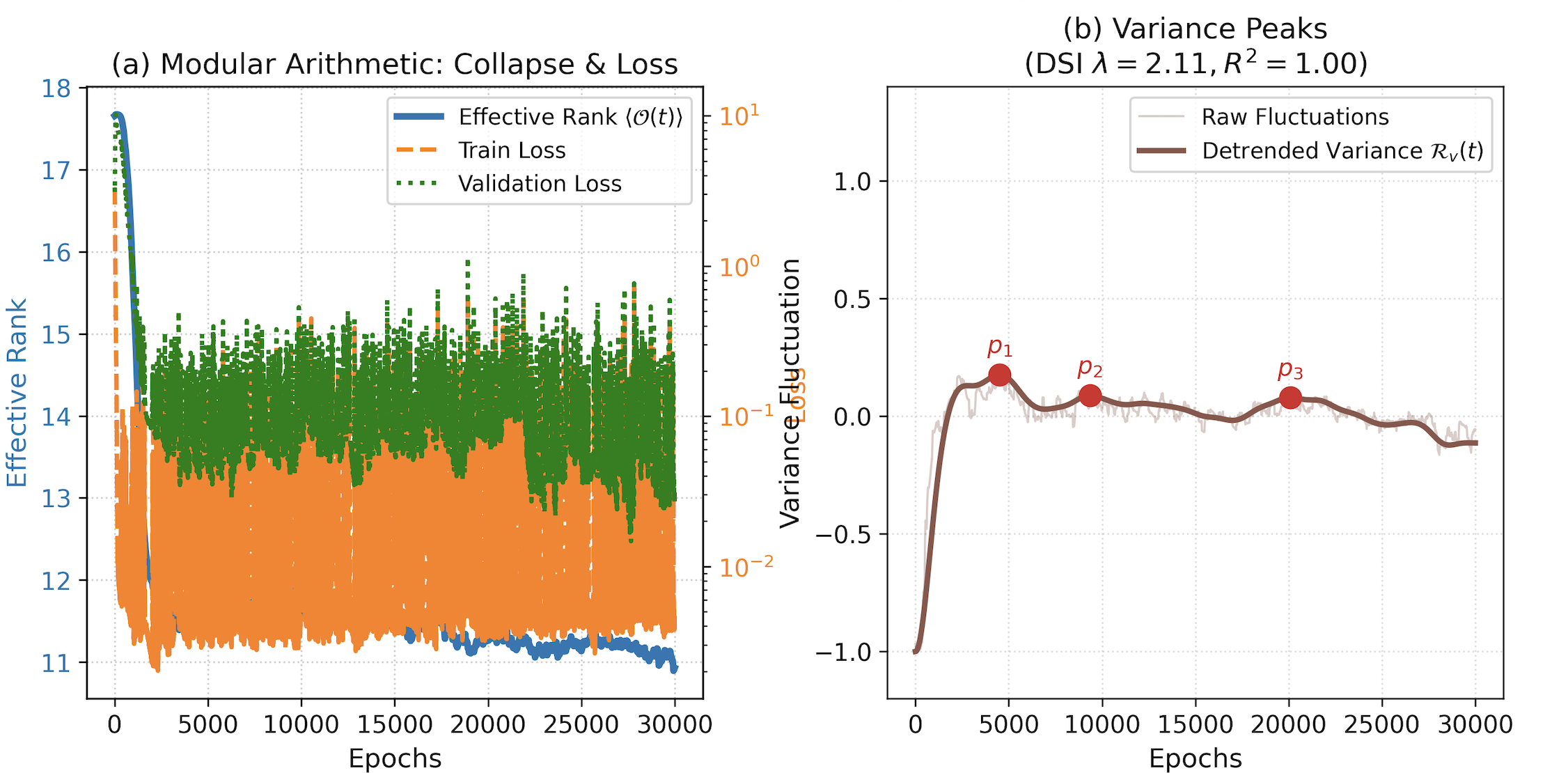}
    \caption{\textbf{Modular Arithmetic (Grokking):} A shallow Transformer trained with AdamW experiences a delayed topological collapse. The macroscopic decay of the token-embedding table's effective rank is composed of a log-linear 3-peak cascade ($\lambda_t = 2.11, \text{FPR} = 0.1\%$).}
    \label{fig:suite_grokking}
\end{figure*}

\end{document}